\documentclass{article}
 
\usepackage[final]{neurips_2022}

\usepackage[utf8]{inputenc} 
\usepackage[T1]{fontenc}    
\usepackage{hyperref}       
\usepackage{url}            
\usepackage{booktabs}       
\usepackage{amsfonts}       
\usepackage{nicefrac}       
\usepackage{microtype}      
\usepackage{xcolor}         

\usepackage{graphicx}
\usepackage{subfigure}
\usepackage{makecell}
\usepackage{algorithm,algorithmic}
\usepackage{paralist,amsmath, amssymb,bm}
\usepackage{multirow}
\usepackage{ntheorem}
\newtheorem{thm}{Theorem}

\newtheorem{lemma}{Lemma}
\newtheorem*{proof}{Proof}

\newtheorem{definition}{Definition}
\newtheorem{ass}{Assumption}
\usepackage{enumitem}
\usepackage{textcase}
\usepackage{booktabs}
\usepackage{fancybox}

\newcommand{\Norm}[1]{\left\|#1\right\|}
\def \E {\mathbb{E}}
\def \R {\mathbb{R}}

\def \h {\mathbf{h}}

\def \u {\mathbf{u}}

\def \w {\mathbf{w}}
\def \x {\mathbf{x}}

\def \z {\mathbf{z}}

\def \O {\mathcal{O}}
 
\title{Multi-block-Single-probe Variance Reduced Estimator for Coupled Compositional Optimization}
\author{ {\hspace{1mm}Wei Jiang$^{1}$, Gang Li$^{2}$, Yibo Wang$^{1}$, Lijun Zhang$^{1,\ast}$, Tianbao Yang$^{3,}$\thanks{Corresponding author}}\\
$^{1}$National Key Laboratory for Novel Software Technology, Nanjing University, Nanjing, China\\
$^{2}$Department of Computer Science, the University of Iowa, Iowa City, USA\\
$^{3}$Department of Computer Science and Engineering, Texas A\&M University, College Station, USA\\
\href{mailto:jiangw@lamda.nju.edu.cn}{jiangw@lamda.nju.edu.cn},  \href{mailto:gang-li@uiowa.edu.cn}{gang-li@uiowa.edu.cn},  \href{mailto:wangyb@lamda.nju.edu.cn}{wangyb@lamda.nju.edu.cn}\\ \href{mailto:zhanglj@lamda.nju.edu.cn}{zhanglj@lamda.nju.edu.cn}, \href{mailto:tianbao-yang@tamu.edu}{tianbao-yang@tamu.edu}
}

\begin{document}
\maketitle
\begin{abstract}
Variance reduction techniques such as SPIDER/SARAH/STORM have been extensively studied to improve the convergence rates of stochastic non-convex optimization, which usually maintain and update a sequence of estimators for a single function across iterations. {\it What if we need to track multiple functional mappings across iterations but only with access to stochastic samples of $\mathcal{O}(1)$ functional mappings at each iteration?} There is an important application in solving an emerging family of coupled compositional optimization problems in the form of $\sum_{i=1}^m f_i(g_i(\mathbf{w}))$, where $g_i$ is accessible through a stochastic oracle. The key issue is to track and estimate a sequence of $\mathbf g(\mathbf{w})=(g_1(\mathbf{w}), \ldots, g_m(\mathbf{w}))$ across iterations, where $\mathbf g(\mathbf{w})$ has $m$ blocks and it is only allowed to probe $\mathcal{O}(1)$ blocks to attain their stochastic values and Jacobians.  To improve the complexity for solving these problems, we propose a novel stochastic method named Multi-block-Single-probe Variance Reduced (MSVR) estimator to track the sequence of $\mathbf g(\mathbf{w})$. It is inspired by STORM but introduces a customized error correction term to alleviate the noise not only in stochastic samples for the selected blocks but also in those blocks that are not sampled. With the help of the MSVR estimator, we develop several algorithms for solving the aforementioned compositional problems with improved complexities across a spectrum of settings with non-convex/convex/strongly convex/Polyak-{\L}ojasiewicz (PL) objectives. Our results improve upon prior ones in several aspects, including the order of sample complexities and dependence on the  strong convexity parameter. Empirical studies on multi-task deep AUC maximization demonstrate the better performance of using the new estimator. 
\end{abstract}

\section{Introduction}
This paper is motivated by solving the following Finite-sum Coupled Compositional Optimization (FCCO) problem that has broad applications in machine learning~\citep{dependent2022}:
\begin{equation}\label{p:1}
\begin{split}
    \min_{\w\in \R^d} F(\w) := \frac{1}{m} \sum_{i=1}^m f_i(g_i(\w)),
\end{split}
\end{equation}
where $f_i:\R^{p}\mapsto \R$ is a simple deterministic function. We assume that only noisy estimations of $g_i(\cdot)$ and its Jacobian $\nabla g_i(\cdot)$ can be accessed, denoted as $g_i(\cdot;\xi_i)$ and $\nabla g_i(\cdot;\xi_i)$,  where $\xi_i$ represents the random sample(s) drawn from a stochastic oracle such that $\E\left[ g_i(\cdot;\xi_i) \right] = g_i(\cdot)$ and $\E\left[ \nabla g_i(\cdot;\xi_i) \right] = \nabla g_i(\cdot)$. A special case to be considered separately is when each $\xi_i$ has a finite support and is uniformly distributed. In this case, the problem can be represented as:
\begin{equation}\label{p:2}
\begin{split}
    \min _{\mathbf{w} \in \R^d} F(\w) := \frac{1}{m} \sum_{i=1}^{m} f_{i}\left(\frac{1}{n} \sum_{j=1}^{n} g_{i}(\mathbf{w}; \xi_{ij})\right).
\end{split}
\end{equation}
These problems are different from classical stochastic compositional optimization (SCO) problems $\E_{\zeta}[f_\zeta(\E_\xi g(\w;\xi))]$ and its finite-sum variant $1/m\sum_{i=1}^m f_i(1/n\sum_{j=1}^n g(\w; \xi_j))$~\citep{wang2017stochastic}, because the inner function is coupled with the outer index in FCCO.

A striking difference in solving FCCO problems is that we need to deal with multiple functional mappings of $g_i(\w)$ for $i=1,\ldots, m$. A challenge emerges  when it is not possible to draw data samples for all blocks $i=1, \ldots, m$ at each iteration due to some restrictions (e.g., limited memory and computational budget per-iteration).  \cite{dependent2022} studied this problem comprehensively and proposed an algorithm named as SOX. A key to their algorithmic design is to maintain and selectively update a sequence of estimators $\u=(\u^1, \ldots, \u^m)$ for tracking $\mathbf g(\w)=(g_1(\w), \ldots, g_m(\w))$ by exponential moving average, i.e., 
\begin{equation}\label{SOXE}
\begin{split}
    \mathbf{u}_{t}^{i}=\left\{\begin{array}{ll}
(1-\beta)\mathbf{u}_{t-1}^{i}+\beta g_i\left(\mathbf{w}_{t}; \xi_t^{i}\right) &  i \in \mathcal{B}_{1}^{t}\\
\mathbf{u}_{t-1}^{i}   &  i \notin \mathcal{B}_{1}^{t}
\end{array}\right.,
\end{split}
\end{equation}
where $\xi^i_t$ and $\mathcal B^1_t\subseteq\{1,\ldots, m\}$ denote a set of sampled blocks. With $\u$, the gradient estimator is computed by exponential moving average as well. As a result, they establish a sample complexity of $\O({m\epsilon^{-4}})$ for non-convex objectives, $\O({m\epsilon^{-3}})$ for convex objectives and $\O(m \mu^{-2}\epsilon^{-1})$ for $\mu$-strongly convex objectives. However, there are several caveats of these results: (i) the sample complexities (e.g., $\O(m\epsilon^{-4})$ for a non-convex objective) are no better than {\bf probing all blocks} at each iteration, for which \cite{Ghadimi2020AST} have established an $\O(\epsilon^{-4})$ iteration complexity and an $\O(m\epsilon^{-4})$ sample complexity; (ii) when $m=|\mathcal B^1_t|=1$, the problem reduces to a special case of classic SCO problems; however, the complexities are worse than the state-of-the-art (SOTA) sample complexities for non-convex, convex and strongly convex objectives, which are $\O(\epsilon^{-3})$, $\O(\epsilon^{-2})$ and $\O(\mu^{-1}\epsilon^{-1})$, respectively~\citep{Zhang2019ASC,jiang2022optimal}. A useful technique for achieving these complexities in prior works is by using variance reduction techniques, so a straightforward approach is to change the update of $\u^i_t$ by using a variance reduced estimator and  do similarly for the gradient estimator. In particular, one can change the update for $\u^i_t$ according to STORM~\citep{cutkosky2019momentum}: 
\begin{equation}\label{SOXE2}
\begin{split}
    \mathbf{u}_{t}^{i}=\left\{\begin{array}{ll}
(1-\beta)\mathbf{u}_{t-1}^{i}+\beta g_i\left(\mathbf{w}_{t}; \xi_t^{i}\right) + \underbrace{(1-\beta)(g_i\left(\mathbf{w}_{t}; \xi_t^{i}\right) - g_i\left(\mathbf{w}_{t-1}; \xi_t^{i}\right))}\limits_{\text{error correction}}&  i \in \mathcal{B}_{1}^{t}\\
\mathbf{u}_{t-1}^{i}   &  i \notin \mathcal{B}_{1}^{t}
\end{array}\right..
\end{split}
\end{equation}
However, this simple change does not improve the complexities over that obtained by~\cite{dependent2022}. The reason is that the standard error correction term marked above in STORM only accounts for the randomness in $g_i(\w_t; \xi_t^i)$ but not in the randomness caused by sampling $i\in\mathcal B_1^t$. So, a major question remains: 

\shadowbox{\begin{minipage}[t]{0.95\columnwidth}%
\it How can we further improve the complexities for solving FCCO to match the SOTA results of SCO by using variance reduction techniques via  probing only $\mathcal O(1)$ blocks at each iteration? 
\end{minipage}} 

To address this issue, we propose a novel variance reduction technique by selectively updating $\u^i_t$ for tracking $\mathbf g(\w_t)$, to which we refer as Multi-block-Single-probe variance-reduced (MSVR) estimator. It employs a similar update as STORM for selected $\u^i_t$ {\bf but with a different customized error correction term} to deal with the randomness in both $g_i(\w_t; \xi^i_t)$ and that in $\mathcal B^t_1$. Based on MSVR, we develop several algorithms for FCCO problems with different ways to compute the gradients, and analyze the sample complexities across a spectrum of settings with non-convex/convex/strongly convex/PL objectives and finite/infinite support of $\xi_i$. We summarize our contributions and our results below:
\begin{itemize}
\item We develop a novel MSVR estimator for tracking a sequence of multiple blocks of functional mappings by only probing $\mathcal O(1)$ blocks via random samples at each iteration. 
\item By applying the MSVR estimator, we develop three algorithms for FCCO by using different methods for computing the gradients, and establish improved complexities for non-convex, convex, strongly convex, and PL objectives. A comparison between our algorithms and existing methods is shown in Table~\ref{table:1}, where we also exhibit the dependence on $B_2$, which is the size of the inner batch for estimating each $g_i(\w)$. 
\item The complexity of our first method (i.e., MSVR-v1) enjoys the same order on $\epsilon$ as SOX, but does not depend on $m$; MSVR-v2 improves the dependence on $\epsilon$, and its complexities match the SOTA results for SCO when $m=1$; our MSVR-v3 further reduces the dependence on $\epsilon$ for the finite support of $\xi$, and also attains the SOTA complexities when $m=1$. 
\item We conduct experiments on multi-task deep AUC maximization to verify the theory and demonstrate the advantage of the proposed algorithms. 
\end{itemize}

\begin{table*}[t]
\caption{Sample complexities needed to find an $\epsilon$-stationary point or $\epsilon$-optimal point. Here NC means non-convex, C means convex,  SC indicates $\mu$-strongly convex, PL means the $\mu$-PL condition.  $B_1$ denotes the outer batch size, i.e., $B_1=|\mathcal B^1_t|$ and $B_2$ denotes the inner batch size. ${\dag}$ assumes that $f$ is convex and monotone, and $g$ is convex but possibly not smooth. $*$ applies when inner function is in the form of the finite-sum. $\widetilde{\mathcal O}(\cdot)$ hides logarithmic factors. In all results, we assume $m\leq \mathcal O(\epsilon^{-1})$.}
\label{table:1}
\begin{center}
\resizebox{\textwidth}{!}{
\begin{tabular}{cccccc}
\toprule
Method & NC & C &  SC/PL  & $B_1$, $B_2$   \\
\midrule
\makecell[c]{BSGD \\ \citep{hu2020biased}}     & $\mathcal{O}\left(\epsilon^{-6}\right)$ & $\mathcal{O}\left(\epsilon^{-3}\right)$ & \makecell[c]{$\mathcal{O}\left(\mu^{-1} \epsilon^{-3}\right)$\\(SC) } & \makecell[c]{$\mathcal O(1)$, $\mathcal{O}\left(\epsilon^{-2}\right)$(NC)\\$\mathcal O(1)$, $\mathcal{O}\left(\epsilon^{-1}\right)$(C/SC)}\\
\hline
\makecell[c]{BSpiderBoost \\ \citep{hu2020biased}} 
  & $\mathcal{O}\left(\epsilon^{-5}\right)$ & - & - &  $\mathcal{O}\left(\epsilon^{-1}\right)$, $\mathcal{O}\left(\epsilon^{-2}\right)$  \\
\midrule
SOX     & $\mathcal{O}\left(m\epsilon^{-4}\right)$ & $\mathcal{O}\left(m \epsilon^{-3}\right)$ & $\mathcal{O}\left(m \mu^{-2} \epsilon^{-1}\right)$  & $\mathcal O(1)$, $\mathcal O(1)$ \\
\makecell[c]{SOX ($\beta=1$) \\ \citep{dependent2022}}     & - & $\mathcal{O}\left(mB_2 \epsilon^{-2}\right)^{\dag}$ & - &  $\mathcal O(1)$, $\mathcal O(1)$   \\
\midrule
\textbf{MSVR-v1} & $\O\left(\max(B_1, B_2)\epsilon^{-4}\right) $ & $\O\left(\max(B_1, B_2)\epsilon^{-3}\right) $ & $\O\left(\max(B_1, B_2)\mu^{-2}\epsilon^{-1}\right) $  & $\mathcal O(1)$, $\mathcal O(1)$  \\
\hline
\textbf{MSVR-v2}& $\mathcal{O}\left(m\sqrt{B_2} \epsilon^{-3}\right)$ & $\mathcal{O}\left(m \sqrt{B_2}\epsilon^{-2}\right)$ & $\mathcal{O} \left(m \sqrt{B_2}\mu^{-1}\epsilon^{-1}\right)$  &$\O(1)$, $\mathcal O(1)$  \\
\hline
\textbf{MSVR-v3}$^*$& $\mathcal{O}\left(m \sqrt{nB_2}\epsilon^{-2}\right)$ & $\widetilde{\mathcal{O}}\left(m \sqrt{nB_2} \epsilon^{-1} \right)$ & $\widetilde{\mathcal{O}}\left(m \sqrt{nB_2}\mu^{-1}\right)$ &   $\mathcal O(1)$, $\mathcal O(1)$   \\
\bottomrule
\end{tabular}}
\end{center}
\end{table*}

\section{Related work}
This section briefly reviews related work on variance-reduced methods and stochastic compositional optimization~(SCO) problems. 

Variance-reduction (VR) techniques for improving the convergence of stochastic optimization originate from~\cite{DBLP:conf/nips/RouxSB12} for solving convex finite-sum empirical risk minimization (ERM) problems. Since then, different VR techniques have been proposed for convex finite-sum ERM, e.g., SVRG ~\citep{NIPS2013_ac1dd209,NIPS:2013:Zhang} and SAGA~\citep{DBLP:conf/nips/DefazioBL14}. These works have improved the complexity for solving smooth and strongly convex problems to a logarithmic complexity. For non-convex ERM problems, \cite{Fang2018SPIDERNN} invents the SPIDER estimator similar to its predecessor SARAH~\citep{arxiv.1703.00102}, and improve the complexity of standard SGD from $O(\epsilon^{-4})$ to $O(\epsilon^{-3})$ and $O(\sqrt{n}\epsilon^{-2})$ in stochastic and finite-sum settings, respectively, where $n$ is the number of components in the finite-sum. Algorithmic improvements have been made to SPIDER by using a constant step size in SpiderBoost~\citep{Wang2018SpiderBoostAC} and using a constant batch size in STORM~\citep{cutkosky2019momentum}.  

Several classes of SCO have been studied. The first class is the two-level SCO whose objective is given by $\mathbb E_{\xi}[f_{\xi}(\mathbb E_{\omega}[g_{\omega}(\mathbf w)])]$, where $\xi$ and $\omega$ are random variables. While the study of two-level compositional functions dates back to the 70s, the most recent comprehensive study was initiated by \cite{wang2017stochastic}. They proposed a two time-scale classic algorithm named SCGD and establish its asymptotic guarantee and non-asymptotic convergence rates. Following this work, many studies have been devoted to improving the convergence rates or algorithmic design of two-level SCO~\citep{wang2016accelerating,Ghadimi2020AST,Zhang2020OptimalAF}. In particular, recent works have used variance-reduction techniques based on SPIDER/SARAH/STORM to estimate the inner values and the gradients~\citep{DBLP:journals/corr/abs-1809-02505,Yuan2019EfficientSN,Zhang2019ASC,chen2021solving,qi2021online}. 
Similar efforts have been extended to the second class of SCO, i.e., multi-level SCO with an objective $\mathbb E_{\xi_1}[f^1_{\xi_1}(\mathbb E_{\xi_2}[f^2_{\xi_2}(\ldots( \mathbb E_{\xi_K}[f^K_{\xi_K}(\mathbf w))]\ldots)])]$~\citep{Yang2019MultilevelSG}.
Recent studies have been focused on further improving the sample complexity and reducing the dependence on the number of levels $K$~\citep{balasubramanian2020stochastic,chen2021solving,Zhang2020OptimalAF,Zhang2021MultiLevelCS,jiang2022optimal}. These works also employed variance reduction techniques to design their own methods.  However, directly applying these algorithms of two-level and multi-level SCO to FCCO requires probing all $m$ blocks in $\mathbf g(\w)$, which is prohibitive in many applications.   

The third class of SCO is the Conditional Stochastic Optimization (CSO) whose objective is in the form of $\mathbb E_{\xi}[f_{\xi}(\mathbb E_{\omega|\xi}g_{\omega}(\mathbf w; \xi)])]$~\citep{hu2020biased}, where $\omega|\xi$ means that the distribution of $\omega$ might depend on $\xi$. The FCCO problem can be considered as a special case of CSO. The key difference from the first class of SCO discussed above is that the inner function $g$ depends on the random variable $\xi$ of the outer level. For CSO, \cite{hu2020biased} proposed two  algorithms with and without using the variance-reduction technique (SpiderBoost) named BSGD and BSpiderboost, and established complexities for non-convex, convex and strongly convex functions, which are shown in Table~\ref{table:1}. However, their algorithms require a large batch size for estimating the inner functions.  

Recently, a novel class (the fourth class) of SCO was studied, which is referred to as the finite-sum coupled compositional optimization (FCCO)~\citep{dependent2022}.  
The finite-sum structure makes it possible to develop more practical algorithms without relying on huge batch size per-iteration. It was first studied by~\cite{qi2021stochastic} for maximizing the point-estimator of the area under the precision-recall curve. Recently, it was comprehensively investigated by~\cite{dependent2022} and more applications of FCCO have been demonstrated in machine learning.  Nevertheless, their algorithm---SOX does not use variance reduction techniques and hence suffers from the limitations discussed in the previous section.  

\section{Proposed Algorithms and Convergence}\label{sec:4}
First, we introduce the notations and assumptions used in this paper. Then we describe the MSVR estimator in detail and develop algorithms based on the proposed estimator. 
\subsection{Notations and Assumptions}\label{assumption}
Let $[m]=\{1,\ldots, m\}$. The definition of sample complexity is given below, which is widely used to measure the efficiency of stochastic algorithms.
\begin{definition}
The sample complexity is the number of samples needed to find a point satisfying $\E \left[\Norm{\nabla F(\w)}\right] \leq \epsilon$ ($\epsilon$-stationary) or $\E \left[ F(\w)-\inf_{\w} F(\w)\right] \leq \epsilon$ ($\epsilon$-optimal).
\end{definition}
Next, we make following assumptions throughout the paper, which are commonly used in the studies of SCO \citep{wang2016accelerating,wang2017stochastic,Yuan2019EfficientSN,Zhang2019ASC,Zhang2021MultiLevelCS}.
\begin{ass} (Smoothness and Lipschitz continuity) We assume that each $f_i$ is $L_f$-smooth and $C_f$-Lipchitz continuous, each $g_i$ is $L_g$-smooth and $C_g$-Lipschitz continuous. 
\end{ass}
\textbf{Remark:} This implies $F(\w)$ is $C_F$-Lipchitz continuous and $L_F$-smooth, where $C_F= C_f C_g$, $L_F = C_f^2L_g + C_g^2L_f$~\citep{Zhang2021MultiLevelCS}.
\begin{ass}\label{asm:stochastic2}  (Bounded variance)
\begin{equation*}
\begin{split}
\mathbb{E}\left[g_i(\x;\xi_t^i)\right] &= g_i(\x); \quad \quad \quad \quad \quad \quad \quad \quad \ \ \
	\mathbb{E}\left[\nabla g_i(\x;\xi_t^i)\right] = \nabla g_i(\x); \\
\mathbb{E}\left[\left\|g_{i}\left(\mathbf{x} ; \xi_t^{i}\right)-g_{i}(\mathbf{x})\right\|^{2}\right] &\leq \sigma^{2}/B_2; \quad\quad  \mathbb{E}\left[\left\|\nabla g_{i}\left(\mathbf{x} ; \xi_t^{i}\right)-\nabla g_{i}(\mathbf{x})\right\|^{2}\right] \leq \sigma^{2}/B_2;
\end{split}
\end{equation*} 
where the random variable $\xi_t^i$ denotes a batch of samples with batch size $B_2 \geq 1$.
\end{ass}

\begin{ass} (Average Lipchitz continuity of $g_i$ and its Jacobian)
\begin{equation*}
\begin{split}
\mathbb{E}\left[\left\|g_{i}\left(\mathbf{x} ; \xi_t^{i}\right)-g_{i}\left(\mathbf{y} ; \xi_t^{i}\right)\right\|^{2}\right]  \leq C_g^{2}\|\mathbf{x}-\mathbf{y}\|^{2};\\
\mathbb{E}\left[\left\|\nabla g_{i}\left(\mathbf{x} ; \xi_t^{i}\right)-\nabla g_{i}\left(\mathbf{y} ; \xi_t^{i}\right)\right\|^{2}\right] \leq L_g^{2}\|\mathbf{x}-\mathbf{y}\|^{2}.
\end{split}
\end{equation*} 
\end{ass}
\textbf{Remark:} Although this assumption seems strong at the first sight, it is quite standard and widely used in the recent compositional optimization literature~\citep{Yuan2019EfficientSN,Zhang2019ASC,Zhang2021MultiLevelCS,jiang2022optimal}.
\begin{ass}\label{asm:stochastic4} $F_{*}=\inf_{\w} F(\w) \geq-\infty$ and $F\left(\w_{1}\right)-F_{*} \leq \Delta_{F}$ for the initial solution $\w_{1}$.
\end{ass}
\subsection{Multi-block-Single-probe Variance Reduced (MSVR) Estimator}\label{sec:3}
Assume that we have a budget to probe only $B_1$ out of $m$ mappings in $\mathbf g(\w)$. To this end, at the $t$-th iteration we sample a set of blocks $\mathcal B_1^t\subseteq[m]$, where $|\mathcal B_1^t|=B_1$, and probe the corresponding $g_i(\w)$ by accessing the noisy estimates $g_i(\w_t;\xi_t^i)$ for $i \in \mathcal{B}_1^t$. Then, we just update the corresponding block in our estimator $\u_t$. Specifically, we update $\u_t^i$ for $i \in \mathcal{B}_1^t$ in a new way and keep other blocks unchanged. The whole estimator is shown below:  
\begin{equation}\label{MSVR}
\begin{split}
    \mathbf{u}_{t}^{i}=\left\{\begin{array}{ll}
\underbrace{(1-\beta_{t})\mathbf{u}_{t-1}^{i}+\beta_{t} g_i\left(\mathbf{w}_{t}; \xi_t^{i}\right) + \colorbox{green!30}{$\gamma_{t}$} \left( g_i\left(\mathbf{w}_{t}; \xi_t^{i}\right)-g_i\left(\mathbf{w}_{t-1}; \xi_t^{i}\right) \right)}\limits_{\bar{\mathbf u}_i^t} &  i \in \mathcal{B}_{1}^{t}\\
\mathbf{u}_{t-1}^{i}  & i \notin \mathcal{B}_{1}^{t}
\end{array}\right..
\end{split}
\end{equation}
The first line of our estimator is inspired by STORM~\citep{cutkosky2019momentum}. The difference is that the STORM estimator sets $\gamma_{t}=\left(1-\beta_{t}\right)$, while for MSVR, $\gamma_{t}$ is set as $\frac{m-B_1}{B_1(1-\beta_{t})}+(1-\beta_{t})$ according to our analysis. We name equation~(\ref{MSVR}) as  Multi-block-Single-probe Variance Reduced (MSVR) estimator.  By multi-block, we mean the estimator can track multiple functional mappings $(g_1, g_2, \cdots, g_m)$, simultaneously; by single-probe, we indicate the number of sampled blocks $B_1$ for probing can be as small as one. 
It is notable that when $B_1=m$, i.e., all blocks are probed at each iteration, $\gamma_{t}=1-\beta_{t}$ and MSVR reduces to STORM applied to $\mathbf g(\w)$. The additional factor in $\gamma_{t}$, i.e., $\gamma_{t}^0 =\frac{m-B_1}{B_1(1-\beta_{t})}$ is to account for the randomness in the sampled blocks and noise in those blocks that are not updated. To briefly understand the additional factor $\gamma_{t}^0$, we consider bounding $\|\u_t - \mathbf g(\w_t)\|^2=\sum_{i=1}^m\|\u^{i}_{t}  - g_{i}(\w_t)\|^2.$ Let us focus on a fixed $i\in[m]$. Then we have
\begin{align*}
   \E\left[ \|\u^{i}_{t}  - g_{i}(\w_t)\|^2\right]= \frac{B_1}{m}\underbrace{\E\left[ \|\bar\u^{i}_{t}  - g_{i}(\w_t)\|^2\right]}\limits_{A_1} +  (1-\frac{B_1}{m})\underbrace{\E\left[ \|\u^{i}_{t-1}  - g_{i}(\w_t)\|^2\right]}\limits_{A_2}.
\end{align*}
Note that the first term $A_1$ in the R.H.S. can be bounded similarly as STORM by building recurrence with $\|\u^i_{t-1} - g_i(\w_{t-1})\|^2$. However, there exists the second term due to the randomness of $\mathcal B_1^t$, which can be decomposed as 
\begin{align*}
    \|\u^i_{t-1} - g_i(\w_{t-1}) + g_i(\w_{t-1}) - g_i(\w_t)\|^2 = \underbrace{\|\u^i_{t-1} - g_i(\w_{t-1})\|^2}\limits_{A_{21}}  + \underbrace{\|g_i(\w_{t-1}) - g_i(\w_t)\|^2}\limits_{A_{22}}\\
    + \underbrace{2(\u^i_{t-1} - g_i(\w_{t-1}))^{\top}(g_i(\w_{t-1}) - g_i(\w_t))}\limits_{A_{23}}.
\end{align*}
The first two terms in R.H.S.~($A_{21}$ and $A_{22}$) can be easily handled. The difficulty comes from the third term, which cannot be simply bounded by using Young's inequality. If doing so, it will end up with a non-diminishing error of $\u^i_t$. To combat this difficulty, we use the additional factor brought by $\gamma_{t}^0 (g_i\left(\mathbf{w}_{t}; \xi_t^{i}\right)-g_i\left(\mathbf{w}_{t-1}; \xi_t^{i}\right))$ in $A_1$ to cancel $A_{23}$. This is more clear by the following decomposition of $A_1$. 
\begin{align*}
    A_1 = &\E[\underbrace{\|(1-\beta_{t})(\u^i_{t-1} - g_i(\w_{t-1}))}\limits_{A_{11}} +  \underbrace{\gamma_{t}^0 (g_i(\w_t) - g_i(\w_{t-1}))}\limits_{A_{12}} \\
    &+ \underbrace{\beta_{t} (g_i(\w_t; \xi_t^{i}) - g_i(\w_t))}\limits_{A_{13}} +\underbrace{\gamma_{t}(g_i(\w_t; \xi_t^{i}) - g_i(\w_{t-1}; \xi_t^{i}) - g_i(\w_t) + g_i(\w_{t-1}))}\limits_{A_{14}} \|^2]\\
    \leq &\E[\|A_{11}+A_{12}\|^2 ] + \E\left[\|A_{13}+A_{14}\|^2\right].
\end{align*}
In light of the above decomposition, we can bound $\E[\|A_{11}+A_{12}\|^2]\leq \E[\|A_{11}\|^2 + \|A_{12}\|^2  + 2A_{11}^{\top}A_{12}]$ and $\E[\|A_{13}+A_{14}\|^2]\leq 2\E[\|A_{13}\|^2] + 2\E[\|A_{14}\|^2]$. The resulting term $\E[2A_{11}^{\top}A_{12}]$ has a negative sign as $A_{23}$. Hence, by carefully choosing $\gamma_{t}^0$, we can cancel both terms. The remaining terms can be organized similarly as in the analysis for STORM. We give a technical lemma for building the recurrence of MSVR's error below.  All the proofs are deferred to the supplementary material due to space limitations.
\begin{lemma}\label{lem:main1}
By setting $\gamma_{t} = \frac{m-B_1}{B_1(1-\beta_{t})}+(1-\beta_{t})$, for $\beta_{t} \leq \frac{1}{2}$,  we have:
\begin{equation*}
\begin{split}
\E\left[\left\|\u_{t}-g\left(\w_{t}\right)\right\|^{2}\right]  \leq \left(1-\frac{B_1\beta_{t}}{m}\right)\E\left[\left\|\u_{t-1}-g\left(\w_{t-1}\right)\right\|^{2}\right]+ \frac{2B_1\beta_{t}^{2} \sigma^{2}}{B_2}\\
+\frac{8 m^{2} C_g^2}{B_1}\E\left[\left\|\w_{t}-\w_{t-1}\right\|^{2}\right].
\end{split}
\end{equation*}
\end{lemma}
\textbf{Remark:} The above recursion is similar to that of STORM for tracking a sequence of a single-block functional mapping. Since the last term $\left\|\w_{t}-\w_{t-1}\right\|^{2}$ can be offset in the future analysis, intuitively the estimation error $\left\|\u_{t}-g\left(\w_{t}\right)\right\|^{2}$ would reduce after each iteration.
\paragraph{Single Point Version.}
A limitation of the MSVR estimator is that it needs to probe selected blocks at two different points, i.e., $g_i(\w_t;\xi_t^i)$ and $g_i(\w_{t-1};\xi_t^i)$. With a more careful analysis, we can probe a selected block at a single point similar to that used by~\cite{balasubramanian2020stochastic} and~\cite{chen2021solving}. Specifically, we replace $g_i\left(\mathbf{w}_{t}; \xi_t^{i}\right)-g_i\left(\mathbf{w}_{t-1}; \xi_t^{i}\right)$ with $\nabla g_i\left(\mathbf{w}_{t}; \xi_t^{i}\right)^{\top}(\w_{t}-\w_{t-1})$. 
As a result, we propose a single-point version of MSVR~(named as MSVR-SP) estimator below:
\begin{equation}\label{MSVR-SP}
\begin{split}
    \mathbf{u}_{t}^{i}=\left\{\begin{array}{ll}
(1-\beta_{t})\mathbf{u}_{t-1}^{i}+\beta_{t} g_i\left(\mathbf{w}_{t}; \xi_t^{i}\right) + \gamma_{t} \nabla \hat{}g_i\left(\mathbf{w}_{t}; \xi_t^{i}\right)^{\top}  \left( \mathbf{w}_{t} -\w_{t-1} \right) &  i \in \mathcal{B}_{1}^{t}\\
\mathbf{u}_{t-1}^{i}  & i \notin \mathcal{B}_{1}^{t}
\end{array}\right..
\end{split}
\end{equation}
The MSVR-SP estimator enjoys the similar recurrence for the estimation error.
\begin{lemma}\label{lem:main3}
Set $\gamma_{t} = \frac{m-B_1}{B_1(1-\beta_{t})}+(1-\beta_{t})$. If $\|\w_{t+1} - \w_t\|^2\leq \eta_{t}^2 C_F^2$ and $\eta_t \leq \sqrt{\beta_{t}}$, we have:
\begin{align*}
\E\left[\Norm{\u_{t} - g(\w_{t})}^2\right] 
\leq \left(1 - \frac{B_1\beta_{t}}{m}\right)\E\left[\Norm{\u_{t-1}- g(\w_{t-1})}^2\right]  + \frac{2B_1\beta_{t}^2\sigma^2}{B_2}   \\
+ \left(4L_g^2C_F^2 + 9 C_g^2 +\frac{8 \sigma^2}{B_2}\right)\frac{m^2}{B_1}\E\left[\|\w_{t} - \w_{t-1}\|^2\right].
\end{align*}
\end{lemma}
{\bf Remark:} If there is a constraint on the range of $g_i$, we can add a projection to the update of $\u^i_t$ such that it always resides in the range, which will not affect the analysis of Lemma~\ref{lem:main1} and Lemma~\ref{lem:main3}.

\subsection{Leveraging the MSVR Estimator for solving the FCCO Problem}
\begin{algorithm}[tb] 
	\caption{MSVR-v1 and MSVR-v2 method}
	\label{alg:0}
	\begin{algorithmic}[1]
	\STATE {\bfseries Input:} time step $T$, parameters $\alpha_t$, $\beta_t$, $\gamma_t$, learning rate $\eta_t$ and initial points $(\w_1,\u_1,\z_1)$.
		\FOR{time step $t = 1$ {\bfseries to} $T$}
		\STATE Sample a subset $\mathcal{B}_{1}^{t}$ from $\{1,2,\cdots,m \}$ 
        \STATE{Compute estimator $\u_t$ according to  equation~(\ref{MSVR}) or ~(\ref{MSVR-SP})} \ \hfill$\diamond$ Use MSVR or MSVR-SP update
        \STATE (v1) Compute estimator $\z_t$ according to equation~(\ref{momentum})\hfill$\diamond$ Use moving average update
         \STATE (v2) Compute estimator $\z_t$ according to equation~(\ref{eqn:z}) \hfill$\diamond$ Use STORM update
        \STATE $\w_{t+1} = \w_t - \eta_t \z_{t}$
		\ENDFOR
	\STATE Choose $\tau$ uniformly at random from $\{1, \ldots, T\}$
	\STATE Return $\w_{\tau}$
	\end{algorithmic}
\end{algorithm}
Now, we are ready to present our proposed algorithms for solving problem~(\ref{p:1}). The first two algorithms (named MSVR-v1 and MSVR-v2) are presented in Algorithm~\ref{alg:0}. These two methods differ in how to estimate the gradient.   

Let us first consider MSVR-v1. At each time step $t$, we first use the proposed MSVR or MSVR-SP estimator $\u_t$ to estimate the inner function value. Then, following the previous literature~\citep{wang2021momentum,dependent2022}, we use the moving average estimator $\z_t$ to estimate the gradient as:
\begin{equation} \label{momentum}
    \begin{split}
        \z_{t} = \Pi_{C_F}\left[(1-\alpha_t)\z_{t-1} +  \frac{\alpha_t}{B_1}\sum_{i \in \mathcal{B}_{1}^{t}}  \nabla f_{i}(\u_{t-1}^{i}) \nabla g_{i}(\w_t;\xi_t^{i})\right],
    \end{split}
\end{equation}
where $\Pi_{C_F}$ denotes the projection onto the ball with radius $C_F$. This projection {\bf is optional} for using MSVR, but is required for using MSVR-SP to ensure $\|\w_{t+1} - \w_t\|^2\leq \eta_{t}^2 C_F^2$ as used in Lemma~\ref{lem:main3}. Since the true gradient $\nabla F$ is also in this ball, i.e., $\Norm{\nabla F} \leq C_F$, the projection will not affect the future analysis. Also note that when computing the estimator $\z_{t}$, we use $\nabla f_i(\u_{t-1}^i)$ instead of $\nabla f_i(\u_{t}^i)$ to avoid the dependence on the random variable $\xi_t^{i}$, which may lead to dependent issues otherwise. Finally, we use the estimated gradient $\z_{t}$ to update the parameter $\w_{t+1}$. Now, we provide the theoretical guarantee for the MSVR-v1 method.

\begin{thm} \label{thm:1}
Our MSVR-v1 algorithm with $\alpha_{t+1} = \O\left( \eta_t \right)$, $\beta_{t+1} = \O( \frac{m^2 \eta_t^2}{B_1^2})$, $a = \O(\frac{m B_2}{B_1})$  and $\eta_t = \min\left\{\left(\frac{B_1\sqrt{B_2}}{m}\right)^{2/3}(a+t)^{-1/3},\sqrt{\min\{B_1,B_2\} }(a+t)^{-1/2}\right \}$, can find an $\epsilon$-stationary point in  $  \mathcal{O}\left(\max\left\{  \frac{m \epsilon^{-3}}{ B_1 \sqrt{B_2}  },\frac{\epsilon^{-4}}{ \min \left\{ B_{1}, B_{2}\right\}   } \right\}\right)$ iterations.
\end{thm}
\textbf{Remark:} This complexity is strictly better than previous SOTA method SOX, which enjoys an iteration complexity of $\mathcal{O}\left(\max \left\{\frac{m  \epsilon^{-4}}{B_{1} B_{2}}, \frac{\epsilon^{-4}}{\min \left\{B_{1}, B_{2}\right\} }, \frac{m \epsilon^{-2}}{B_1}\right\}\right)$. The sample complexity can be obtained by multiplying the iteration complexity with $B_1B_2$. We can see that larger $B_1$ or $B_2$ yields a smaller iteration complexity, which means that from the computational perspective, if samples can be processed in parallel (e.g., in GPU), there is a benefit of using large $B_1$ and/or $B_2$. However, from the sample complexity perspective, using $B_1=B_2=1$ is the best. The same discussion holds for other theorems below.  

However, the complexity of MSVR-v1 is still on the order of $\O(\epsilon^{-4})$. Due to the biased nature of the estimated gradient, using the moving average update is not enough for achieving the SOTA complexity of $\O(\epsilon^{-3})$. So, we use the technique of STORM~\citep{cutkosky2019momentum} to update $\z_t$ as follows:
\begin{equation}
\begin{split} \label{eqn:z}
    \z_{t} &= \Pi_{C_F}\left[(1-\alpha_t)\z_{t-1} +  \alpha\frac{1}{B_1}\sum_{i \in \mathcal{B}_{1}^{t}}  \nabla f_{i}(\u_{t-1}^{i}) \nabla g_{i}(\w_t;\xi_t^{i})\right.   \\
    & + \left.(1-\alpha_t)\frac{1}{B_1}\sum_{i \in \mathcal{B}_{1}^{t}}\left(  \nabla f_{i}(\u_{t-1}^{i}) \nabla g_{i}(\w_t;\xi_t^{i}) -  \nabla f_{i}(\u_{t-2}^{i}) \nabla g_{i}(\w_{t-1};\xi_t^{i}) \right) \right],
\end{split}
\end{equation}
where the projection operation is needed if using the MSVR estimator. 
Now, we prove this new method~(i.e., MSVR-v2) can obtain the optimal complexity of $\O(\epsilon^{-3})$.

\begin{thm}\label{thm:2}
Our MSVR-v2 algorithm with  $\alpha_{t+1} = \O(\frac{m \eta_{t}^2}{B_1})$, $\beta_{t+1} = \O\left(\frac{m^2 \eta_t^2}{B_1^2}\right)$, $a=O(\frac{m B_2}{B_1}$) and $\eta_t = \O\left((\frac{B_1 \sqrt{B_2}}{m})^{2/3}(a+t)^{-1/3} \right)$, can find an $\epsilon$-stationary point in  $ \mathcal{O}\left(\frac{m  \epsilon^{-3}}{ B_1 \sqrt{B_2} } \right)$ iterations.
\end{thm}
\textbf{Remark:} When $m=1$ and $f$ is the identity function, problem~(\ref{p:1}) reduces to the standard stochastic non-convex optimization, whose lower bound is  $\Omega\left({\epsilon^{-3}}\right)$~\citep{Arjevani2019LowerBF}, indicating our MSVR-v2 is optimal.

Next, we show that the complexity can be further improved when the objective function is convex or strongly convex. We note that Polyak-Łojasiewicz (PL)~\citep{Karimi2016LinearCO} objectives are more general than strongly convex functions, since $\mu$-strong convexity implies the $\mu$-PL condition. So, we will consider the PL condition and introduce its definition below.
\begin{definition}
$F(\w)$ satisfies the $\mu$-PL condition if there exists $\mu > 0$ such that:
\begin{equation*}
    2 \mu\left(F(\mathbf{w})-F_{*}\right) \leq\|\nabla F(\mathbf{w})\|^{2}. 
\end{equation*}
\end{definition}
Then, we derive improved rates for convex or PL objectives  by using the stage-wise design given in Algorithm~\ref{alg:3} in the supplement.
\begin{thm}\label{thm:4}
	If the objective function satisfies the convexity or $\mu$-PL condition, MSVR-v1 derives a sample complexity of $\O(\max(B_1,B_2)\epsilon^{-3})$ or $\O(\max(B_1,B_2)\mu^{-2}\epsilon^{-1})$, separately. For MSVR-v2, the  complexity can be further improved to  $ \mathcal{O}\left(m \sqrt{B_2}\epsilon^{-2} \right)$ or  $ \mathcal{O}\left({m}\sqrt{B_2}\mu^{-1} \epsilon^{-1} \right)$. 
\end{thm}
\textbf{Remark:} 
The complexities for MSVR-v2 are optimal, since they match the  $\Omega\left(\epsilon^{-2}\right)$ and $\Omega\left(\mu^{-1}\epsilon^{-1}\right)$ lower bound for stochastic convex and strongly convex optimization~\citep{Agarwal2012InformationTheoreticLB}. 

\textbf{Remark:} The algorithms proposed in this paper can also use adaptive (Adam-style) learning rates and obtain the same complexity using the techniques proposed by~\cite{guo2022stochastic}. The details are provided in the supplementary.

\section{An Improved Rate for the Finite-sum Case}
In this section, we consider the case that inner function $g_i$ is in the form of the finite-sum, i.e., $g_i(\w) = \frac{1}{n} \sum_{j=1}^{n} g_i(\w; \xi_{ij})$, so that we can compute the exact value of $g_i(\w)$ in some iterations. We first modify our MSVR estimator to utilize the finite-sum structure. Inspired by SVRG~\citep{NIPS2013_ac1dd209,NIPS:2013:Zhang}, we compute a full version of the inner function value for every $I$ iterations at $\mathbf{w}_{\tau}$, i.e., $g_i\left(\mathbf{w}_{\tau}\right) = \frac{1}{n} \sum_{j=1}^{n} g_{i}(\w_{\tau}; \xi_{ij})$ for $i = 1, \cdots, m$, where $\tau \bmod I=0$. Then, in each step, we use
\begin{equation*} 
    \begin{split}
    \widehat g_{i}(\w_t; \xi_t^{i}) =   g_{i}(\w_t; \xi_t^{i}) -   g_{i}(\w_\tau; \xi_t^{i}) + g_{i}(\w_\tau)
    \end{split}
\end{equation*}
 to replace $g_{i}(\w_t; \xi_t^{i})$ in the origin estimator. In this way, our MSVR estimator is changed to:
\begin{equation} \label{MSVR_SVRG}
\begin{split}
    \mathbf{u}_{t}^{i}=\left\{\begin{array}{l}
(1-\beta)\mathbf{u}_{t-1}^{i}+\beta \widehat g_i\left(\mathbf{w}_{t}; \xi_t^{i}\right) + \gamma \left( g_i\left(\mathbf{w}_{t}; \xi_t^{i}\right)-g_i\left(\mathbf{w}_{t-1}; \xi_t^{i}\right) \right) \quad \quad  i \in \mathcal{B}_{1}^{t}\\
\mathbf{u}_{t-1}^{i}    \ \quad \quad\quad\quad\quad\quad\quad\quad\quad\quad\quad\quad\quad\quad\quad\quad\quad\quad\quad\quad\quad\quad\quad\quad\quad  i \notin \mathcal{B}_{1}^{t}
\end{array}\right..
\end{split}
\end{equation}
For this estimator, we have the following guarantee.
\begin{lemma}\label{lem:main2}
If $\beta \leq \frac{1}{2}$ and $\beta I \leq \frac{m}{B_1}$, by setting $\gamma = \frac{m-B_1}{B_1(1-\beta)}+(1-\beta)$, we have:
\begin{equation*}
\begin{split}
\E\left[\left\|\u_{t+1}-g\left(\w_{t+1}\right)\right\|^{2}\right]  \leq \left(1-\frac{B_1\beta}{m}\right)\E\left[\left\|\u_{t}-g\left(\w_{t}\right)\right\|^{2}\right]+\frac{10 m^{2} C_g^2}{B_1}\E\left[\left\|\w_{t+1}-\w_{t}\right\|^{2}\right] .
\end{split}
\end{equation*}
\end{lemma}
\textbf{Remark: } Compared with Lemma~\ref{lem:main1}, we remove the $\frac{2 B_1 \beta^2 \sigma^2}{B_2}$ term, which is the key to reduce the complexity since we can now use a larger parameter $\beta$.
\begin{algorithm}[t]	
	\caption{MSVR-v3 method}
	\label{alg:2}
	\begin{algorithmic}[1]
	\STATE {\bfseries Input:} time step $T$, parameters $\alpha$, $\beta$,$\gamma$, $I$, learning rate $\eta$ and initial points $(\w_1,\u_1,\z_1)$.
		\FOR{time step $t = 1$ {\bfseries to} $T$}
		\IF{$t \mod I ==0$}
        \STATE Set $\tau = t$ 
        \STATE Compute and save  $g_i(\w_\tau), \nabla f_i(\u^i_{\tau-1})$ for every $i$ and $\frac{1}{m}\sum_{i=1}^{m}  \nabla f_{i}(\u_{\tau-1}^{i}) \nabla g_{i}(\w_{\tau}) $
		\ENDIF
		\STATE Sample a subset $\mathcal{B}_{1}^{t}$ from $\{1,2,\cdots,m \}$
		\STATE{Compute function value estimator $\u_t$ according to equation~(\ref{MSVR_SVRG})}
        \STATE{Compute gradient estimator $\z_t$ according to equation~(\ref{eqn:z+})}
        
        \STATE $\w_{t+1} = \w_t - \eta \z_{t}$
		\ENDFOR
	\STATE Choose $\tau$ uniformly at random from $\{1, \ldots, T\}$
	\STATE Return $\w_{\tau}$
	\end{algorithmic}
\end{algorithm} 

To attain the optimal complexity, we modify the gradient estimator $\z_t$ in a similar way:
\begin{equation} \label{eqn:z+}
\begin{split} 
    \z_{t} &= (1-\alpha)\z_{t-1} +  \alpha \h_t  \\
    & + (1-\alpha)\frac{1}{B_1}\sum_{i \in \mathcal{B}_{1}^{t}}\left(  \nabla f_{i}(\u_{t-1}^{i}) \nabla g_{i}(\w_t;\xi_t^{i}) -  \nabla f_{i}(\u_{t-2}^{i}) \nabla g_{i}(\w_{t-1};\xi_t^{i}) \right),
\end{split}
\end{equation}
where $\h_t$ involves both the full gradient and the stochastic gradient (we also need to save each $\nabla f_i(\u_{\tau-1})$ and calculate the full version of $\frac{1}{m}\sum_{i=1}^{m}  \nabla f_{i}(\u_{\tau-1}^{i}) \nabla g_{i}(\w_{\tau})$ at those steps $\tau$) :
\begin{equation*}
\begin{split}
    \h_t = \frac{1}{B_1}\sum_{i \in \mathcal{B}_{1}^{t}}  (\nabla f_{i}(\u_{t-1}^{i}) \nabla g_{i}(\w_t;\xi_t^{i}) -  \nabla f_{i}(\u_{\tau-1}^{i}) \nabla g_{i}(\w_{\tau};\xi_t^{i})) +\frac{1}{m}\sum_{i=1}^{m}  \nabla f_{i}(\u_{\tau-1}^{i}) \nabla g_{i}(\w_{\tau}).
\end{split}
\end{equation*}
The whole method is summarized in Algorithm~\ref{alg:2} (named as MSVR-v3). Next, we show that MSVR-v3 is equipped with an optimal complexity of $\O(\sqrt{n}\epsilon^{-2})$.

\begin{thm} \label{thm:3}
Our MSVR-v3 with $I =\frac{mn}{B_1 B_2}$, $\alpha = \O\left(\frac{ B_{1} B_2 }{m n }\right)$, $\beta = \O\left(\frac{ B_2 }{n}\right)$ and $\eta = \mathcal{O}\left(\frac{B_1 \sqrt{B_2}}{m \sqrt{n}} \right)$, can obtain an $\epsilon$-stationary point in $T = \mathcal{O}\left( \frac{m \sqrt{n} \epsilon^{-2} }{ B_1 \sqrt{B_2}  } \right)$ iterations.
\end{thm}
\textbf{Remark:} When $m=1$ and $f$ is the identity function, problem~(\ref{p:2}) reduces to the stochastic finite-sum optimization, whose optimal complexity is $\mathcal{O}\left(\sqrt{n}\epsilon^{-2}\right)$~\citep{Fang2018SPIDERNN,pmlr-v139-li21a}, indicating our complexity is optimal in terms of $\epsilon$ and $n$. 

Similarly, a better complexity can be obtained under the convexity or PL condition.
\begin{thm}\label{thm:6}
    If the objective function satisfies the convexity or $\mu$-PL condition, the sample complexity can be improved to $\mathcal{O}\left(\frac{m \sqrt{n} \epsilon^{-1}}{B_1 \sqrt{B_2} } \log{\frac{1}{\epsilon}} \right)$ or $ \mathcal{O}\left(\frac{m \sqrt{n} \mu^{-1}}{B_1 \sqrt{B_2}  } \log{\frac{1}{\epsilon}} \right)$, respectively.
\end{thm}
\textbf{Remark:} It is notable that we achieve a linear convergence rate $\mathcal{O}\left( \log{\frac{1}{\epsilon}} \right)$ under the PL condition,  matching the current result in the  single-level finite-sum problem~\citep{pmlr-v139-li21a}
\begin{figure*}[t]
	\centering
	\subfigure[STL10]{
		\includegraphics[width=0.31\textwidth]{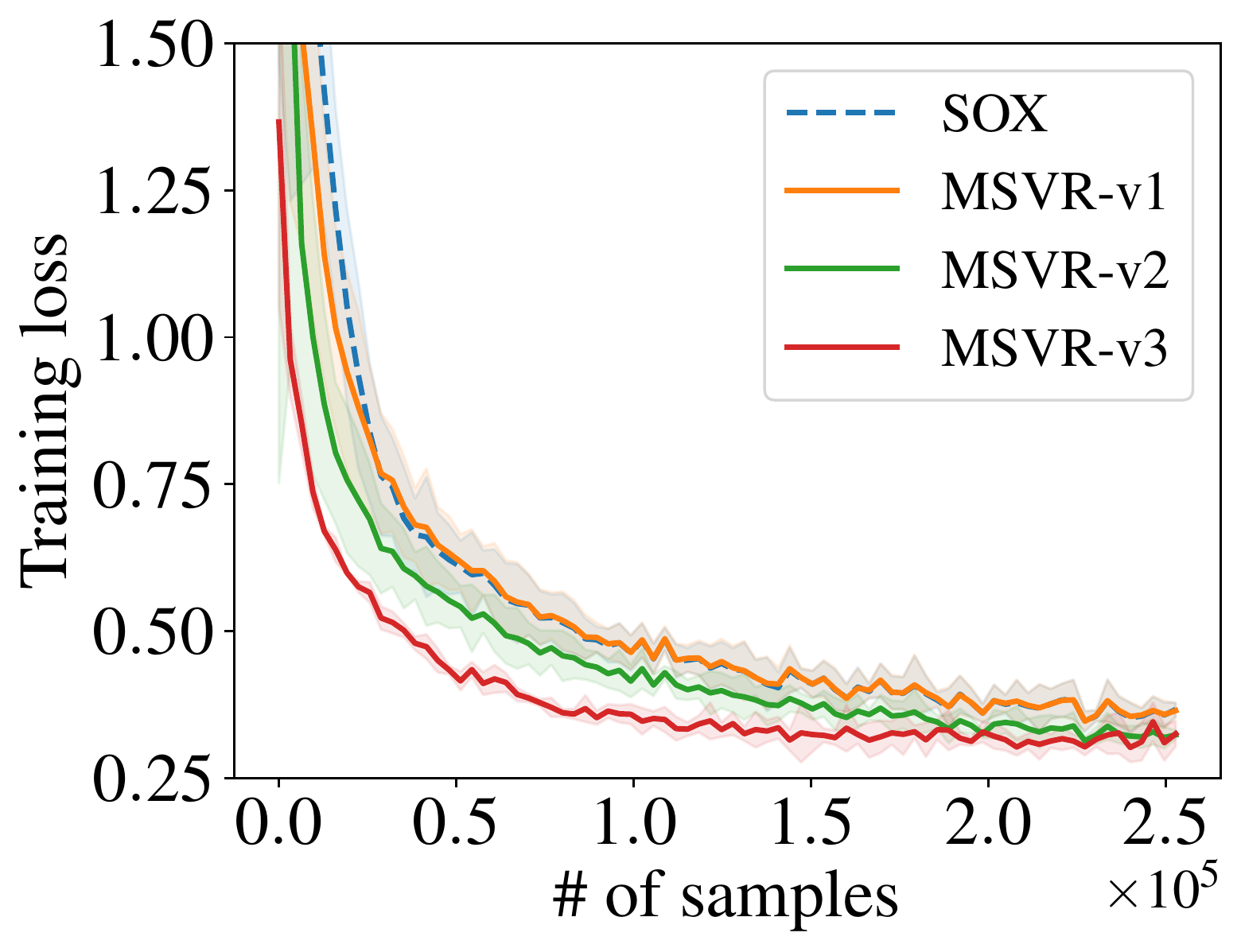}
	}	
	\subfigure[CIFAR10]{
		\includegraphics[width=0.31\textwidth]{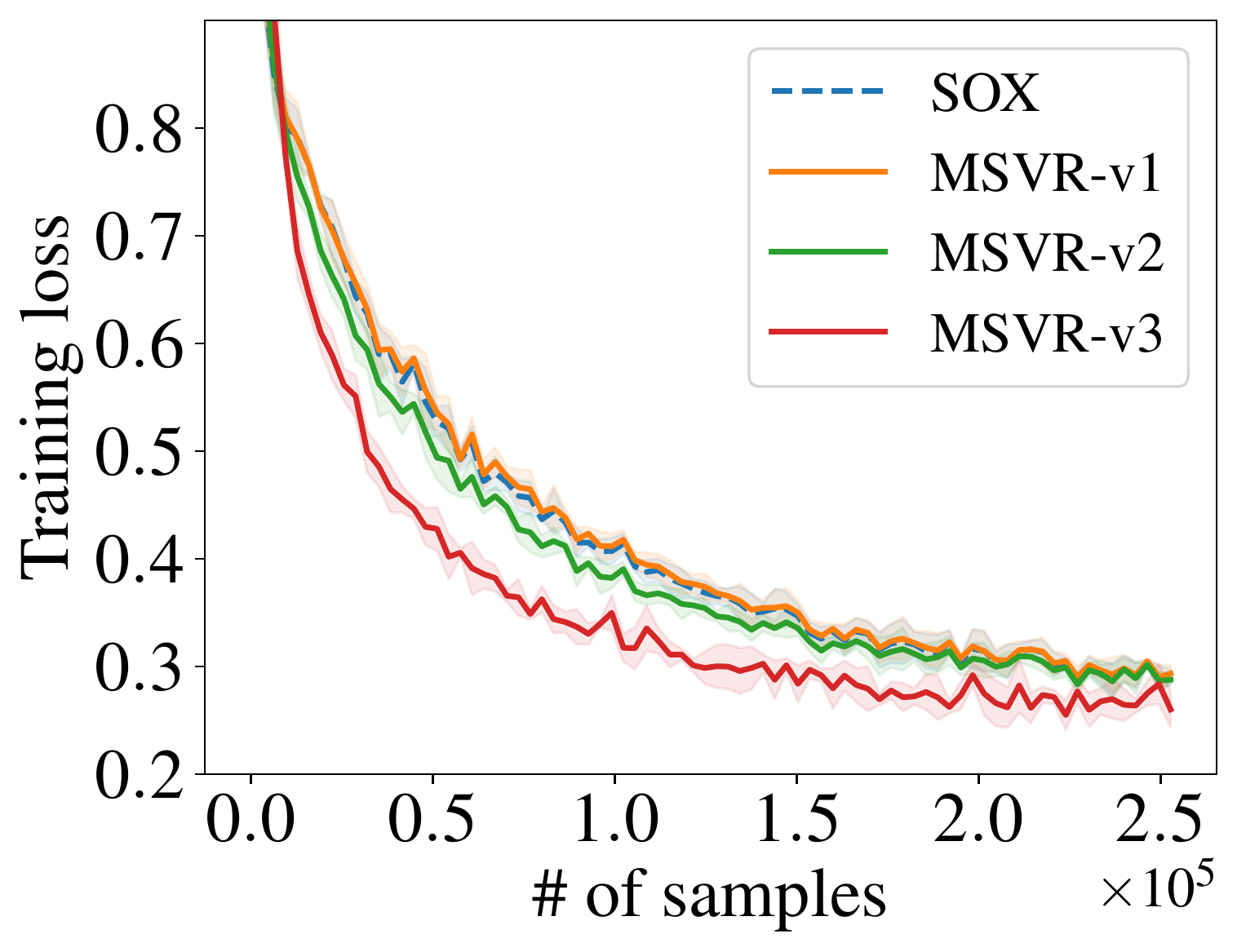}
	}
	\subfigure[CIFAR100]{
		\includegraphics[width=0.31\textwidth]{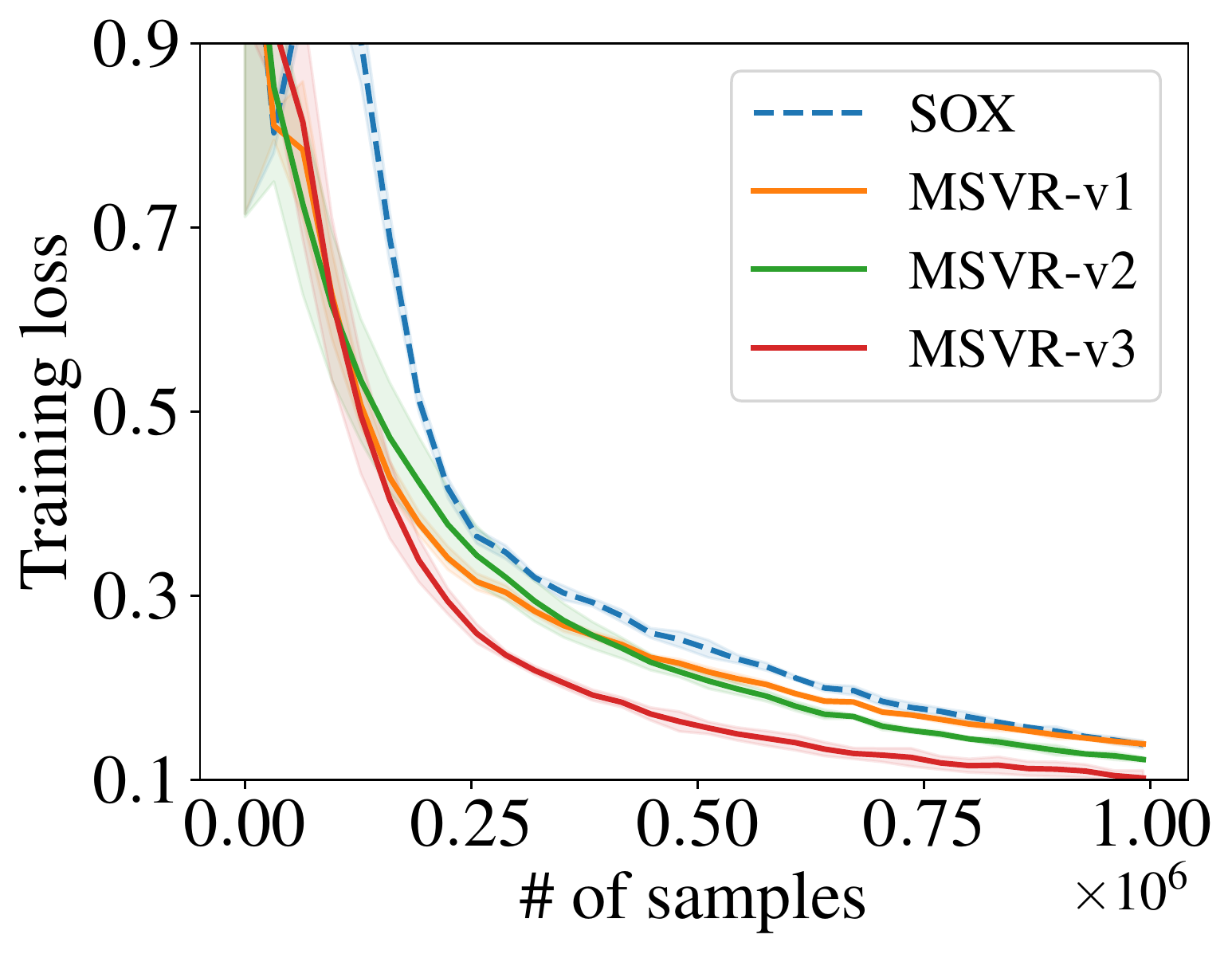}
	}
	\subfigure[MNIST]{
		\includegraphics[width=0.31\textwidth]{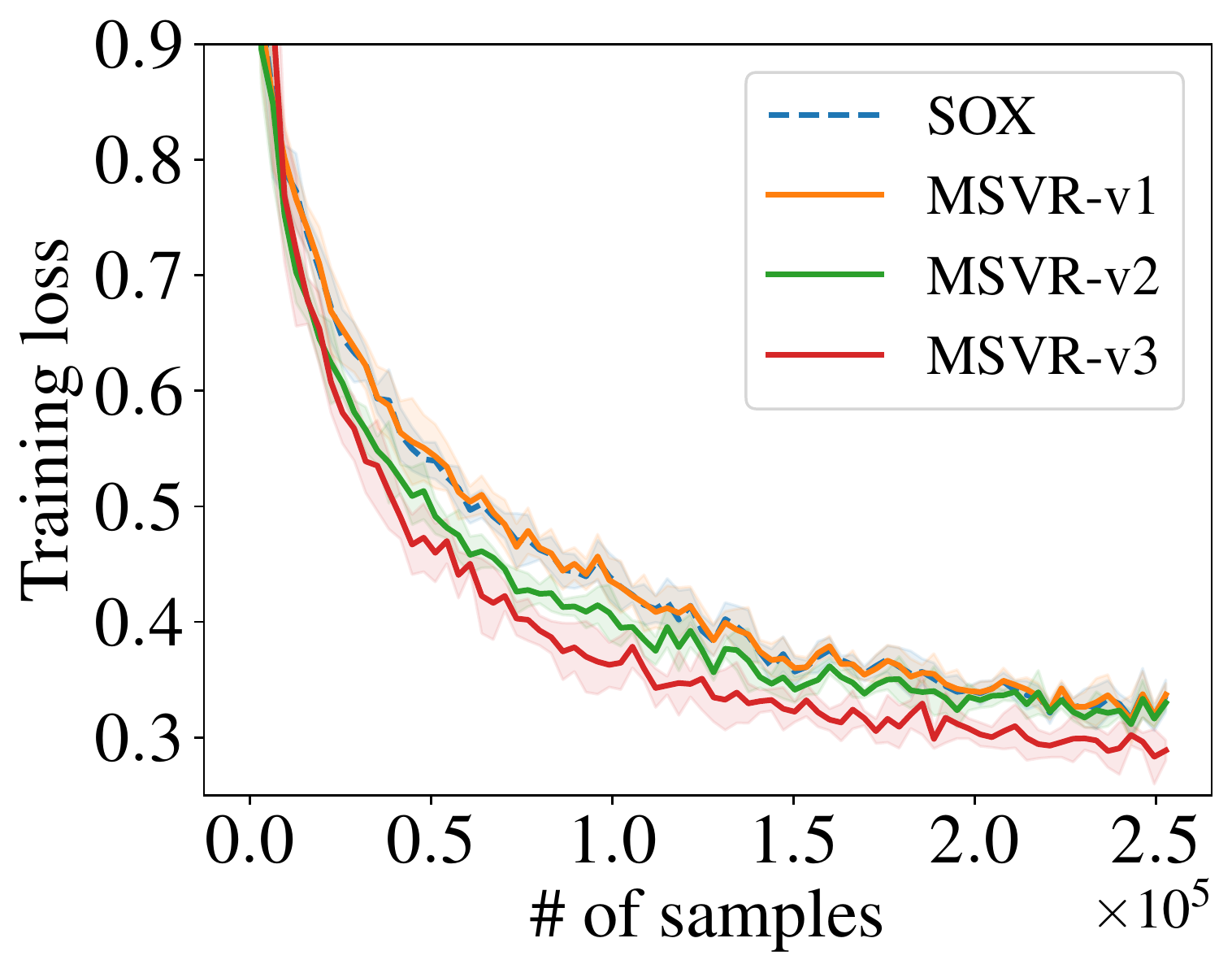}
	}
	\subfigure[Fashion-MNIST]{
		\includegraphics[width=0.31\textwidth]{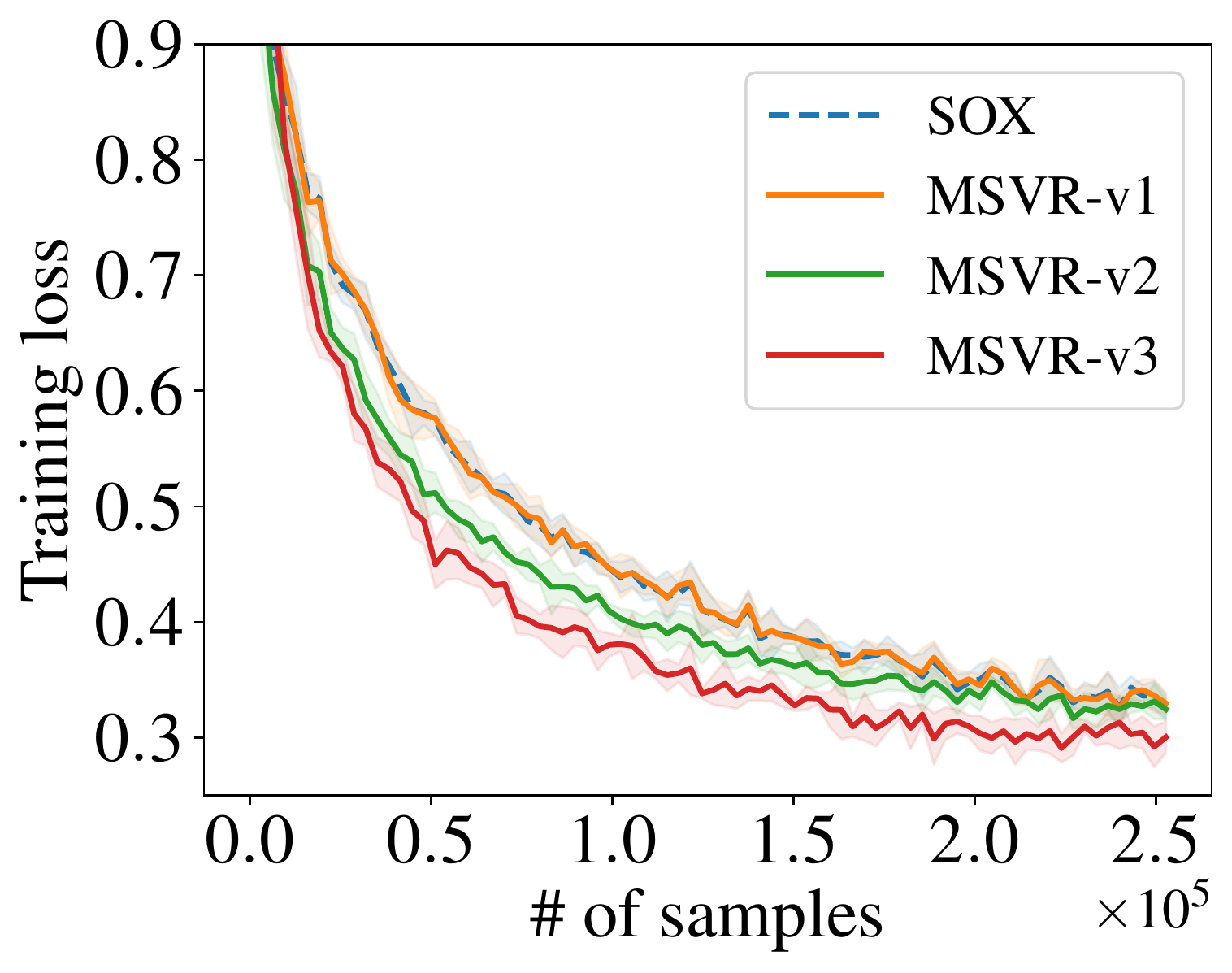}
	}
	\subfigure[SVHN]{
		\includegraphics[width=0.31\textwidth]{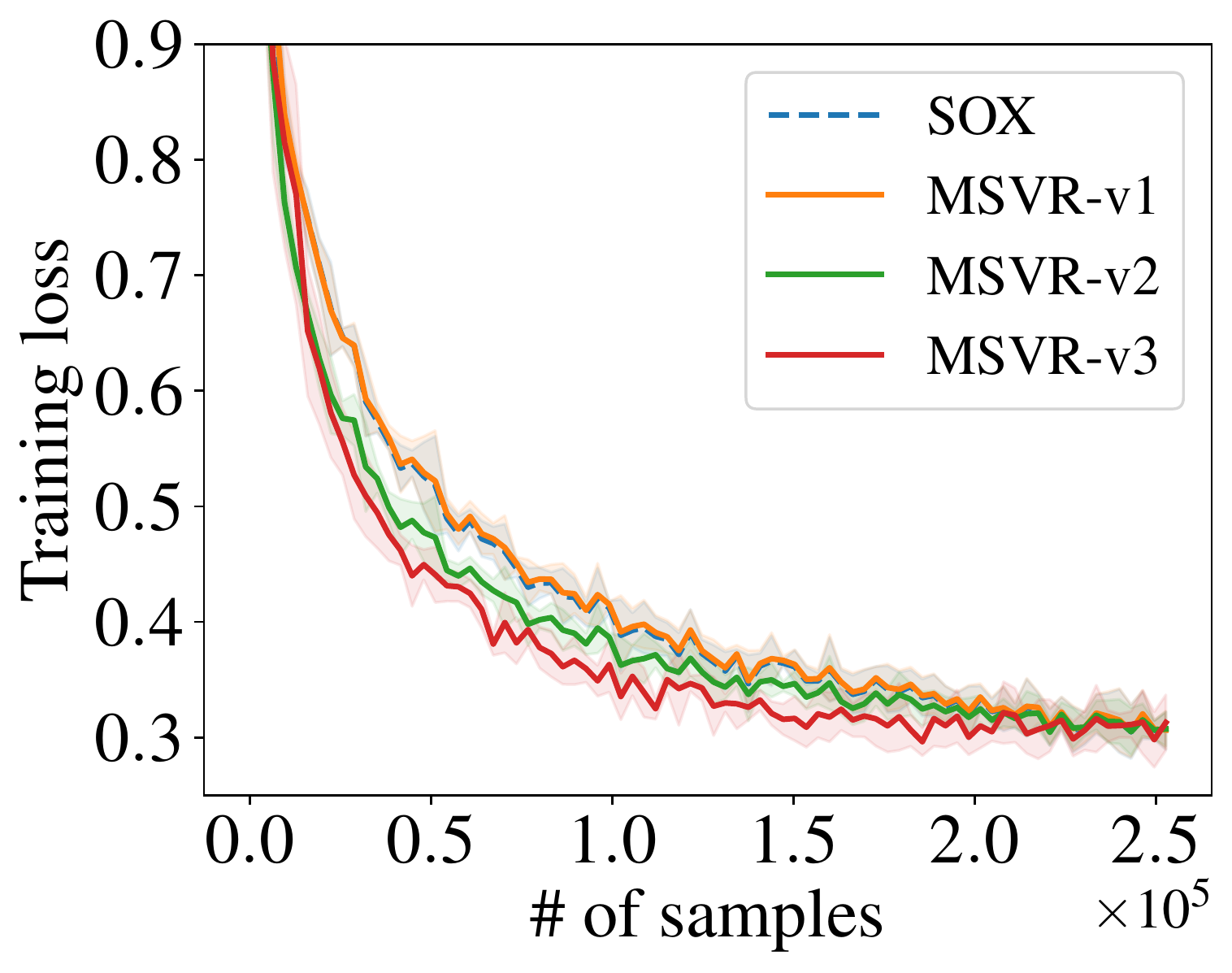}
	}
	\caption{Results for Multi-task AUC Optimization.}
	\label{fig:1}
\end{figure*}
\section{Experiments}\label{sec:exp}
In this section, we conduct experiments on the multi-task deep AUC maximization to evaluate the proposed methods and we will consider more applications in the long version of the paper. For binary classification (label $y = 1$ or $y = -1$),  AUC maximization can be formulated as minimizing the following composite loss~\citep{AUROC2022}:
\begin{equation*}
    \begin{split}
    \min _{\mathbf{w}, a, b} \mathrm{E}_{\mathbf{x} \mid y=1}\left[\left(h_{\mathbf{w}}(\mathbf{x})\right.\right.&\left.-a)^{2}\right]+\mathrm{E}_{\mathbf{x}^{\prime} \mid y^{\prime}=-1}\left[\left(h_{\mathbf{w}}\left(\mathbf{x}^{\prime}\right)-b\right)^{2}\right] +\ell(a(\mathbf{w})-b(\mathbf{w})),
    \end{split}
\end{equation*}
where $a(\mathbf{w})=\mathrm{E}\left[h_{\mathbf{w}}(\mathbf{x}) \mid y=1\right]$,  $b(\mathbf{w})=\mathrm{E}\left[h_{\mathbf{w}}(\mathbf{x}) \mid y=-1\right]$ and $\ell(\cdot)$ is a surrogate function. The above objective recovers the pairwise square loss and the min-max margin loss proposed by~\cite{robustdeepAUC} for deep AUC maximization by setting $\ell(\cdot)$ as the square function or squared hinge function, respectively.  When applied to  multi-task  classification (e.g., multiple classes), we can optimize the averaged AUC losses over all tasks, i.e., $AUC = \frac{1}{m} \sum_{i=1}^m AUC(i)$. The nested structure only comes from the term $\ell(a(\mathbf{w})-b(\mathbf{w}))$, and we can rewrite it as the form of FCCO problem, where
\begin{equation*}
    \begin{split}
        g_{i}(\mathbf{w})=\frac{1}{|\mathcal{D}^i_{+}|} \sum_{\mathbf{x} \in \mathcal{D}^i_{+}} h_{\mathbf{w}}(\mathbf{x})- \frac{1}{|\mathcal{D}^i_{-}|} \sum_{\mathbf{x} \in \mathcal{D}^i_{-}} h_{\mathbf{w}}(\mathbf{x}),\quad \quad f\left(g_{i}(\mathbf{w})\right)=\ell(g_{i}(\mathbf{w})).
    \end{split}
\end{equation*} 
where $\mathcal D^i_{+/-}$ denots the positive/negative datasets of the $i$-th task. 

\noindent{\bf Configurations.} In the experiment, we follow the setup in~\cite{AUROC2022} and set the surrogate function $\ell$ as squared hinge $\ell(x) =\frac{1}{2}(\max\{c+x,0\})^{2}$. We use ResNet18 as backbone network, and train on six datasets: STL10~\citep{Coates2011STL10}, CIFAR10~\citep{Krizhevsky2009Cifar10}, CIFAR100~\citep{Krizhevsky2009Cifar10}, MNIST~\citep{LeCun1998MNIST}, Fashion-MNIST~\citep{Xiao2017Fashion-MNIST}, and SVHN~\citep{Netzer2011SVHN}.  We compare our methods with previous SOTA algorithm SOX~\citep{dependent2022}.  For our methods, parameters $\alpha$ and $\beta$ are searched from $\{0.1, 0.5, 0.9, 1.0\}$. For SOX algorithm, its parameters $\beta$ and $\gamma$ are searched from the same set. $B_1$ is set as 50 for CIFAR100 and 5 for other datasets. Inner batch size $B_2$ is chosen as 128 for all methods. We tune the learning rate from the set $\{1e-4, 1e-3, 2e-3, 5e-3, 1e-2\}$ and pick the best one for each method. The experiments are conducted on single NVIDIA Tesla M40 GPU.

\noindent{\bf Results.} Figure~\ref{fig:1} shows the loss against the number of samples drawn by different methods, and all curves are averaged over 5 runs. We observe that MSVR-V1 is better than SOX on the CIFAR100 dataset, and close to it on other datasets. MSVR-v2 converges faster than SOX and MSVR-v1, and the loss of MSVR-v3 decreases most rapidly, demonstrating a low sample complexity.

\section{Conclusion and Future Work}
In this paper, we develop a novel MSVR estimator for tracking multiple functional mappings by probing only $\O(1)$ blocks. Equipped with this estimator, we design three algorithms for FCCO problems and obtain improved complexities across a spectrum of settings. Experimental results on multi-task deep AUC maximization also verify the effectiveness of our methods. In future work, we will investigate other applications that can be solved by using the proposed estimator.

\begin{ack}
W. Jiang, Y. Wang and L. Zhang were partially supported
by NSFC (62122037, 61921006). G. Li and T. Yang were partially supported by Amazon research award. The authors would like to thank the anonymous reviewers for their helpful comment.
\end{ack}

\bibliography{ref}
\bibliographystyle{abbrvnat}
\section*{Checklist}
\begin{enumerate}
\item For all authors...
\begin{enumerate}
  \item Do the main claims made in the abstract and introduction accurately reflect the paper's contributions and scope?
    \answerYes
  \item Did you describe the limitations of your work?
    \answerNo
  \item Did you discuss any potential negative societal impacts of your work?
    \answerNA This work is mainly theoretical.
  \item Have you read the ethics review guidelines and ensured that your paper conforms to them?
    \answerYes
\end{enumerate}

\item If you are including theoretical results...
\begin{enumerate}
  \item Did you state the full set of assumptions of all theoretical results?
    \answerYes See Section~\ref{assumption}.
        \item Did you include complete proofs of all theoretical results?
    \answerYes The complete proofs are provided in the supplementary.
\end{enumerate}

\item If you ran experiments...
\begin{enumerate}
  \item Did you include the code, data, and instructions needed to reproduce the main experimental results (either in the supplemental material or as a URL)?
    \answerNo 
  \item Did you specify all the training details (e.g., data splits, hyperparameters, how they were chosen)?
    \answerYes See Section~\ref{sec:exp}.
        \item Did you report error bars (e.g., with respect to the random seed after running experiments multiple times)?
    \answerYes See Section~\ref{sec:exp}
        \item Did you include the total amount of compute and the type of resources used (e.g., type of GPUs, internal cluster, or cloud provider)?
    \answerYes
\end{enumerate}

\item If you are using existing assets (e.g., code, data, models) or curating/releasing new assets...
\begin{enumerate}
  \item If your work uses existing assets, did you cite the creators?
    \answerYes
  \item Did you mention the license of the assets?
    \answerYes
  \item Did you include any new assets either in the supplemental material or as a URL?  \answerNA
  \item Did you discuss whether and how consent was obtained from people whose data you're using/curating?
    \answerNA We use the public benchmark datasets.
  \item Did you discuss whether the data you are using/curating contains personally identifiable information or offensive content?
    \answerNA We use the public benchmark datasets.
\end{enumerate}

\item If you used crowdsourcing or conducted research with human subjects...
\begin{enumerate}
  \item Did you include the full text of instructions given to participants and screenshots, if applicable?
    \answerNA
  \item Did you describe any potential participant risks, with links to Institutional Review Board (IRB) approvals, if applicable?
    \answerNA
  \item Did you include the estimated hourly wage paid to participants and the total amount spent on participant compensation?
    \answerNA
\end{enumerate}

\end{enumerate}


\appendix
\newpage
\section{Analysis }
\subsection{Proof of Lemma~\ref{lem:main1}} \label{p:L1}
\begin{proof}
Denote $\bar{\u}_t^i=(1-\beta_{t+1}) \u_{t}^{i}+\beta_{t+1} g_{i}(\w_{t+1};\xi_{t+1}^{i})+\gamma_{t+1} \left( g_{i}(\w_{t+1};\xi_{t+1}^{i}) -g_{i}(\w_{t};\xi_{t+1}^{i})\right)$. 
\begin{equation}\label{overall}
    \begin{split}
        \E \left[ \Norm{\u_{t+1}^i - g_i(\w_{t+1})}^2 \right] =& \E \left[(1-\frac{B_1}{m})\Norm{ \u_{t}^i - g_i(\w_{t+1})}^2 + \frac{B_1}{m} \Norm{\bar{\u}_t^i-g_{i}(\w_{t+1})}^2  \right]
        \end{split}
\end{equation}
For the first term, we can decompose it as:
\begin{equation}\label{first_term}
    \begin{split}
        &(1-\frac{B_1}{m})\E \left[\Norm{ \u_{t}^i - g_i(\w_{t+1})}^2 \right] \\
        =& (1-\frac{B_1}{m})\E \left[\Norm{ \u_{t}^i - g_i(\w_{t})}^2 \right]+ (1-\frac{B_1}{m}) \E \left[\Norm{ g_i(\w_{t}) - g_i(\w_{t+1})}^2 \right]  \\
        & + \underbrace{ 2 (1-\frac{B_1}{m}) \E \left[\left( \u_{t}^i - g_i(\w_{t}) \right) \left( g_i(\w_{t}) - g_i(\w_{t+1}) \right)  \right]} \limits_{\text{\textcircled{1}}}
        \end{split}
\end{equation}
Also, the second term can be written as:
\begin{equation}\label{second_term}
    \begin{split}
        &\frac{B_1}{m}\E \left[  \Norm{\bar{\u}_t^i-g_{i}(\w_{t+1})}^2  \right] \\
        = & \frac{B_1}{m}\E \left[  \left\|(1-\beta_{t+1}) \left(\u_{t}^{i} - g_i(\w_t)\right) + (1-\beta_{t+1}) \left(g_i(\w_t) - g_i(\w_{t+1})\right)  \right.\right. \\
        & \left.\left. + \beta_{t+1} \left(g_{i}(\w_{t+1};\xi_{t+1}^{i}) - g_i(\w_{t+1})\right)+\gamma_{t+1} \left( g_{i}(\w_{t+1};\xi_{t+1}^{i}) -g_{i}(\w_{t};\xi_{t+1}^{i})\right)\right\|^2  \right] \\
        =  & \frac{B_1}{m}\E \left[  \left\|(1-\beta_{t+1}) \left(\u_{t}^{i} - g_i(\w_t)\right)  + (1-\beta_{t+1}) \left(g_i(\w_t) - g_i(\w_{t+1})\right) \right.\right. \\
        & \left.\left.  +\gamma_{t+1} \left( g_{i}(\w_{t+1};\xi_{t+1}^{i}) -g_{i}(\w_{t};\xi_{t+1}^{i})\right)\right\|^2\right] +\frac{B_1 \beta_{t+1}^2}{m}\E\left[\Norm{ g_{i}(\w_{t+1};\xi_{t+1}^{i}) - g_i(\w_{t+1})}^2\right] \\
        & + \frac{2B_1 \beta_{t+1} \gamma_{t+1}}{m}\E\left[ \left(g_{i}(\w_{t+1};\xi_{t+1}^{i}) - g_i(\w_{t+1})\right) \left( g_{i}(\w_{t+1};\xi_{t+1}^{i}) -g_{i}(\w_{t};\xi_{t+1}^{i})\right) \right]\\
        = &  \frac{B_1 }{m} \E \left[(1-\beta_{t+1})^2  \left\| \left(\u_{t}^{i} - g_i(\w_t)\right)\right\|^2 \right]\\
        &+ \frac{B_1 }{m}\E \left[\left\|(1-\beta_{t+1}) \left(g_i(\w_t) - g_i(\w_{t+1})\right) +\gamma_{t+1} \left( g_{i}(\w_{t+1};\xi_{t+1}^{i}) -g_{i}(\w_{t};\xi_{t+1}^{i})\right)\right\|^2\right]  \\
        & +\underbrace{\frac{2B_1 }{m} (1-\beta_{t+1})(1-\beta_{t+1}-\gamma_{t+1}) \E \left[   \left(\u_{t}^{i} - g_i(\w_t)\right) \left( g_{i}(\w_{t}) -g_{i}(\w_{t+1})\right)\right]}\limits_{\text{\textcircled{2}}}  \\
        & +\frac{B_1 \beta_{t+1}^2}{m}\E\left[\Norm{ \left(g_{i}(\w_{t+1};\xi_{t+1}^{i}) - g_i(\w_{t+1})\right)}^2\right] \\
        & + \frac{2B_1 \beta_{t+1} \gamma_{t+1}}{m}\E\left[ \left(g_{i}(\w_{t+1};\xi_{t+1}^{i}) - g_i(\w_{t+1})\right) \left( g_{i}(\w_{t+1};\xi_{t+1}^{i}) -g_{i}(\w_{t};\xi_{t+1}^{i})\right) \right] 
        \end{split} 
\end{equation}
 The second equation is because of $\E \left[g_{i}(\w_{t+1};\xi_{t+1}^{i}) - g_i(\w_{t+1})\right] = 0$. The last equation is due to $\E \left[ g_{i}(\w_{t+1};\xi_{t+1}^{i})- g_{i}(\w_{t};\xi_{t+1}^{i})\right] = g_{i}(\w_{t+1}) - g_{i}(\w_{t})$ and $\u_t^i$ is independent of $\xi_{t+1}^{i}$. We want to ensure $\text{\textcircled{1}} + \text{\textcircled{2}}=0$, which requires that $2 (1-\frac{B_1}{m}) + \frac{2B_1 }{m} (1-\beta_{t+1})(1-\beta_{t+1}-\gamma_{t+1}) =0$. Solve $\gamma_{t+1}$ and we have $\gamma_{t+1} = \frac{m-B_1}{B_1(1-\beta_{t+1})}+(1-\beta_{t+1})$. According to equation~(\ref{first_term}) and equation~(\ref{second_term}), the equation~(\ref{overall}) can now be written as:
\begin{equation}\label{overall_2}
    \begin{split}
        &\E \left[ \Norm{\u_{t+1}^i - g_i(\w_{t+1})}^2 \right] \\
        =& \E \bigg[\left(1-\frac{B_1}{m} + \frac{(1-\beta_{t+1})^2B_1 }{m}\right)\Norm{ \u_{t}^i - g_i(\w_{t})}^2 + (1-\frac{B_1}{m})\Norm{ g_i(\w_{t}) - g_i(\w_{t+1})}^2  \\
        &+\frac{B_1 \beta_{t+1}^2}{m}\Norm{ \left(g_{i}(\w_{t+1};\xi_{t+1}^{i}) - g_i(\w_{t+1})\right)}^2\\
        &+ \frac{B_1 }{m}\left\|(1-\beta_{t+1}) \left(g_i(\w_t) - g_i(\w_{t+1})\right) +\gamma_{t+1} \left( g_{i}(\w_{t+1};\xi_{t+1}^{i}) -g_{i}(\w_{t};\xi_{t+1}^{i})\right)\right\|^2  \\
        &+ \frac{2B_1 \beta_{t+1} \gamma_{t+1}}{m} \left(g_{i}(\w_{t+1};\xi_{t+1}^{i}) - g_i(\w_{t+1})\right) \left( g_{i}(\w_{t+1};\xi_{t+1}^{i}) -g_{i}(\w_{t};\xi_{t+1}^{i})\right) \bigg] \\
        = & \E \bigg[\left(1-\frac{B_1}{m} + \frac{(1-\beta_{t+1})^2B_1 }{m}\right)\Norm{ \u_{t}^i - g_i(\w_{t})}^2 + (1-\frac{B_1}{m}) \Norm{ g_i(\w_{t}) - g_i(\w_{t+1})}^2  \\
        &+\frac{B_1 \beta_{t+1}^2}{m}\Norm{ \left(g_{i}(\w_{t+1};\xi_{t+1}^{i}) - g_i(\w_{t+1})\right)}^2\ + \frac{B_1 (1-\beta_{t+1})^2}{m}\left\| \left(g_i(\w_t) - g_i(\w_{t+1})\right)\right\|^2 \\
        & - \frac{2 B_1 (1-\beta_{t+1}) }{m}\gamma_{t+1}\left\| \left(g_i(\w_t) - g_i(\w_{t+1})\right)\right\|^2\\
        & +\frac{B_1\gamma_{t+1}^2}{m} \left\|\left( g_{i}(\w_{t+1};\xi_{t+1}^{i}) -g_{i}(\w_{t};\xi_{t+1}^{i})\right)\right\|^2  \\
        & + \frac{2B_1 \beta_{t+1} \gamma_{t+1}}{m} \left(g_{i}(\w_{t+1};\xi_{t+1}^{i}) - g_i(\w_{t+1})\right) \left( g_{i}(\w_{t+1};\xi_{t+1}^{i}) -g_{i}(\w_{t};\xi_{t+1}^{i})\right) \bigg]\\
        \leq & \left(1-\frac{\beta_{t+1} B_1}{m} \right)\E \left[\Norm{ \u_{t}^i - g_i(\w_{t})}^2 \right] + \frac{B_1 \beta_{t+1}^2}{m}\E \left[\Norm{ \left(g_{i}(\w_{t+1};\xi_{t+1}^{i}) - g_i(\w_{t+1})\right)}^2\right]\\
        &+\frac{4 m C_g^2}{B_1} \E\left[\left\| \w_{t+1} -\w_{t}\right\|^2\right] \\
        & + \frac{2B_1 \beta_{t+1} \gamma_{t+1}}{m}\E\left[ \left(g_{i}(\w_{t+1};\xi_{t+1}^{i}) - g_i(\w_{t+1})\right) \left( g_{i}(\w_{t+1};\xi_{t+1}^{i}) -g_{i}(\w_{t};\xi_{t+1}^{i})\right) \right] \\
        \leq & \left(1-\frac{\beta_{t+1} B_1}{m} \right)\E \left[\Norm{ \u_{t}^i - g_i(\w_{t})}^2 \right] + \frac{2B_1 \beta_{t+1}^2}{m}\E \left[\Norm{ \left(g_{i}(\w_{t+1};\xi_{t+1}^{i}) - g_i(\w_{t+1})\right)}^2\right]\\
        &+\frac{8 m C_g^2}{B_1} \E\left[\left\| \w_{t+1} -\w_{t}\right\|^2 \right]
        \end{split}
\end{equation}
The second equation is due to 
\begin{equation*}
    \begin{split}
    \E\left[ \left(g_{i}(\w_{t+1};\xi_{t+1}^{i}) -g_{i}(\w_{t};\xi_{t+1}^{i})\right) \left(g_{i}(\w_{t+1}) -g_{i}(\w_{t})\right)\right] = \Norm{ g_{i}(\w_{t+1}) -g_{i}(\w_{t})}^2.
    \end{split}
\end{equation*}
The first inequality is because of $\gamma_{t+1} \leq \frac{2m}{B_1}$ (since $\beta_{t+1} \leq \frac{1}{2}$) and $1-\frac{B_1}{m} + \frac{B_1}{m}(1-\beta_{t+1})^2 \leq 2\frac{B_1}{m}(1-\beta_{t+1})\gamma_{t+1}$.
The last inequality is due to the fact that
\begin{equation*}
\begin{split} 
    &\frac{2 B_1 \beta_{t+1} \gamma_{t+1}}{m}\E\left[\left(g_{i}(\w_{t+1};\xi_{t+1}^{i})-g_{i}(\w_{t+1})\right) \left(g_{i}(\w_{t+1};\xi_{t+1}^{i}) - g_{i}(\w_{t};\xi_{t+1}^{i})\right)\right]  \\
    \leq & \frac{2 B_1 \beta_{t+1} \gamma_{t+1}}{m}\E\left[\left\|g_{i}(\w_{t+1};\xi_{t+1}^{i})-g_{i}(\w_{t+1})\right\| \left\|g_{i}(\w_{t+1};\xi_{t+1}^{i}) - g_{i}(\w_{t};\xi_{t+1}^{i})\right\|\right]   \\
    \leq & \frac{B_1 \beta_{t+1}^2}{m} \E\left[\left\|g_{i}(\w_{t+1};\xi_{t+1}^{i})-g_{i}(\w_{t+1})\right\|^{2}\right]+ \frac{B_1 \gamma_{t+1}^2}{m}\E\left[\left\|g_{i}(\w_{t+1};\xi_{t+1}^{i}) - g_{i}(\w_{t};\xi_{t+1}^{i})\right\|^2\right] \\
    \leq & \frac{B_1 \beta_{t+1}^2}{m} \E\left[\left\|g_{i}(\w_{t+1};\xi_{t+1}^{i})-g_{i}(\w_{t+1})\right\|^{2}\right]+ \frac{4 m C_g^2}{B_1}\E\left[\left\|\w_{t+1} - \w_{t}\right\|^2\right]
\end{split}
\end{equation*}
Finally,  we have:
\begin{equation*}
    \begin{split}
        &\E \left[ \Norm{\u_{t+1} - g(\w_{t+1})}^2 \right]= \sum_{i=1}^{m}\E \left[\Norm{ \u_{t+1}^i - g_i(\w_{t+1})}^2 \right] \\
        \leq & \left(1-\frac{\beta_{t+1} B_1}{m} \right)\E \left[\Norm{ \u_{t} - g(\w_{t})}^2 \right] +\frac{8 m^2 C_g^2}{B_1} \E \left[\left\| \w_{t+1} -\w_{t}\right\|^2\right] + \frac{2 B_1 \sigma^2 \beta_{t+1}^2}{B_2} 
        \end{split}
\end{equation*}
\end{proof}

\subsection{Proof of Lemma~\ref{lem:main3}}
\begin{proof}
We will start with single block and them sum over multiple blocks. To start, we have
\begin{align*}
&\Norm{\u^i_{t+1} - g_i(\w_{t+1})}^2\\
=& \left(1 - \frac{B_1}{m}\right)\Norm{\u^i_t - g_i(\w_{t+1})}^2 \\
+& \frac{B_1}{m}\Norm{(1-\beta_{t+1}) \u^i_t  + \beta_{t+1}  g_i(\w_{t+1};\xi_{t+1}^i) + \gamma_{t+1} \nabla g_i(\w_{t+1};\xi_{t+1}^i)^{\top}(\w_{t+1} - \w_{t}) - g_i(\w_{t+1})}^2\\
=&\left(1 - \frac{B_1}{m}\right)\left(\Norm{\u^i_t - g_i(\w_{t})}^2 + \underbrace{2(\u^i_t - g_i(\w_{t}))^{\top}(g_i(\w_{t}) - g_i(\w_{t+1}))}\limits_{A_0} + \Norm{g_i(\w_{t}) - g_i(\w_{t+1})}^2\right) \\
 +& \frac{B_1}{m}\underbrace{\Norm{(1-\beta_{t+1}) \u^i_t  + \beta_{t+1} g_i(\w_{t+1};\xi_{t+1}^i) + \gamma_{t+1} \nabla g_i(\w_{t+1};\xi_{t+1}^i)^{\top}(\w_{t+1} - \w_{t}) - g_i(\w_{t+1})}^2}\limits_{A}\\
\end{align*}
Next, we will proceed to decompose $A$. 
\begin{align*}
A  =&  \bigg\|(1-\beta_{t+1}) (\u^i_t - g_i(\w_{t}))  + (1-\beta_{t+1})(g_i(\w_{t}) - g_i(\w_{t+1}))  \\ 
&\quad+ \beta_{t+1} ( g_i(\w_{t+1};\xi_{t+1}^i) - g_i(\w_{t+1}))  +  \gamma_{t+1} (g_i(\w_{t+1}) - g_i(\w_{t}))\\ 
&\quad +\gamma_{t+1}( g_i(\w_{t}) - g_i(\w_{t+1}) + \nabla  g_i(\w_{t+1};\xi_{t+1}^i)^{\top}(\w_{t+1} - \w_{t}) \bigg\|^2\\
=&  \bigg\|(1-\beta_{t+1}) (\u^i_t - g_i(\w_{t}))  + (\gamma_{t+1} + \beta_{t+1} -1) (g_i(\w_{t+1}) - g_i(\w_{t}))   \\
&  \quad + \gamma_{t+1}\underbrace{( g_i(\w_{t}) - g_i(\w_{t+1}) + \nabla g_i(\w_{t+1})^{\top}(\w_{t+1} - \w_{t}))}\limits_{\Delta_t} \\
&  \quad + \beta_{t+1} (g_i(\w_{t+1};\xi_{t+1}^i) - g_i(\w_{t+1}))\\
& \quad +\gamma_{t+1}(\nabla  g_i(\w_{t+1};\xi_{t+1}^i)- \nabla  g_i(\w_{t+1}))^{\top}(\w_{t+1} - \w_{t}) \bigg\|^2
\end{align*}
By taking expectation over $A$, we have
\begin{align*}
\E[A]  =& \E\bigg[(1-\beta_{t+1})^2 \|\u^i_t - g_i(\w_{t})\|^2  +  \gamma_{t+1}^2\|\Delta_t\|^2+  (\gamma_{t+1} + \beta_{t+1} -1)^2 \|g_i(\w_{t+1}) - g_i(\w_{t})\|^2\\
& +\underbrace{ 2 (1-\beta_{t+1})(\gamma_{t+1}+\beta_{t+1}-1)  (\u^i_t - g_i(\w_{t})^{\top}(g_i(\w_{t+1}) - g_i(\w_{t}))}\limits_{A_1} \\
& + 2\gamma_{t+1} (\gamma_{t+1}+\beta_{t+1}-1)\Delta_t^{\top}(g_i(\w_{t+1}) - g_i(\w_{t}))\\
&+2(1-\beta_{t+1})\gamma_{t+1}(\u^i_t - g_i(\w_{t}))^{\top}\Delta_t  + \frac{2\beta_{t+1}^2\sigma^2}{B_2} + \frac{2\gamma_{t+1}^2 \sigma^2}{B_2}\|\w_{t+1} - \w_{t}\|^2\bigg]
\end{align*}
Since $\gamma_{t+1} + \beta_{t+1} -1 = \frac{m-B_1}{B_1(1-\beta_{t+1})}$, the terms involving $A_0, A_1$ will cancel. As a result, we have
\begin{align*}
&\E\left[\Norm{\u^i_{t+1} - g_i(\w_{t+1})}^2\right] \\
=&\E \bigg[\left(1 - \frac{B_1}{m}\right)\Norm{\u^i_t - g_i(\w_{t})}^2  + \Norm{g_i(\w_{t}) - g_i(\w_{t+1})}^2 \\
& + \frac{B_1}{m}(1-\beta_{t+1})^2(1+\beta_{t+1}) \|\u^i_t - g_i(\w_{t})\|^2 +(2+\frac{1}{\beta_{t+1}}) \gamma_{t+1}^2\|\Delta_t\|^2  \\
&+  2(\gamma_{t+1} + \beta_{t+1} -1)^2 \|g_i(\w_{t+1}) - g_i(\w_{t})\|^2 +\frac{2\beta_{t+1}^2\sigma^2}{B_2} + \frac{2 \gamma_{t+1}^2 \sigma^2}{B_2}\|\w_{t+1} - \w_{t}\|^2 \bigg]\\
\end{align*}
Since $\|\Delta_t\|\leq \min(\frac{L_g}{2}\|\w_{t+1} - \w_{t}\|^2, 2C_g \|\w_{t+1} - \w_{t}\|)$ and $\|g_i(\w_{t+1}) - g_i(\w_{t})\|\leq C_g \|\w_{t+1} - \w_{t}\|^2$, we have
\begin{align*}
&\E\left[\Norm{\u^i_{t+1} - g_i(\w_{t+1})}^2\right] \\
=&\E\bigg[\left(1 - \frac{B_1}{m}\right)(\u^i_t - g_i(\w_{t}))^2  + C_g^2\|\w_{t+1}  - \w_t\|^2 \\
& + \frac{B_1}{m}(1-\beta_{t+1}) \|\u^i_t - g_i(\w_{t})\|^2 +(2+\frac{1}{\beta_{t+1}}) \gamma_{t+1}^2\frac{L_g^2\|\w_{t+1} - \w_{t}\|^2}{4}\|\w_{t+1} - \w_{t}\|^2 \\
& +  2(\gamma_{t+1} + \beta_{t+1} -1)^2 C_g^2\|\w_{t+1} - \w_{t}\|^2 +\frac{2\beta_{t+1}^2\sigma^2}{B_2} + \frac{2\gamma_{t+1}^2 \sigma^2}{B_2}\|\w_{t+1} - \w_{t}\|^2 \bigg]\\
\end{align*}
Note that we have $\|\w_{t+1} - \w_t\|^2\leq \eta_{t+1}^2 C_F^2$, $\gamma_{t+1}\leq \frac{2m}{B_1}$ and $\eta_{t+1} \leq \sqrt{\beta_{t+1}}$. Therefore 
\begin{align*}
&\E\left[\Norm{\u^i_{t+1} - g_i(\w_{t+1})}^2\right] \\
&\leq \left(1 - \frac{B_1\beta_{t+1}}{m}\right)\E\left[\Norm{\u^i_t - g_i(\w_{t})}^2 \right] + \frac{2B_1\beta_{t+1}^2\sigma^2}{m B_2}  \\
&\quad + \E\left[\frac{4m L_g^2}{\beta_{t+1} B_1}\|\w_{t+1} - \w_{t}\|^4+  \frac{9m C_g^2}{B_1}\|\w_{t+1} - \w_{t}\|^2 +  \frac{8m \sigma^2}{B_1B_2}\|\w_{t+1} - \w_{t}\|^2 \right] \\
&\leq \left(1 - \frac{B_1\beta_{t+1}}{m}\right)\E\left[\Norm{\u^i_t - g_i(\w_{t})}^2\right]  + \frac{2B_1\beta_{t+1}^2\sigma^2}{m B_2}   \\
&\quad + \left(4L_g^2C_F^2 + 9 C_g^2 +\frac{8 \sigma^2}{B_2}\right)\frac{m}{B_1}\E\left[\|\w_{t+1} - \w_{t}\|^2\right]
\end{align*}
Finally, we have:
\begin{align*}
\E[\Norm{\u_{t+1} - g(\w_{t+1})}^2] \leq & \left(1 - \frac{B_1\beta_{t+1}}{m}\right)\E\left[\Norm{\u_t - g(\w_{t})}^2\right]  + \frac{2B_1\beta_{t+1}^2\sigma^2}{B_2}  \\
&+ \left(4L_g^2C_F^2 + 9 C_g^2 +\frac{8 \sigma^2}{B_2}\right)\frac{m^2}{B_1}\E\left[\|\w_{t+1} - \w_{t}\|^2\right]
\end{align*}
\end{proof}
\subsection{Proof of Lemma~\ref{lem:main2}}
\begin{proof}
Note that we have:
\begin{equation}
\begin{split} \label{eqn:tau1}
    &\E\left[\left\| \widehat g_{i}(\w_{t+1} ; \xi_{t+1}^i) - g_{i}(\w_{t+1})\right\|^{2}\right]   \\
    =& \E\left[\left\| g_{i}(\w_{t+1}; \xi_{t+1}^i) -   g_{i}(\w_{\tau}; \xi_{t+1}^i) + g_{i}(\w_{\tau}) - g_{i}(\w_{t+1})\right\|^{2}\right]  \\
    =& \E\left[\left\| g_{i}(\w_{t+1}; \xi_{t+1}^i) -   g_{i}(\w_{\tau}; \xi_{t+1}^i) \right\|^{2}\right]+\left\| g_{i}(\w_{\tau}) - g_{i}(\w_{t+1})\right\|^{2}  \\
    &\quad\quad\quad +2\E\left[ g_{i}(\w_{t+1}; \xi_{t+1}^i) -   g_{i}(\w_{\tau}; \xi_{t+1}^i)\right] \left[ g_{i}(\w_{\tau}) - g_{i}(\w_{t+1})\right]  \\
    =& \E\left[\left\| g_{i}(\w_{t+1}; \xi_{t+1}^i) -   g_{i}(\w_{\tau}; \xi_{t+1}^i) \right\|^{2}\right]-\left\| g_{i}(\w_{\tau}) - g_{i}(\w_{t+1})\right\|^{2}  \\
    \leq& C_g^2\left\|\w_{t+1}-\w_{\tau}\right\|^{2}
\end{split}
\end{equation}
Since $\tau$ is the closest small index such that $\tau$ mod $I$ = 0, we have:
\begin{equation}
\begin{split}\label{eqn:tau2}
    \sum_{t=1}^{T} \left\|\w_{t+1}-\w_{\tau}\right\|^{2} &\leq \sum_{t=1}^T\left\|\sum_{k=\tau}^t (\w_{k+1}-\w_k)\right\|^2   \\
    &\leq  \sum_{t=1}^T \sum_{k=\tau}^t I \left\| \w_{k+1}-\w_k\right\|^2   
    \leq  I^2 \sum_{t=1}^T   \left\| \w_{t+1}-\w_t\right\|^2
\end{split}
\end{equation}
We can then apply the same analysis as in Section~\ref{p:L1}, until equation~(\ref{overall_2}):
\begin{equation*}
    \begin{split}
        &\E \left[ \Norm{\u_{t+1}^i - g_i(\w_{t+1})}^2 \right] \\
        \leq & \left(1-\frac{\beta B_1}{m} \right)\E \left[\Norm{ \u_{t}^i - g_i(\w_{t})}^2 \right] + \frac{2 B_1 \beta^2}{m}\E \left[\Norm{ \left(\widehat g_{i}(\w_{t+1};\xi_{t+1}^{i}) - g_i(\w_{t+1})\right)}^2\right]\\
        &+\frac{8 m C_g^2}{B_1} \E \left[\left\| \w_{t+1} -\w_{t}\right\|^2\right]\\
        \leq & \left(1-\frac{\beta B_1}{m} \right)\E \left[\Norm{ \u_{t}^i - g_i(\w_{t})}^2 \right] + \left(\frac{2 B_1 C_g^2\beta^2 I^2 }{m}+\frac{8 m C_g^2}{B_1} \right)\E \left[\left\| \w_{t+1} -\w_{t}\right\|^2\right] \\
        \leq & \left(1-\frac{\beta B_1}{m} \right)\E \left[\Norm{ \u_{t}^i - g_i(\w_{t})}^2 \right] +\frac{10 m C_g^2}{B_1} \E \left[\left\| \w_{t+1} -\w_{t}\right\|^2\right] 
        \end{split}
\end{equation*}
The last inequality is due to $\beta I \leq \frac{m}{B_1}$. Finally, we have:
\begin{equation*}
\begin{split}
\E\left[\left\|\u_{t+1}-g\left(\w_{t+1}\right)\right\|^{2}\right]  = & \sum_{i=1}^{m}\E \left[\Norm{ \u_{t+1}^i - g_i(\w_{t+1})}^2 \right]\\
\leq &(1-\frac{B_1\beta}{m})\E\left[\left\|\u_{t}-g\left(\w_{t}\right)\right\|^{2}\right]+\frac{10 m^{2} C_g^2}{B_1}\left\|\w_{t+1}-\w_{t}\right\|^{2}
\end{split}
\end{equation*}
\end{proof}

\subsection{Proof of Theorem~\ref{thm:1}}
We denote constant $C = \max\left\{1,C_g^2,L_F^2, C_F^2,\sigma^2, L_f^2C_g^2,L_g^2C_f^2,L_f^2C_g^4, L_f^2C_g^2\sigma^2, C_f^2(\sigma^2+C_g^2)  \right\}$.
\begin{lemma} \label{lem:1}(Lemma 2 in \cite{pmlr-v139-li21a}) Suppose function F is ${L_F}$-smooth and consider the update $\w_{t+1}:=\w_{t}-\eta_t \z_{t}$. With $\eta_t L\leq \frac{1}{2}$, we have: 
\begin{equation*}
\begin{split}
 F(\w_{t+1}) \leq F(\w_t) - \frac{\eta_t}{2} \|\nabla F(\w_t)\|^2 + \frac{\eta_t}{2} \|\z_{t} - \nabla F(\w_t)\|^2 - \frac{\eta_t}{4} \Norm{\z_t}^2
\end{split}
\end{equation*}
\end{lemma}

\begin{lemma}\label{lem:2}Denote $\|\u_{t}  - g(\w_t)\|^2 = \sum_{i=1}^m\|\u^{i}_{t}  - g_{i}(\w_t)\|^2$ and $\|\u_{t}  - \u_{t-1}\|^2 = \sum_{i=1}^m\|\u^{i}_{t}  - \u^{i}_{t-1}\|^2$.
\begin{equation*}
\begin{split}
&\mathbb{E}\left[ \|\z_{t+1} - \nabla F(\w_{t+1})\|^2 \right] \leq (1-\alpha_{t+1}) \mathbb{E}\left[ \|\z_{t} - \nabla F(\w_{t})\|^2 \right]+\frac{3 C \eta_t^{2} \mathbb{E}\left[\left\|\mathbf{z}_{t}\right\|^{2}\right]}{\alpha_{t+1}} \\
&\quad\quad +\frac{4 L_{f}^{2} C_{g}^{2}}{m} \mathbb{E}\left[\left\|\u_{t+1}-\u_{t}\right\|^{2}\right]+\frac{2 \alpha_{t+1}^{2} C_{f}^{2}\left(\sigma^{2}+C_{g}^{2}\right)}{\min \left\{B_{1}, B_{2}\right\}}+ \frac{5\alpha_{t+1} L_{f}^{2} C_{g}^{2}}{m} \mathbb{E}\left[\left\|\u_{t} - g(\w_{t})\right\|^{2}\right] 
\end{split}
\end{equation*}\end{lemma}
\begin{proof}
According to Lemma 1 in \cite{dependent2022}, if $\alpha \leq \frac{2}{7}$, we have:
\begin{equation*}
\begin{split}
&\mathbb{E}\left[ \|\z_{t+1} - \nabla F(\w_{t+1})\|^2 \right] \leq(1-\alpha_{t+1}) \mathbb{E}\left[ \|\z_{t} - \nabla F(\w_{t})\|^2 \right]+\frac{2 L_{F}^{2} \eta_t^{2} \mathbb{E}\left[\left\|\mathbf{z}_{t}\right\|^{2}\right]}{\alpha_{t+1}} \\
&\quad +\frac{3 L_{f}^{2} C_{g}^{2}}{m} \mathbb{E}\left[\left\|\u_{t+1}-\u_{t}\right\|^{2}\right]+\frac{2 \alpha_{t+1}^{2} C_{f}^{2}\left(\sigma^{2}+C_{g}^{2}\right)}{\min \left\{B_{1}, B_{2}\right\}}+ \frac{5\alpha_{t+1} L_{f}^{2} C_{g}^{2}}{m} \mathbb{E}\left[\left\|\u_{t+1} - g(\w_{t+1})\right\|^{2}\right].
\end{split}
\end{equation*}
By setting $\alpha \leq \frac{1}{15}$, we have the above lemma.
\end{proof}

\begin{lemma}\label{lem:3} If $\beta_{t+1} \leq \frac{1}{2}$, we have:
\begin{equation*}
\begin{split}
  \E\left[\ \|\u_{t+1}-\u_{t}\|^2\right]   &\leq \frac{2B_1\beta_{t+1}^{2}  \sigma^{2} }{B_2}   + \frac{4B_1 \beta_{t+1}^{2} }{m} \E\left[ \left\|\u_{t}- g(\w_{t})\right\|^{2}\right]+\frac{9 m^2 C_g^2}{B_1}\E\left[\|\w_{t+1}-\w_{t}\|^{2}\right]
\end{split}
\end{equation*}
\end{lemma}

\begin{proof} 
Note that with $\beta_{t+1} \leq \frac{1}{2}$, we have $\gamma_{t+1} \leq \frac{2m}{B_1}$
\begin{equation*}
\begin{split}
&\E\left[\|\u_{t+1}-\u_t\|^2 \right] \\
=& \frac{B_1}{m} \sum_{i=1}^{m} \E\left[\left\|\beta_{t+1}\left(g_{i}(\w_{t+1};\xi_{t+1}^i)-\u_{t}^{i}\right)+\gamma_{t+1}\left(g_{i}(\w_{t+1};\xi_{t+1}^{i})-g_{i}(\w_{t};\xi_{t+1}^{i})\right)\right\|^{2}\right] \\
\leq& \frac{B_1}{m} \sum_{i=1}^{m} \E\left[2\beta_{t+1}^{2} \left\|g_{i}(\w_{t+1};\xi_{t+1}^{i})-\u_{t}^{i}\right\|^{2}+2\gamma_{t+1}^{2} \|g_{i}(\w_{t+1};\xi_{t+1}^{i})-g_{i}(\w_{t};\xi_{t+1}^{i})\|^{2}\right] \\
\leq& \E\left[\frac{2B_1\beta_{t+1}^{2}}{m} \sum_{i=1}^{m}  \left\|g_{i}(\w_{t+1};\xi_{t+1}^{i})  -\u_{t}^{i}\right\|^{2}+{2B_1\gamma_{t+1}^{2} C_g^2}\|\w_{t+1}-\w_{t}\|^{2}\right] \\
\leq& \frac{2B_1\beta_{t+1}^{2}}{m} \sum_{i=1}^{m} \left(\E\left[ \left\|g_{i}(\w_{t+1};\xi_{t+1}^{i}) - g_{i}(\w_{t+1})\right\|^{2}\right]+ \E\left[ \left\| g_{i}(\w_{t+1}) -\u_{t}^{i}\right\|^{2}\right]\right)\\
&\quad\quad\quad+{2B_1\gamma_{t+1}^{2} C_g^2} \E\left[ \|\w_{t+1}-\w_{t}\|^{2}\right] \\
\leq& \frac{2 B_1 \beta_{t+1}^{2} \sigma^{2}}{B_2} +\frac{2 B_1 \beta_{t+1}^{2}}{m} \sum_{i=1}^{m}\E\left[\left\|g_{i}(\w_{t+1})-\u^{i}_t\right\|^{2}\right]+\frac{8m^2 C_g^2}{B_1}\E\left[\|\w_{t+1}-\w_{t}\|^{2}\right] \\
\leq& \frac{2 B_1 \beta_{t+1}^{2} \sigma^{2}}{B_2} +\frac{4 B_1 \beta_{t+1}^{2}}{m} \sum_{i=1}^{m}\left\| g_{i}(\w_{t+1}) - g_{i}(\w_{t})\right\|^{2}+\frac{4 B_1 \beta_{t+1}^{2}}{m} \sum_{i=1}^{m}\E\left[\left\| g_{i}(\w_{t})-\u^{i}_t\right\|^{2}\right]\\
&\quad\quad\quad +\frac{8m^2 C_g^2}{B_1}\E\left[\|\w_{t+1}-\w_{t}\|^{2}\right] \\
\leq& \frac{2 B_1 \beta_{t+1}^{2} \sigma^{2}}{B_2} +\frac{4 B_1 \beta_{t+1}^{2}}{m} \sum_{i=1}^{m}\E\left[\left\| g_{i}(\w_{t}) - \u^{i}_t\right\|^{2}\right]+\frac{9m^2 C_g^2}{B_1}\E\left[\|\w_{t+1}-\w_{t}\|^{2}\right].
\end{split}
\end{equation*}
The third inequality is due to $\E \left[g_{i}(\w_{t+1};\xi_{t+1}^{i}) - g_{i}(\w_{t+1})\right] = 0$. 
\end{proof}

\paragraph{The rest proof of Theorem~\ref{thm:1}}  Let  $\Gamma_{t}=F(\w_t)+\frac{B_1}{c_0 \eta_{t-1} m^2}\Norm{\u_{t} - g(\w_{t})}^2+  \frac{1}{c_0}\|\z_{t} - \nabla F(\w_{t})\|^2 $. By setting $\eta_t =\frac{2 \alpha_{t+1}}{c_0}$, $C_0 = 72 C$, $\eta_t \leq \frac{B_1}{4 m}$ we have:
\begin{equation*}
    \begin{split}
        &\E\left[\Gamma_{t+1}-\Gamma_{t} \right]\\
	    = &\E\bigg[F(\w_{t+1}) - F(\w_t)+\frac{B_1}{c_0 \eta_{t} m^2}\Norm{\u_{t+1} - g(\w_{t+1})}^2 + \frac{1}{c_0}\|\z_{t+1} - \nabla F(\w_{t+1})\|^2 \\
	    &  - \frac{B_1}{c_0 \eta_{t-1} m^2 }\Norm{\u_{t} - g(\w_{t})}^2 - \frac{1}{c_0}\|\z_{t} - \nabla F(\w_{t})\|^2 \bigg]\\
	    \leq & \E\bigg[- \frac{\eta_t}{2} \|\nabla F(\w_t)\|^2 + \frac{\eta_t}{2} \|\z_{t} - \nabla F(\w_t)\|^2 - \frac{\eta_t}{4} \Norm{\z_t}^2  -\frac{\alpha_{t+1}}{c_0}  \|\z_{t} - \nabla F(\w_{t})\|^2\\
        &  +\frac{3 C \eta_t^{2} }{\alpha_{t+1} c_0} \left\|\mathbf{z}_{t}\right\|^{2}+\frac{4 L_{f}^{2} C_{g}^{2}}{m c_0} \left\|\u_{t+1}-\u_{t}\right\|^{2}+\frac{2 \alpha_{t+1}^{2} C_{f}^{2}\left(\sigma^{2}+C_{g}^{2}\right)}{\min \left\{B_{1}, B_{2}\right\} c_0} +\frac{8 C \eta_{t}}{c_0} \Norm{\z_t}^2\\
	    & +  \left(\frac{5\alpha_{t+1} L_{f}^{2} C_{g}^{2}}{m c_0} +\frac{B_1}{c_0 \eta_{t} m^2} - \frac{B_1^2 \beta_{t+1}}{m^3 c_0 \eta_{t}}-\frac{B_1}{c_0 \eta_{t-1} m^2 }\right) \Norm{\u_{t} - g(\w_{t})}^2  + \frac{2 B_1^2 \beta_{t+1}^{2} \sigma^{2}}{B_2 m^2 c_0 \eta_{t}} \bigg] \\
        \leq & \E\bigg[ - \frac{\eta_t}{2} \|\nabla F(\w_t)\|^2  +\frac{4 L_{f}^{2} C_{g}^{2}}{m c_0} \left\|\u_{t+1}-\u_{t}\right\|^{2}+\frac{2 \alpha_{t+1}^{2} C_{f}^{2}\left(\sigma^{2}+C_{g}^{2}\right)}{\min \left\{B_{1}, B_{2}\right\} c_0}  - \frac{\eta_t}{8} \Norm{\z_t}^2\\
	    &  \quad+  \left(\frac{5\alpha_{t+1} C}{m c_0}  + \frac{B_1}{c_0 \eta_{t} m^2} - \frac{B_1^2 \beta_{t+1}}{m^3 c_0 \eta_{t}}-\frac{B_1}{c_0 \eta_{t-1} m^2 }\right) \Norm{\u_{t} - g(\w_{t})}^2 + \frac{2 B_1^2 \beta_{t+1}^{2} \sigma^{2}}{B_2 m^2 c_0 \eta_{t}} \bigg]\\
	    \leq & \E\bigg[- \frac{\eta_t}{2} \|\nabla F(\w_t)\|^2  +\frac{2 \alpha_{t+1}^{2} C}{\min \left\{B_{1}, B_{2}\right\} c_0}  + \frac{4 B_1^2 \beta_{t+1}^{2} C}{B_2 m^2 c_0 \eta_{t}}\\
	    &  \quad +  \left(\frac{5\alpha_{t+1} C}{m c_0}  + \frac{B_1}{c_0 \eta_{t} m^2} - \frac{B_1^2 \beta_{t+1}}{m^3 c_0 \eta_{t}}-\frac{B_1}{c_0 \eta_{t-1} m^2 } + \frac{16 B_1 \beta_{t+1}^2 C}{m^2 c_0}\right) \Norm{\u_{t} - g(\w_{t})}^2 \bigg]
    \end{split}
\end{equation*}
By setting $\beta_{t+1} =  \frac{256 m^2 C^2 \eta_t^2}{B_1^2}$ (and note that $c_0 = 72C$, $\alpha_{t+1}=36C \eta_t$), we have:
\begin{equation*}
    \begin{split}
        \E\left[\Gamma_{t+1}-\Gamma_{t}\right] 
	    \leq & \E\left[- \frac{\eta_t}{2} \|\nabla F(\w_t)\|^2  +\frac{2 \alpha_{t+1}^{2} C}{\min \left\{B_{1}, B_{2}\right\} c_0}  + \frac{4 B_1^2 \beta_{t+1}^{2}  C}{B_2 m^2 c_0 \eta_{t}}\right] \\
	    \leq & \E\left[- \frac{\eta_t}{2} \|\nabla F(\w_t)\|^2  +\frac{36 C^2 \eta_t^2}{\min \left\{B_{1}, B_{2}\right\}}  + \frac{16^4 m^2 C^4 \eta_{t}^{3}}{18 B_2 B_1^2} \right]
    \end{split}
\end{equation*}
This means that, by setting $\eta_t = \min\{\sqrt{\min\{B_1,B_2\} }(a+t)^{-1/2}, \left(\frac{B_1\sqrt{B_2}}{m}\right)^{2/3}(a+t)^{-1/3}\}$:
\begin{equation*}
\begin{split}
    \frac{\eta_T}{2}\E\left[\sum_{t=1}^{T}\|\nabla F(\w_t) \|^2\right] \leq& \E\left[\Gamma_{1}-\Gamma_{T+1}\right] +\frac{36 C^2}{ \min \left\{B_{1}, B_{2}\right\}}\E\left[ \sum_{t=1}^{T} \eta_{t}^{2}\right]   + \frac{16^4 m^2 C^4}{18 B_2 B_1^2}\E\left[ \sum_{t=1}^{T} \eta_{t}^{3}\right]  \\
    \leq& \E\left[\Gamma_{1}\right] + 16^3 C^4 \E\left[ \sum_{t=1}^{T} (a+t)^{-1}\right] \\
    \leq& \Delta_F + \frac{2 C}{c_0 \eta_0} + 16^3 C^4 \ln{(1+T)} \\
\end{split}
\end{equation*}
Similar to the proof of Theorem 1 in STORM \citep{cutkosky2019momentum}, denote $M= \Delta_F + \frac{2 C}{c_0 \eta_0} + 16^3 C^4 \ln{(1+T)}$. Using Cauchy-Schwarz inequality, we have:
\begin{align*}
     &\mathbb{E}\left[\sqrt{\sum_{t=1}^{T}\left\|\nabla F\left(\boldsymbol{\w}_{t}\right)\right\|^{2}}\right]^{2} \leq \mathbb{E}\left[1 / \eta_{T}\right] \mathbb{E}\left[\eta_{T} \sum_{t=1}^{T}\left\|\nabla F\left(\boldsymbol{\w}_{t}\right)\right\|^{2}\right]  \leq \mathbb{E}\left[\frac{M}{\eta_{T}}\right]\\ 
     &\quad\quad\quad\quad \leq \mathbb{E}\left[M \max\left\{ \frac{1}{\sqrt{\min\{B_1,B_2\}}} \left(a+T\right)^{1 / 2},\left(\frac{m}{B_1 \sqrt{B_2}} \right)^{2/3}\left(a+T\right)^{1 / 3}\right\}\right],
\end{align*}
which indicate that
\begin{align*}
    &\mathbb{E}\left[\sqrt{\sum_{t=1}^{T}\left\|\nabla F\left(\w_{t}\right)\right\|^{2}}\right] \\
    \leq &\sqrt{M} \max\left\{\left(\min \left\{ B_1,B_2\right\}\right)^{-1/4}  \left(a+T \right)^{1 / 4}, \left(\frac{m}{B_1 \sqrt{B_2}} \right)^{1/3} \left(a+T \right)^{1 / 6}\right\}.
\end{align*}
Finally, using Cauchy-Schwarz we have $
\sum_{t=1}^{T}\left\|\nabla F\left(\w_{t}\right)\right\| / T \leq \sqrt{\sum_{t=1}^{T}\left\|\nabla F\left(\w_{t}\right)\right\|^{2}} / \sqrt{T}$ so that:
\begin{align*}
    &\mathbb{E}\left[\sum_{t=1}^{T} \frac{\left\|\nabla F\left(\boldsymbol{\w}_{t}\right)\right\|}{T}\right] \\
    \leq&  \max\left\{\sqrt{M} \left(\min \left\{ B_1,B_2\right\}\right)^{-1/4}  \frac{\left(a+T \right)^{1 / 4}}{\sqrt{T}}, \sqrt{M}\left(\frac{m}{B_1 \sqrt{B_2}} \right)^{1/3} \frac{\left(a+T \right)^{1 / 6}}{\sqrt{T}}\right\}\\
    \leq& \max\left\{\sqrt{M}\left(\min \left\{ B_1,B_2\right\}\right)^{-1/4} \left( \frac{a^{1 / 4}}{\sqrt{T}} + \frac{1}{T^{1/4}}\right), \sqrt{ M}\left(\frac{m}{B_1 \sqrt{B_2}} \right)^{1/3}  \left(\frac{a^{1 / 6} }{\sqrt{T}}+\frac{1}{T^{1 / 3}}\right)\right\} \\
    \leq& \O\left( \max\left\{\left(\frac{1}{\min \left\{ B_1,B_2\right\} T}\right)^{1/4}, \left(\frac{m}{B_1 \sqrt{B_2} T} \right)^{1/3}  \right\} \right),
\end{align*}
where the last inequality is due to $(a+b)^{1 / 3} \leq a^{1 / 3}+b^{1 / 3}$. So, we can achieve the stationary point with $T=\mathcal{O}\left(\max\left\{  \frac{m}{ B_1 \sqrt{B_2} \epsilon^3 },\frac{1}{ \min \left\{ B_{1}, B_{2}\right\} \epsilon^4  } \right\}\right)$.

\subsection{Proof of Theorem~\ref{thm:2}}
\begin{lemma} \label{lem:5}
Denote $\|\u_{t}  - g(\w_t)\|^2 = \sum_{i=1}^m\|\u^{i}_{t}  - g_{i}(\w_t)\|^2$ and $\|\u_{t}  - \u_{t-1}\|^2 = \sum_{i=1}^m\|\u^{i}_{t}  - \u^{i}_{t-1}\|^2$.
\begin{equation*}
\begin{split}
\E\left[\|\z_{t} - \nabla F(\w_t)\|^2\right] 
&\leq 4\E\left[\bigg\|\z_{t} -  \frac{1}{m} \sum_{i=1}^m \nabla g_{i}(\w_t)\nabla f_i(\u^i_{t-1})\bigg\|^2\right] \\
&\quad + \frac{4C_g^2L_f^2}{m} \E\left[\|  \u_t - \u_{t-1} \|^2\right] +\frac{2 C_g^2L_f^2}{m} \E\left[\|  \u_t - g(\w_t) \|^2 \right]
\end{split}
\end{equation*}
\end{lemma}
\begin{proof}
\begin{equation*}
\begin{split}
&\E\left[\|\z_{t} - \nabla F(\w_t)\|^2 \right]\\
=& 2\E\Bigg[\bigg\|\z_{t} -  \frac{1}{m} \sum_{i=1}^m \nabla g_{i}(\w_t)\nabla f_i(\u^i)\bigg\|^2\\
&\quad + 2\bigg\| \frac{1}{m} \sum_{i=1}^m \nabla g_{i}(\w_t)\nabla f_i(\u^i) -  \frac{1}{m} \sum_{i=1}^m \nabla g_i(\w_{t})\nabla f_i(g_i(\w_{t})) \bigg\|^2\Bigg]\\
\leq&\E\left[ 2\bigg\|\z_{t} -  \frac{1}{m} \sum_{i=1}^m \nabla g_{i}(\w_t)\nabla f_i(\u^i)\bigg\|^2 + \frac{2}{m} \sum_{i=1}^m\bigg\|  \nabla g_{i}(\w_t)\nabla f_i(\u^i) -   \nabla g_i(\w_{t})\nabla f_i(g_i(\w_{t})) \bigg\|^2\right]\\
\leq& \E\left[2\bigg\|\z_{t} -  \frac{1}{m} \sum_{i=1}^m \nabla g_{i}(\w_t)\nabla f_i(\u^i)\bigg\|^2 + \frac{2 C_g^2L_f^2}{m} \sum_{i=1}^m\bigg\|  \u_t^i - g_i(\w_t) \bigg\|^2 \right]\\
\leq& \E\left[4\bigg\|\z_{t} -  \frac{1}{m} \sum_{i=1}^m \nabla g_{i}(\w_t)\nabla f_i(\u^i_{t-1})\bigg\|^2 + \frac{4C_g^2L_f^2}{m} \bigg\|  \u_t - \u_{t-1} \bigg\|^2 +\frac{2 C_g^2L_f^2}{m} \bigg\|  \u_t - g(\w_t) \bigg\|^2\right]
\end{split}
\end{equation*}
\end{proof}

\begin{lemma} \label{lem:6} 
\begin{equation*}
\begin{split}
      \E\left[\left\|\z_{t} - \frac{1}{m} \sum_{i=1}^m \nabla g_{i}(\w_t)\nabla f_i(\u^i_{t-1})\right\|^2 \right]
    \leq   \E\Bigg[(1-\alpha_{t}) \left\|\z_{t-1} - \frac{1}{m} \sum_{i=1}^m \nabla g_{i}(\w_{t-1})\nabla f_i(\u^i_{t-2})\right\|^2 \\ 
  +\frac{2\alpha_{t}^2\sigma^2}{B_1} + \frac{4C_g^2L_f^2}{m} \|\u_{t-1} - \u_{t-2}\|^2 + {4C_f^2 L_g^2}\|\w_{t} - \w_{t-1}\|^2 \Bigg]
\end{split}
\end{equation*}
\end{lemma}

\begin{proof}
\begin{equation*}
\begin{split}
&\E\left[\|\z_{t} - \frac{1}{m} \sum_{i=1}^m \nabla g_{i}(\w_t)\nabla f_i(\u^i_{t-1})\|^2 \right]\\
=&\E\left[\bigg\| (1-\alpha_{t}) \left(\z_{t-1} -  \frac{1}{m} \sum_{i=1}^m \nabla g_{i}(\w_{t-1})\nabla f_i(\u^i_{t-2})\right)\right. \\
&\left. + \alpha_{t}\left(\frac{1}{B_1}\sum_{i \in \mathcal{B}_{1}^{t}}  \nabla g_{i}(\w_t;\xi_t^{i}) \nabla f_{i}(\u_{t-1}^{i}) - \frac{1}{m} \sum_{i=1}^m \nabla g_{i}(\w_t)\nabla f_i(\u^i_{t-1}) \right)\right.\\
&\left. + (1-\alpha_{t})\left(\frac{1}{B_1}\sum_{i \in \mathcal{B}_{1}^{t}} \nabla g_{i}(\w_t;\xi_t^{i}) \nabla f_{i}(\u_{t-1}^{i})- \frac{1}{B_1}\sum_{i \in \mathcal{B}_{1}^{t}}\nabla g_{i}(\w_{t-1};\xi_t^{i})\nabla f_{i}(\u_{t-2}^{i})\right.\right.\\
&\left.\left. \quad\quad\quad\quad\quad  - \frac{1}{m} \sum_{i=1}^m \nabla g_{i}(\w_{t})\nabla f_i(\u^i_{t-1}) + \frac{1}{m} \sum_{i=1}^m \nabla g_{i}(\w_{t-1})\nabla f_i(\u^i_{t-2})\right)  \bigg\|^2\right]
\end{split}
\end{equation*}
We assume that $\E\left[\|\frac{1}{B_1}\sum_{i \in \mathcal{B}_{1}^{t}}  \nabla g_{i}(\w_t;\xi_t^{i}) \nabla f_{i}(\u_{t-1}^{i}) - \frac{1}{m} \sum_{i=1}^m \nabla g_{i}(\w_t)\nabla f_i(\u^i_{t-1})\|^2\right]\leq \frac{\sigma^2}{B_1}$. Due to the fact that the expectation over the last two terms equals zero, we have:
\begin{equation*}
\begin{split}
&\E\left[\left\|\z_{t} - \frac{1}{m} \sum_{i=1}^m \nabla g_{i}(\w_t)\nabla f_i(\u^i_{t-1})\right\|^2 \right]\\
& \leq \E\Bigg[(1-\alpha_{t})^2 \left\|\z_{t-1} - \frac{1}{m} \sum_{i=1}^m \nabla g_{i}(\w_{t-1})\nabla f_i(\u^i_{t-2})\right\|^2  + \frac{2\alpha_{t}^2\sigma^2}{B_1 } \\
&\quad\quad + 2(1-\alpha_{t})^2\frac{1}{B_1}\sum_{i \in \mathcal{B}_{1}^{t}}\left\| \nabla g_{i}(\w_t;\xi_t^{i}) \nabla f_{i}(\u_{t-1}^{i})- \nabla g_{i}(\w_{t-1};\xi_t^{i})\nabla f_{i}(\u_{t-2}^{i})\right\|^2\Bigg]\\
& \leq \E\Bigg[(1-\alpha_{t}) \left\|\z_{t-1} - \frac{1}{m} \sum_{i=1}^m \nabla g_{i}(\w_{t-1})\nabla f_i(\u^i_{t-2})\right\|^2  + \frac{2\alpha_{t}^2\sigma^2} {B_1 }\\
&\quad\quad+4(1-\alpha_{t})^2 \frac{1}{B_1}\sum_{i \in \mathcal{B}_{1}^{t}}\bigg\|\nabla g_{i}(\w_t;\xi_t^{i})\left(\nabla f_{i}(\u_{t-1}^{i})-\nabla f_{i}(\u_{t-2}^{i}) \right)\bigg\|^2 \\
&\quad\quad  +4(1-\alpha_{t})^2\frac{1}{B_1} \sum_{i \in \mathcal{B}_{1}^{t}}\bigg\| \nabla f_{i}(\u_{t-2}^{i}) \left(\nabla g_{i}(\w_t;\xi_t^{i}) - \nabla g_{i}(\w_{t-1};\xi_t^{i})\right)\bigg\|^2\Bigg]\\
& \leq \E\Bigg[(1-\alpha_{t}) \left\|\z_{t-1} - \frac{1}{m} \sum_{i=1}^m \nabla g_{i}(\w_{t-1})\nabla f_i(\u^i_{t-2})\right\|^2  + \frac{2\alpha_{t}^2\sigma^2}{B_1} + \frac{4C_g^2L_f^2}{m} \|\u_{t-1} - \u_{t-2}\|^2 \\
&\quad\quad + {4C_f^2 L_g^2}\|\w_{t} - \w_{t-1}\|^2 \Bigg]
\end{split}
\end{equation*}
\end{proof}
\begin{lemma}
Suppose that $\beta \leq \frac{1}{32 C}$ and $B_1 \beta_{t+1} \leq m \alpha_{t+1}$. Then, we have:
\begin{equation*}
    \begin{split}
    &\E\left[\frac{1}{m}\left\|\u_{t+1}-g\left(\w_{t+1}\right)\right\|^{2} + \left\|\z_{t+1} - \frac{1}{m} \sum_{i=1}^m \nabla g_{i}(\w_{t+1})\nabla f_i(\u^i_{t})\right\|^2\right]    \\
    \leq&  (1-\frac{B_1\beta_{t+1}}{m})\frac{1}{m}\E\left[\left\|\u_{t}-g\left(\w_{t}\right)\right\|^{2}\right]+\frac{8 m C_g^2}{B_1}\left\|\w_{t+1}-\w_{t}\right\|^{2}+ \frac{2B_1\beta_{t+1}^{2} \sigma^{2}}{B_2 m}\\ 
    &\quad + (1-\alpha_{t+1}) \E\left[\left\|\z_{t} - \frac{1}{m} \sum_{i=1}^m \nabla g_{i}(\w_{t})\nabla f_i(\u^i_{t-1})\right\|^2\right] +\frac{2\alpha_{t+1}^2\sigma^2}{B_1}  \\
    &\quad + \frac{4C_g^2L_f^2}{m} \E\left[\|\u_{t} - \u_{t-1}\|^2 \right] + {4C_f^2 L_g^2}\E\left[\|\w_{t+1} - \w_{t}\|^2 \right] \\
    \leq&  (1-\frac{B_1\beta_{t+1}}{m})\E\left[\frac{1}{m}\left\|\u_{t}-g\left(\w_{t}\right)\right\|^{2} + \left\|\z_{t} - \frac{1}{m} \sum_{i=1}^m \nabla g_{i}(\w_{t})\nabla f_i(\u^i_{t-1})\right\|^2 \right]\\
    &\quad +\frac{12 m C}{B_1}\left\|\w_{t+1}-\w_{t}\right\|^{2}+ \frac{2B_1\beta_{t+1}^{2} \sigma^{2}}{B_2 m} +\frac{2\alpha_{t+1}^2\sigma^2}{B_1}  + \frac{4C_g^2L_f^2}{m} \E\left[\|\u_{t} - \u_{t-1}\|^2 \right] \\
     \leq & (1-\frac{B_1\beta_{t+1}}{2m})\E\left[\frac{1}{m}\left\|\u_{t}-g\left(\w_{t}\right)\right\|^{2} + \left\|\z_{t} - \frac{1}{m} \sum_{i=1}^m \nabla g_{i}(\w_{t})\nabla f_i(\u^i_{t-1})\right\|^2 \right]\\ 
    &\quad +\frac{48 m C}{B_1}\left\|\w_{t+1}-\w_{t}\right\|^{2}+ \frac{10 B_1\beta_{t+1}^{2} C}{B_2 m} +\frac{2\alpha_{t+1}^2 C}{B_1}
    \end{split}
\end{equation*}
\end{lemma}
\paragraph{The rest proof of Theorem~\ref{thm:2}}  Set $\eta_t \leq \frac{B_1}{m c_0}$. Denote  $\Gamma_{t}=F(\w_t) +\frac{B_1}{c_0 \eta_t m} \Delta_t$, where $\Delta_t = \frac{1}{m}\Norm{\u_{t} - g(\w_{t})}^2+ \left\|\z_{t} - \frac{1}{m} \sum_{i=1}^m \nabla g_{i}(\w_t)\nabla f_i(\u^i_{t-1})\right\|^2$. We have:
\begin{equation*}
    \begin{split}
        &\E\left[\Gamma_{t+1}-\Gamma_{t}\right] \\
	    = & \E\left[F(\w_{t+1}) - F(\w_t)+\frac{B_1}{c_0 \eta_{t} m}\Delta_{t+1}   - \frac{B_1}{c_0 \eta_{t-1} m }\Delta_{t}\right] \\
	    \leq & \E\bigg[ - \frac{\eta_t}{2} \|\nabla F(\w_t)\|^2 + \frac{\eta_t}{2} \|\z_{t} - \nabla F(\w_t)\|^2 - \frac{\eta_t}{4} \Norm{\z_t}^2  \\
	    & +  \left(\frac{B_1}{c_0 \eta_{t} m} - \frac{B_1^2 \beta_{t+1}}{2 m^2 c_0 \eta_{t}}-\frac{B_1}{c_0 \eta_{t-1} m }\right)\Delta_{t}+\frac{48 C}{c_0 \eta_{t}}\left\|\w_{t+1}-\w_{t}\right\|^{2}+ \frac{10 B_1^2 \beta_{t+1}^{2} C}{B_2 m^2 c_0 \eta_{t}} +\frac{2\alpha_{t+1}^2 C}{m c_0 \eta_{t}} \bigg]\\
	    \leq & \E\bigg[- \frac{\eta_t}{2} \|\nabla F(\w_t)\|^2  - \frac{\eta_t}{4} \Norm{\z_t}^2 +\frac{66 C}{c_0 \eta_{t}}\left\|\w_{t+1}-\w_{t}\right\|^{2}+ \frac{14 B_1^2 \beta_{t+1}^{2} C}{B_2 m^2 c_0 \eta_{t}} +\frac{2\alpha_{t+1}^2 C}{m c_0 \eta_{t}}\\
	    &  +  \left(2C\eta_t + \frac{B_1}{c_0 \eta_{t} m} - \frac{B_1^2 \beta_{t+1}}{2 m^2 c_0 \eta_{t}}-\frac{B_1}{c_0 \eta_{t-1} m }\right)\Delta_{t}\bigg]
    \end{split}
\end{equation*}
By setting $264 C =  c_0$, $\eta_{t}^2 = \frac{32 B_1^2 \beta_{t+1}}{m^2 c_0^2 }$, $\alpha_{t+1} = \frac{B_1 \beta_{t+1}}{m}$ and $B_2 \leq m$, we have:
\begin{equation*}
    \begin{split}
        \E\left[\Gamma_{t+1}-\Gamma_{t} \right]
	    \leq & \E\left[- \frac{\eta_t}{2} \|\nabla F(\w_t)\|^2  + \frac{14 B_1^2\beta_{t+1}^{2} C}{B_2 m^2 c_0 \eta_{t}} +\frac{2\alpha_{t+1}^2 C}{m c_0 \eta_{t}} \right]\\
	    \leq &\E\left[ - \frac{\eta_t}{2} \|\nabla F(\w_t)\|^2  + \frac{m^2 \eta_{t}^{3} c_0^4}{512 B_2 B_1^2 } \right]
    \end{split}
\end{equation*}

This means that, by setting $\eta_t = (\frac{B_1 \sqrt{B_2}}{m})^{\frac{2}{3}}(a+t)^{-\frac{1}{3}}$
\begin{equation*}
\begin{split}
    \frac{\eta_T}{2}\E\left[\sum_{t=1}^{T}\|\nabla F(\w_t) \|^2\right] \leq& \E\left[\Gamma_{1}-\Gamma_{T+1} + \frac{m^2 c_0^4}{512 B_2 B_1^2 } \sum_{t=1}^{T} \eta_{t}^{3}\right] \\
    \leq& \E\left[\Gamma_{1} + \frac{c_0^4}{512}  \sum_{t=1}^{T} (a+t)^{-1}\right] \\
    \leq& \Delta_F + \frac{1}{8 \eta_0} + \frac{c_0^4}{512 } \ln{(1+T)}
\end{split}
\end{equation*}
Denote $M=\Delta_F + \frac{1}{8 \eta_0} + \frac{c_0^4}{512 } \ln{(1+T)}$. Using Cauchy-Schwarz inequality, we have:
\begin{align*}
     \mathbb{E}\left[\sqrt{\sum_{t=1}^{T}\left\|\nabla F\left(\boldsymbol{\w}_{t}\right)\right\|^{2}}\right]^{2} &\leq \mathbb{E}\left[1 / \eta_{T}\right] \mathbb{E}\left[\eta_{T} \sum_{t=1}^{T}\left\|\nabla F\left(\boldsymbol{\w}_{t}\right)\right\|^{2}\right]  \leq \mathbb{E}\left[\frac{M}{\eta_{T}}\right]\\ &\leq \mathbb{E}\left[M \left(\frac{m}{B_1 \sqrt{B_2}} \right)^{2/3}\left(a+T\right)^{1 / 3}\right],
\end{align*}
which indicate that
\begin{align*}
    \mathbb{E}\left[\sqrt{\sum_{t=1}^{T}\left\|\nabla F\left(\w_{t}\right)\right\|^{2}}\right] \leq \sqrt{M} \left(\frac{m}{B_1 \sqrt{B_2}} \right)^{1/3} \left(a+T \right)^{1 / 6}.
\end{align*}
Finally, using Cauchy-Schwarz we have $
\sum_{t=1}^{T}\left\|\nabla F\left(\w_{t}\right)\right\| / T \leq \sqrt{\sum_{t=1}^{T}\left\|\nabla F\left(\w_{t}\right)\right\|^{2}} / \sqrt{T}$ so that:
\begin{align*}
    \mathbb{E}\left[\sum_{t=1}^{T} \frac{\left\|\nabla F\left(\boldsymbol{\w}_{t}\right)\right\|}{T}\right] &\leq \frac{\sqrt{M} \left(a+T\right)^{1 / 6}}{\sqrt{T}} \left(\frac{m}{B_1 \sqrt{B_2}} \right)^{1/3}
    \leq \mathcal{O} \left(\frac{a^{1 / 6} \sqrt{ M}}{\sqrt{T}}+\left(\frac{m}{B_1 \sqrt{B_2} T} \right)^{1/3}\right) \\ &= \mathcal{O}\left(\left(\frac{m}{T B_1 \sqrt{B_2}} \right)^{1 / 3}\right),
\end{align*}
where the last inequality is due to $(a+b)^{1 / 3} \leq a^{1 / 3}+b^{1 / 3}$. So, we can achieve the stationary point with $T=\mathcal{O}\left(m /B_1 \sqrt{B_2} \epsilon^{3}\right)$.

\subsection{Proof of Theorem~\ref{thm:4}}
We would show that the complexity can be further improved if the objective function satisfies the Polyak-Łojasiewicz (PL) condition or convexity. To achieve this, we utilize the previous analysis and use a stage-wise version method~\citep{yuan2019stagewise}. In the new algorithm, we decrease $\alpha_s$ and $\beta_s$ after each stage and increase the number of iterations $T_s$. At the end of each stage, we save the output and use it to restart the next stage. With these modifications, we can obtain a better convergence guarantee under the PL condition or convexity. The new method is summarized in Algorithm~\ref{alg:3}, named Stage-wise MSVR. Next, we will show the proof for optimal MSVR-v2 with Stage-wise version, and the proof for MSVR-v1 is nearly the same as the MSVR-v2.

\begin{algorithm}[tb]
	\caption{Stage-wise MSVR method}
	\label{alg:3}
	\begin{algorithmic}
	\STATE {\bfseries Input:} initial points $\left(\w_0,\u_0,\z_0\right)$
		\FOR{stage $s = 1$ {\bfseries to} $S$}
		\STATE $\w_{s},\u_{s},\z_{s}$ = MSVR (with $T_{s}$, $\alpha_s$, $\beta_s$, $\eta_{s}$ and  $\left(\w_{s-1},\u_{s-1},\z_{s-1}\right)$)
		\ENDFOR
	\STATE Return $\w_{S}$
	\end{algorithmic}
\end{algorithm}

Note that in below the numerical subscripts denote the stage index $\{1, \ldots, S\}$. Denote $\Delta_s =  \left\|\z_{s} - \frac{1}{m} \sum_{i=1}^m \nabla g_{i}(\w_s)\nabla f_i(\u^i_{s-1})\right\|^2 + \frac{1}{m} \left\|\u_{s} -  g(\w_s)\right\|^2$. Let's consider the first stage, $ \Delta_1 \leq 2C = \mu \epsilon_1$ and $F(\w_1)-F_{*} \leq \epsilon_1$, where $\epsilon_1 = \max \{\frac{2C}{\mu}, \Delta_F \}$. Starting form the second stage, we would prove by induction.

Suppose at stage $s-1$, we have $ \Delta_{s-1} \leq \mu \epsilon_{s-1}$ and $F\left(\w_{s-1}\right)-F_{*} \leq \epsilon_{s-1}$. Then at $s$ stage, by setting $264 C =  c_0$, $\eta_{s}^2 = \frac{32 B_1^2 \beta_{s}}{m^2 c_0^2 }$, $\alpha_{s} = \frac{B_1 \beta_{s}}{m}$ and $B_2 \leq m$, we have:
\begin{equation*}
    \begin{split}
        \E\left[\Gamma_{t+1}-\Gamma_{t} \right]
	     \leq \E\left[ - \frac{\eta_s}{2} \|\nabla F(\w_t)\|^2  + \frac{m^2 \eta_{s}^{3} c_0^4}{512 B_2 B_1^2 } \right]
    \end{split}
\end{equation*}
This means that by setting $T_s =  \max\left\{\frac{m c_0^2}{B_1 \mu \sqrt{B_2 \mu \epsilon_s}},\frac{m c_0^4}{B_1 B_2 \mu \epsilon_s} \right\}$, $\eta_s = \frac{8 B_1 \sqrt{B_2 \mu \epsilon_s}}{m c_0^2}$, we have: 
\begin{equation*}
\begin{split}
    &\frac{1}{T}\E\left[\sum_{t=1}^{T}\|\nabla F(\w_t) \|^2\right]\\ \leq &\E\left[\frac{2(\Gamma_{1}-\Gamma_{T+1})}{\eta_s T} + \frac{m^2 c_0^4 \eta_{s}^{2}}{256 B_2 B_1^2 } \right]\\
    \leq &\E\left[\frac{2(F(\w_{s-1}) - F_{*})}{\eta_s T} + \frac{2 B_1\Delta_{s-1}}{c_0 \eta_s^2 T m}+ \frac{m^2 c_0^4 \eta_{s}^{2}}{256 B_2 B_1^2 } \right]\\
    \leq & 2\mu_{s} \epsilon_{s}
\end{split}
\end{equation*}
Due to the PL condition, we have:
\begin{equation*}
\begin{split}
    F(\w_s) - F_{*} \leq \frac{1}{2\mu T }\E\left[\sum_{t=1}^{T}\|\nabla F(\w_t) \|^2\right]  \leq \epsilon_s
\end{split}
\end{equation*}
On the other hand, by setting $\beta_s = \frac{B_2 \mu \epsilon_s}{80 C}$ and $\alpha_{s} = \frac{B_1 \beta_{s}}{m}$, we have:
\begin{equation*}
\begin{split}
    \Delta_s \leq&  \frac{2m}{B_1\beta_s T }\Delta_{s-1} +\frac{96 m^2 C}{B_1^2 \beta_s T}\sum_{t=1}^T\left\|\w_{t+1}-\w_{t}\right\|^{2}+ \frac{20 \beta_s C}{B_2} +\frac{4 m \alpha_s^2 C}{B_1^2 \beta} \\
    \leq & \frac{2m \mu \epsilon_{s-1}}{B_1 \beta_s T}+\frac{96 m^2 \eta_s^2  C}{B_1^2 \beta_s T}\sum_{t=1}^T\left\|\z_{t}\right\|^{2}+ \frac{20 \beta_s C}{B_2} +\frac{4 m \alpha_s^2 C}{B_1^2 \beta} \\
    \leq & \mu \epsilon_s
\end{split} 
\end{equation*}
So, we proved that $F\left(\w_{s}\right)-F_{*} \leq \epsilon_{s}$. That is to say, $F\left(\w_{s}\right)-F_{*} \leq \epsilon$ when $S = \log _{2}\left(\frac{2\epsilon_{1}}{\epsilon}\right) $, and the iteration complexity  is computed as:
\begin{equation*}
\begin{split}
    T_{1}+\sum_{s=2}^{S} T_{s} &\stackrel{\mu \geq \epsilon}{=}\mathcal{O}\left(\sum_{s=2}^{S} \frac{m}{B_1\sqrt{B_2} \mu \epsilon_s}\right) \\ & \leq  \mathcal{O}\left(\frac{m}{B_1 \sqrt{B_2} \mu \epsilon}\right)
\end{split}
\end{equation*}
When $F(\w)$ is convex, we define $\hat{F}(\w) = F(\w) + \frac{\mu}{2}\|\w\|^2$. We know that $\hat{F}(\w)$ is $\mu$-strongly convex, which implies $\mu$-PL condition. We have proved: for any $\delta > 0$,  there exist $T=\mathcal{O}\left(\frac{m}{\mu \delta}\right)$ such that $\hat{F}(\w_{T}) - \hat{F}_{*} \leq \delta$.  It indicates that $F(\w_{T}) - F_{*} \leq \delta + \frac{\mu}{2} \|\w_{*}\|^2 -  \frac{\mu}{2} \|\w_{T}\|^2 \leq \delta + \frac{\mu}{2}D$. For any $\epsilon > 0$, if we choose $\mu = \frac{\epsilon}{D}$ and $\delta = \frac{\epsilon}{2}$, we get $F(\w_{T}) - F_{*} \leq \epsilon$, for some $T=\mathcal{O}\left(\frac{m}{\epsilon^2}\right)$.
\subsection{Proof of Theorem~\ref{thm:3} }

\begin{lemma}\label{lem:7} If $\beta \leq \frac{1}{2}$ and $\beta I \leq \frac{m}{B_1}$, we have:
\begin{equation*}
\begin{split}
  \E\left[\sum_{t=1}^{T} \|\u_{t+1}-\u_{t}\|^2\right]   &\leq   \frac{4B_1 \beta^{2} }{m} \E\left[ \sum_{t=1}^{T}\left\|\u_{t}- g(\w_{t})\right\|^{2}\right] +\frac{11 m^2 C_g^2}{B_1}\sum_{t=1}^{T}\|\w_{t+1}-\w_{t}\|^{2}
\end{split}
\end{equation*}
\end{lemma}

\begin{proof}
Following the analysis of Lemma~\ref{lem:3}, we have:
\begin{equation*}
\begin{split}
\E\left[\|\u_{t+1}-\u_t\|^2 \right] \leq& \frac{2B_1\beta^{2}}{m} \sum_{i=1}^{m} \left(\E\left[ \left\| \widehat g_{i}(\w_{t+1};\xi_{t+1}^{i}) - g_{i}(\w_{t+1})\right\|^{2}\right]+ \E\left[ \left\| g_{i}(\w_{t+1}) -\u_{t}^{i}\right\|^{2}\right]\right)\\
&\quad\quad\quad+\frac{8m^2 C_g^2}{B_1}\|\w_{t+1}-\w_{t}\|^{2}
\end{split}
\end{equation*}
So, we have:
\begin{equation*}
\begin{split}
\E\left[\|\u_{t+1}-\u_t\|^2 \right] 
\leq& 2B_1\beta^{2} C_g^2\left\|\w_{t+1}-\w_{\tau+1}\right\|^{2} + \frac{2B_1\beta^{2}}{m} \E\left[ \left\| g(\w_{t+1}) -\u_{t}\right\|^{2}\right]\\
&\quad\quad\quad+\frac{8m^2 C_g^2}{B_1}\|\w_{t+1}-\w_{t}\|^{2}
\end{split}
\end{equation*}
So, with $\beta^2 I^2 \leq m^2 / B_1^2$, we have:
\begin{equation*}
\begin{split}
\E\left[\sum_{t=1}^{T}\|\u_{t+1}-\u_t\|^2 \right] \leq&  \frac{2B_1\beta^{2}}{m} \E\left[ \sum_{t=1}^{T}\left\| g(\w_{t+1}) -\u_{t}\right\|^{2}\right]\\
&+2B_1\beta^{2} C_g^2\sum_{t=1}^{T}\left\|\w_{t+1}-\w_{\tau+1}\right\|^{2} +\frac{8m^2 C_g^2}{B_1}\sum_{t=1}^{T}\|\w_{t+1}-\w_{t}\|^{2} \\
\leq&  \frac{4B_1\beta^{2}}{m} \E\left[ \sum_{t=1}^{T}\left\| g(\w_{t}) -\u_{t}\right\|^{2}\right]\\
&+ 2B_1\beta^{2} C_g^2 I^2 \sum_{t=1}^T   \left\| \w_{t+1}-\w_t\right\|^2+\frac{9m^2 C_g^2}{B_1}\sum_{t=1}^{T}\|\w_{t+1}-\w_{t}\|^{2} \\
\leq&  \frac{4B_1\beta^{2}}{m} \E\left[ \sum_{t=1}^{T}\left\| g(\w_{t}) -\u_{t}\right\|^{2}\right]+\frac{11m^2 C_g^2}{B_1}\sum_{t=1}^{T}\|\w_{t+1}-\w_{t}\|^{2}
\end{split}
\end{equation*}
\end{proof}
We can also replace Lemma~\ref{lem:6} with following lemma.
\begin{lemma} \label{lem:9} 
With $\alpha I \leq 1$ ,  we have:
\begin{equation*}
\begin{split}
     &\E\left[\sum_{t=1}^T \left\|\z_{t} - \frac{1}{m} \sum_{i=1}^m \nabla g_{i}(\w_t)\nabla f_i(\u^i_{t-1})\right\|^2 \right]
    \leq  \frac{1}{\alpha}\left\|\z_{1} - \frac{1}{m} \sum_{i=1}^m \nabla g_{i}(\w_1)\nabla f_i(\u^i_{0})\right\|^2  \\
    &\quad\quad\quad\quad\quad\quad+\frac{8C_g^2L_f^2}{m \alpha} \E\left[\sum_{t=1}^T\|\u_{t} - \u_{t-1}\|^2 \right]+ \frac{8C_f^2 L_g^2}{ \alpha}\E\left[\sum_{t=1}^T\|\w_{t+1} - \w_t\|^2 \right]
\end{split}
\end{equation*}
\end{lemma}

\begin{proof}
First, since $\h_t$ is an unbiased estimation of $\frac{1}{m} \sum_{i=1}^m \nabla g_{i}(\w_t)\nabla f_i(\u^i_{t-1})$, we have:
\begin{equation*}
\begin{split} 
    &\E\left[\left\| \h_t - \frac{1}{m} \sum_{i=1}^m \nabla g_{i}(\w_t)\nabla f_i(\u^i_{t-1})\right\|^{2}\right]  \\
    =& \E\left[\left\| \frac{1}{B_1}\sum_{i \in \mathcal{B}_{1}^{t}}  \nabla f_{i}(\u_{t-1}^{i}) \nabla g_{i}(\w_t;\xi_t^{i}) - \frac{1}{B_1}\sum_{i \in \mathcal{B}_{1}^{t}}  \nabla f_{i}(\u_{\tau-1}^{i}) \nabla g_{i}(\w_{\tau};\xi_t^{i}) \right.\right. \\
    &\quad\quad\quad \left.\left. + \frac{1}{m}\sum_{i=1}^{m}  \nabla f_{i}(\u_{\tau-1}^{i}) \nabla g_{i}(\w_{\tau};\xi_t^{i})- \frac{1}{m} \sum_{i=1}^m \nabla g_{i}(\w_t)\nabla f_i(\u^i_{t-1})\right\|^{2}\right] \\
    \leq& \E\left[\left\| \frac{1}{B_1}\sum_{i \in \mathcal{B}_{1}^{t}}  \nabla f_{i}(\u_{t-1}^{i}) \nabla g_{i}(\w_t;\xi_t^{i}) - \frac{1}{B_1}\sum_{i \in \mathcal{B}_{1}^{t}}  \nabla f_{i}(\u_{\tau-1}^{i}) \nabla g_{i}(\w_{\tau};\xi_t^{i}) \right\|^{2}\right]\\
    \leq& \E\left[\frac{1}{B_1}\sum_{i \in \mathcal{B}_{1}^{t}}\left\|   \nabla f_{i}(\u_{t-1}^{i}) \nabla g_{i}(\w_t;\xi_t^{i}) -   \nabla f_{i}(\u_{\tau-1}^{i}) \nabla g_{i}(\w_{\tau};\xi_t^{i}) \right\|^{2}\right]\\
    =& 2C_f^2L_g^2 \left\|\w_{t} -\w_{\tau} \right\|^{2} +  \frac{2C_g^2L_f^2}{m}   \left\|\u_{t-1} -\u_{\tau-1} \right\|^{2}
\end{split}
\end{equation*}
Next, we have:
\begin{equation*}
\begin{split}
&\E\left[\|\z_{t} - \frac{1}{m} \sum_{i=1}^m \nabla g_{i}(\w_t)\nabla f_i(\u^i_{t-1})\|^2 \right]\\
=&\E\left[\bigg\| (1-\alpha) \left(\z_{t-1} -  \frac{1}{m} \sum_{i=1}^m \nabla g_{i}(\w_{t-1})\nabla f_i(\u^i_{t-2})\right) + \alpha\left(\h_t - \frac{1}{m} \sum_{i=1}^m \nabla g_{i}(\w_t)\nabla f_i(\u^i_{t-1}) \right)\right.\\
&\left. + (1-\alpha)\left(\frac{1}{B_1}\sum_{i \in \mathcal{B}_{1}^{t}} \nabla g_{i}(\w_t;\xi_t^{i}) \nabla f_{i}(\u_{t-1}^{i})- \frac{1}{B_1}\sum_{i \in \mathcal{B}_{1}^{t}}\nabla g_{i}(\w_{t-1};\xi_t^{i})\nabla f_{i}(\u_{t-2}^{i})\right.\right.\\
&\left.\left. \quad\quad\quad\quad\quad  - \frac{1}{m} \sum_{i=1}^m \nabla g_{i}(\w_{t})\nabla f_i(\u^i_{t-1}) + \frac{1}{m} \sum_{i=1}^m \nabla g_{i}(\w_{t-1})\nabla f_i(\u^i_{t-2})\right)  \bigg\|^2\right]\\
& \leq \E\bigg[(1-\alpha)^2 \left\|\z_{t-1} - \frac{1}{m} \sum_{i=1}^m \nabla g_{i}(\w_{t-1})\nabla f_i(\u^i_{t-2})\right\|^2  + 4\alpha^2 C_f^2L_g^2 \left\|\w_{t} -\w_{\tau} \right\|^{2} \\
&\quad\quad+  \frac{4\alpha^2 C_g^2L_f^2}{m}   \left\|\u_{t-1} -\u_{\tau-1} \right\|^{2}  \\
&\quad\quad + 2(1-\alpha)^2\frac{1}{B_1}\sum_{i \in \mathcal{B}_{1}^{t}}\left\| \nabla g_{i}(\w_t;\xi_t^{i}) \nabla f_{i}(\u_{t-1}^{i})- \nabla g_{i}(\w_{t-1};\xi_t^{i})\nabla f_{i}(\u_{t-2}^{i})\right\|^2\bigg]\\
& \E\bigg[\leq (1-\alpha) \|\z_{t-1} - \frac{1}{m} \sum_{i=1}^m \nabla g_{i}(\w_{t-1})\nabla f_i(\u^i_{t-2})\|^2  + 4\alpha^2 C_f^2L_g^2 \left\|\w_{t} -\w_{\tau} \right\|^{2}\\
&\quad\quad +  \frac{4\alpha^2 C_g^2L_f^2}{m}   \left\|\u_{t-1} -\u_{\tau-1} \right\|^{2}  + \frac{4C_g^2L_f^2}{m} \|\u_{t-1} - \u_{t-2}\|^2 + {4C_f^2 L_g^2}\|\w_{t} - \w_{t-1}\|^2 \bigg]
\end{split}
\end{equation*}
The first inequality is due to the fact that the last two terms equal zero in expectation.

Summing up, we have:
\begin{equation*}
\begin{split}
    &\sum_{t=1}^T \left\|\z_{t} - \frac{1}{m} \sum_{i=1}^m \nabla g_{i}(\w_t)\nabla f_i(\u^i_{t-1})\right\|^2 \\
    \leq & \frac{1}{\alpha}\left\|\z_{1} - \frac{1}{m} \sum_{i=1}^m \nabla g_{i}(\w_1)\nabla f_i(\u^i_{0})\right\|^2  + 4\alpha C_f^2L_g^2 \sum_{t=1}^T \left\|\w_{t} -\w_{\tau} \right\|^{2}  \\
      &+\frac{4\alpha C_g^2L_f^2}{m}   \sum_{t=1}^T \left\|\u_{t-1} -\u_{\tau-1} \right\|^{2} + \frac{4C_g^2L_f^2}{m \alpha} \sum_{t=1}^T\|\u_{t} - \u_{t-1}\|^2 + \frac{4C_f^2 L_g^2}{ \alpha}\sum_{t=1}^T\|\w_{t+1} - \w_t\|^2  \\
        \leq & \frac{1}{\alpha}\left\|\z_{1} - \frac{1}{m} \sum_{i=1}^m \nabla g_{i}(\w_1)\nabla f_i(\u^i_{0})\right\|^2  +
      4\alpha C_f^2L_g^2 I^2 \sum_{t=1}^T \left\|\w_{t+1} -\w_{t} \right\|^{2}  \\
      &+\frac{4\alpha C_g^2L_f^2 I^2}{m}   \sum_{t=1}^T \left\|\u_{t} -\u_{t-1} \right\|^{2}  + \frac{4C_g^2L_f^2}{m \alpha} \sum_{t=1}^T\|\u_{t} - \u_{t-1}\|^2 + \frac{4C_f^2 L_g^2}{ \alpha}\sum_{t=1}^T\|\w_{t+1} - \w_t\|^2  \\
       \leq & \frac{1}{\alpha}\left\|\z_{1} - \frac{1}{m} \sum_{i=1}^m \nabla g_{i}(\w_1)\nabla f_i(\u^i_{0})\right\|^2  + \frac{8C_g^2L_f^2}{m \alpha} \sum_{t=1}^T\|\u_{t} - \u_{t-1}\|^2 + \frac{8C_f^2 L_g^2}{ \alpha}\sum_{t=1}^T\|\w_{t+1} - \w_t\|^2 
\end{split}
\end{equation*}
The last inequality is due to $\alpha I \leq 1$.
\end{proof}

\paragraph{The rest proof of Theorem~\ref{thm:3}}
According to Lemma~\ref{lem:5}, we have:
\begin{equation*}
\begin{split}
\sum_{t=1}^{T}\|\z_{t} - \nabla F(\w_t)\|^2 
&\leq 4\sum_{t=1}^{T}\bigg\|\z_{t} -  \frac{1}{m} \sum_{i=1}^m \nabla g_{i}(\w_t)\nabla f_i(\u^i_{t-1})\bigg\|^2 \\
&\quad + \frac{4C_g^2L_f^2}{m}\sum_{t=1}^{T} \|  \u_t - \u_{t-1} \|^2 +\frac{2 C_g^2L_f^2}{m} \sum_{t=1}^{T}\|  \u_t - g(\w_t) \|^2 
\end{split}
\end{equation*}
We use Lemma~\ref{lem:9} to replace $\sum_{t=1}^T \left\|\z_{t} - \frac{1}{m} \sum_{i=1}^m \nabla g_{i}(\w_t)\nabla f_i(\u^i_{t-1})\right\|^2$:
\begin{equation*}
\begin{split}
    &\E\left[\sum_{t=1}^{T}\|\z_{t} - \nabla F(\w_t)\|^2 \right] \\
    \leq& \frac{4}{\alpha}\left\|\z_{1} - \frac{1}{m} \sum_{i=1}^m \nabla g_{i}(\w_1)\nabla f_i(\u^i_{0})\right\|^2 +\frac{32C_g^2L_f^2}{m \alpha} \sum_{t=1}^T\|\u_{t} - \u_{t-1}\|^2   \\
    &\quad + \frac{32C_f^2 L_g^2}{ \alpha}\sum_{t=1}^T\|\w_{t+1} - \w_t\|^2 + \frac{4C_g^2L_f^2}{m} \sum_{t=1}^{T} \|  \u_t - \u_{t-1} \|^2 +\frac{2 C_g^2L_f^2}{m} \sum_{t=1}^{T}\|  \u_t - g(\w_t) \|^2 \\
    \leq &\frac{4}{\alpha}\left\|\z_{1} - \frac{1}{m} \sum_{i=1}^m \nabla g_{i}(\w_1)\nabla f_i(\u^i_{0})\right\|^2 +\frac{36C_g^2L_f^2}{m \alpha} \E\left[\sum_{t=1}^T\|\u_{t} - \u_{t-1}\|^2 \right]  \\
    &\quad  +\frac{32C_f^2 L_g^2}{ \alpha}\E\left[\sum_{t=1}^T\|\w_{t+1} - \w_t\|^2 \right]+\frac{2 C_g^2L_f^2}{m} \E\left[\sum_{t=1}^{T}\|  \u_t - g(\w_t) \|^2 \right]
\end{split}
\end{equation*}
Set $\beta B_1 \leq m \alpha$. We use Lemma~\ref{lem:7} to replace $\E\left[\sum_{t=1}^{T} \|\u_{t}-\u_{t-1}\|^2\right]$ (set $\u_0 = \u_1$):
\begin{equation*}
\begin{split}
    &\E\left[\sum_{t=1}^{T}\|\z_{t} - \nabla F(\w_t)\|^2 \right] \\
    &\leq \frac{4}{\alpha}\left\|\z_{1} - \frac{1}{m} \sum_{i=1}^m \nabla g_{i}(\w_1)\nabla f_i(\u^i_{0})\right\|^2 + \frac{144(C_g^2L_f^2)B_1 \beta^{2}}{m^2 \alpha}\E\left[ \sum_{t=1}^{T}\left\|\u_{t}- g(\w_{t})\right\|^{2}\right]   \\
    &\quad +\frac{396 m(C_g^4L_f^2)}{\alpha B_1}\sum_{t=1}^{T}\|\w_{t+1}-\w_{t}\|^{2} + \frac{32C_f^2 L_g^2}{ \alpha}\E\left[\sum_{t=1}^T\|\w_{t+1} - \w_t\|^2 \right]\\
    &\quad +\frac{2 C_g^2L_f^2}{m} \E\left[\sum_{t=1}^{T}\|  \u_t - g(\w_t) \|^2 \right] \\
    &\leq \frac{4}{\alpha}\left\|\z_{1} - \frac{1}{m} \sum_{i=1}^m \nabla g_{i}(\w_1)\nabla f_i(\u^i_{0})\right\|^2+\frac{428 m C}{\alpha B_1}\E\left[\sum_{t=1}^T\|\w_{t+1} - \w_t\|^2 \right]\\
    &\quad+\frac{146 C_g^2L_f^2}{m} \E\left[\sum_{t=1}^{T}\|  \u_t - g(\w_t) \|^2 \right] 
\end{split}
\end{equation*}
We use Lemma~\ref{lem:main2} to replace $\mathbb{E}\left[\sum_{t=1}^{T}\left\|\u_t - g(\w_t)\right\|^{2}\right]$:
\begin{equation*}
\begin{split}
    &\E\left[\sum_{t=1}^{T}\|\z_{t} - \nabla F(\w_t)\|^2 \right] \\
    \leq& \frac{4}{\alpha}\left\|\z_{1} - \frac{1}{m} \sum_{i=1}^m \nabla g_{i}(\w_1)\nabla f_i(\u^i_{0})\right\|^2 +\frac{428 m C}{\alpha B_1} \sum_{t=1}^T\|\w_{t+1} - \w_t\|^2 \\
    &\quad\quad + \frac{146 C_g^2L_f^2\E\left[\left\|\u_{1}-g\left(\w_{1}\right)\right\|^{2}\right]}{B_1 \beta}+\frac{1460 m^{2} C_g^4L_f^2}{B_1^2 \beta} \sum_{t=1}^{T}\left\|\w_{t+1} - \w_{t}\right\|^{2}\\
    \leq& \frac{4}{\alpha}\left\|\z_{1} - \frac{1}{m} \sum_{i=1}^m \nabla g_{i}(\w_1)\nabla f_i(\u^i_{0})\right\|^2  +  \frac{146 C_g^2L_f^2\E\left[\left\|\u_{1}-g\left(\w_{1}\right)\right\|^{2}\right]}{B_1 \beta}\\
    &\quad\quad+\frac{1888 m^{2} C}{B_1^2 \beta} \sum_{t=1}^{T}\left\|\w_{t+1} - \w_{t}\right\|^{2}\\
    \leq& \frac{\Delta_0}{\alpha T_0} +  \frac{\Delta_0}{\beta T_0}  +\frac{1888 m^{2} C}{B_1^2 \beta} \sum_{t=1}^{T}\left\|\w_{t+1} - \w_{t}\right\|^{2}\\
\end{split}
\end{equation*}
Set $\frac{1888 m^{2} C \eta^{2}}{B_1^2 \beta} \leq \frac{1}{2}$. We have:
\begin{equation*}
\begin{split}
    \mathbb{E}\left[ \sum_{t=1}^{T} \|\z_{t} - \nabla F(\w_{t})\|^2 \right] \leq & \frac{\Delta_0}{\alpha T_0} +  \frac{\Delta_0}{\beta T_0}  +\frac{1}{2} \sum_{t=1}^{T}\left\|\z_{t}\right\|^{2} 
\end{split}
\end{equation*}
According to Lemma~\ref{lem:1}, we have:
\begin{equation*}
\begin{split}
    \E\left[\sum_{t=1}^{T}\|\nabla F(\w_t) \|^2\right] \leq& \frac{2 F(\w_1)}{\eta} + \sum_{t=1}^{T} \E\left[\|\z_{t}- \nabla F(\w_t) \|^2\right] -\frac{1}{2} \sum_{t=1}^{T}\left\|\z_{t}\right\|^{2} \\
    \leq& \frac{2 F(\w_1)}{\eta} + \frac{\Delta_0}{\alpha T_0} +  \frac{\Delta_0}{\beta T_0}  
\end{split}
\end{equation*}
Finally,
\begin{equation*}
\begin{split}
    \frac{1}{T}\E\left[\sum_{t=1}^{T}\|\nabla F(\w_t) \|^2\right]
    \leq \frac{2 F(\w_1)}{\eta T} + \frac{\Delta_0}{\alpha T_0 T} +  \frac{\Delta_0}{\beta T_0 T} 
\end{split}
\end{equation*}
Note that the sample complexity is $\left(B_1 B_2 T + \frac{m n T}{I}\right)$. To ensure the first term and the second term at the same order, we set $I = \left(\frac{m n}{B_1 B_2}\right)$. Also, since we assume that $\alpha I \leq 1$ and $\beta I \leq \frac{m}{B_1}$, we directly set $\alpha = \frac{B_1 B_2}{m n}$ and $\beta = \frac{B_2}{n}$. This setting also satisfies the requirement $B_1 \beta \leq m \alpha$. We also require $\frac{1888 m^2 C \eta^2}{B_1^2 \beta} \leq \frac{1}{2}$. So, we set $\eta = \mathcal{O}(\frac{B_1 \sqrt{B_2}}{m \sqrt{n}})$.
With $T = \mathcal{O}\left( \frac{m \sqrt{n}}{ B_1 \sqrt{B_2} \epsilon^2  } \right)$ and $T_0 = \mathcal{O}\left(\frac{\sqrt{n}}{\sqrt{B_2}}\right)$, We have: $\frac{1}{T}\E\left[\sum_{t=1}^{T}\|\nabla F(\w_t) \|^2\right] \leq \epsilon^2$.   

\subsection{Proof of Theorem~\ref{thm:6}}
The analysis is very similar form Theorem~\ref{thm:4}. We still use Algorithm~\ref{alg:3} but employ MSVR-v3 instead. Also, we do not need to decrease $\alpha$, $\beta$, $\eta$ and increase $T$ during each stage. Let's consider the first stage, $ 4\left\|\z_{1} - \frac{1}{m} \sum_{i=1}^m \nabla g_{i}(\w_1)\nabla f_i(\u^i_{0})\right\|^2 \leq 4C \leq \mu\epsilon_1 $, $\frac{146 C_g^2L_f^2}{m}\left\|\u_{1} -  g(\w_1)\right\|^2 \leq 146 C \leq \mu\epsilon_1$ and $F(\w_1)-F_{*} \leq  \Delta_F \leq \epsilon_1$, where we set $\epsilon_1 = \max\{\Delta_F, \frac{146 C}{\mu}  \}$. Note that in below the numerical subscripts denote the stage index $\{1, \ldots, S\}$. Set $\alpha = \frac{B_1 B_2}{m n}$, $\beta = \frac{B_2}{n}$, $\eta = \mathcal{O}(\frac{B_1 \sqrt{B_2}}{m \sqrt{n}})$ and $T =\O\left(\max \left\{\frac{m n}{B_1 B_2}, \frac{m \sqrt{n}}{\mu B_1 \sqrt{B_2}}\right\} \right)$. 

Starting form the second stage, we would prove by induction. Suppose at the stage $s-1$, we have $ F\left(\w_{s-1}\right)-F_{*} \leq \epsilon_{s-1}$, $ 4 \left\|\z_{s-1} - \frac{1}{m} \sum_{i=1}^m \nabla g_{i}(\w_1)\nabla f_i(\u^i_{s-2})\right\|^2 \leq \mu  \epsilon_{s-1}$, and $\frac{146 C_g^2 L_f^2}{m} \left\|\u_{s-1} -  g(\w_{s-1})\right\|^2 \leq \mu \epsilon_{s-1}$. Then at $s$ stage, we have:
\begin{equation*}
\begin{split}
    F(\w_s)-F_{*} &\leq \frac{1}{2\mu} \Norm{\nabla F(\w_s)}^2\\
    & \leq \frac{\epsilon_{s-1}}{\mu \eta T} + \frac{\epsilon_{s-1}}{\alpha T} + \frac{m \epsilon_{s-1}}{\beta B_1 T} \\
    & \leq \epsilon_s
\end{split}
\end{equation*}
On the other hand, following the very similar analysis in Theorem~\ref{thm:4}, we have:
\begin{equation*}
\begin{split}
    4 \left\|\z_{s} - \frac{1}{m} \sum_{i=1}^m \nabla g_{i}(\w_1)\nabla f_i(\u^i_{s-1})\right\|^2 \leq  \mu \epsilon_s \\
    \frac{146 C_g^2 L_f^2}{m} \left\|\u_{s} -  g(\w_{s})\right\|^2 \leq \mu \epsilon_{s} 
\end{split}
\end{equation*}

We proved that $F\left(\w_{s}\right)-F_{*} \leq \epsilon_{s}$. That is to say, $F\left(\w_{S}\right)-F_{*} \leq \epsilon$ when $S = \log _{2}\left(\frac{2\epsilon_{1}}{\epsilon}\right) = \log _{2}\left(\frac{L}{ \epsilon}\right)$, and the iteration complexity until this stage is computed as:
\begin{equation*}
\begin{split}
    \sum_{s=1}^{S} T_{s}  \leq \O\left(\max \left\{\frac{m n}{B_1 B_2}, \frac{m \sqrt{n}}{\mu B_1 \sqrt{B_2}} \right\} \cdot
    \log\frac{1}{\epsilon}\right)
\end{split}
\end{equation*}
When $F(\w)$ is convex, we define $\hat{F}(\w) = F(\w) + \frac{\mu}{2}\|\w\|^2$. We know that $\hat{F}(\w)$ is $\mu$-strongly convex, which implies $\mu$-PL condition. We have proved: for any $\delta > 0$,  there exist $T=\O\left(\frac{m \sqrt{n}}{\mu B_1 \sqrt{B_2}}  \cdot \log\frac{1}{\epsilon}\right)$ such that $\hat{F}(\w_{T}) - \hat{F}_{*} \leq \delta$.  It indicates that $F(\w_{T}) - F_{*} \leq \delta + \frac{\mu}{2} \|\w_{*}\|^2 -  \frac{\mu}{2} \|\w_{T}\|^2 \leq \delta + \frac{\mu}{2}D$. For any $\epsilon > 0$, if we choose $\mu = \frac{\epsilon}{D}$ and $\delta = \frac{\epsilon}{2}$, we get $F(\w_{T}) - F_{*} \leq \epsilon$, for some $T=\O\left(\frac{m \sqrt{n}}{\epsilon B_1 \sqrt{B_2}}  \cdot
    \log\frac{1}{\epsilon}\right)$.

\section{MSVR with Adaptive Learning Rates} \label{adaptive}
Now we show that the proposed MSVR method can be extended to adaptive learning rates and remains the same sample complexity. To use adaptive learning rates, we can revise the weight update step $\w_{t+1} = \w_t - \eta_t \z_t$ in origin MSVR method as follows:
\begin{equation}\label{rule1}
    \begin{split}
        \w_{t+1} &= \w_t - \frac{\eta_t}{\sqrt{\h_{t}}+\delta} \Pi_{L_f}[\z_t],\\
        \mathbf{h}_{t}^{\prime}&=\left(1-\beta_{t}^{\prime}\right) \mathbf{h}_{t-1}^{\prime}+\beta_{t}^{\prime} \z_{t}^{2},
    \end{split}
\end{equation}
where $\delta > 0$ is a parameter to avoid dividing zero, $\Pi_{L_f}$ denotes the projection onto the ball with radius $L_f$ and $\mathbf{h}_{t} = \mathbf{h}_{t}^{\prime}$ (Adam-style) or $\mathbf{h}_{t} =\max \left(\mathbf{h}_{t-1}, \mathbf{h}_{t}^{\prime}\right)$ (AMSGrad-style).  Inspired by the recent study of Adam-style methods~\citep{guo2022stochastic}, we can give the sample complexity of the Adaptive MSVR using similar analysis. We show the proof of adaptive MSVR-v2 for example:
\begin{thm}
	If we choose parameters $\alpha_{t+1} = \O(\frac{m \eta_{t}^2}{B_1})$, $\beta_{t+1} = \O\left(\frac{m^2 \eta_t^2}{B_1^2}\right)$, $a=O(\frac{m B_2}{B_1}$) and $\eta_t = \O\left((\frac{B_1 \sqrt{B_2}}{m})^{2/3}(a+t)^{-1/3} \right)$, Adaptive MSVR-v2 with learning rate defined in (\ref{rule1}), can obtain a stationary point in $ \mathcal{O}\left(\frac{m \epsilon^{-3} }{ B_1 \sqrt{B_2} } \right)$ iterations.
\end{thm}
\textbf{Remark:} The sample complexity is still at the order of $\mathcal{O}\left(\epsilon^{-3}\right)$. For MSVR-v1 and MSVR-v3, or under the convexity or PL condition, adaptive method can still get the same complexity as the origin rate using a very similar analysis.
\begin{proof}
Note that since the norm of estimated gradient $\left\| \z_{t} \right\|$ is bounded, the value of the learning rate scaling factor $\mathbf{c}=1 /\left(\sqrt{\mathbf{h}_{t}}+\delta\right)$ is also upper bounded and lower bounded, which can be presented as $c_{l} \leq\left\|\mathbf{c}\right\|_{\infty} \leq c_{u}$. (Note that projection onto a ball of radius $C_F$ does not change the analysis, since $\nabla F$ is also in this ball.) With this property, We have: 
\begin{lemma}\label{lem:starter_} (Lemma 3 in \citep{guo2022stochastic})
    For $\mathbf{w}_{t+1}=\mathbf{w}_{t}-\tilde{\eta}_{t} \mathbf{z}_{t}$, with $\eta_t c_{l} \leq \tilde{\eta}_{t} \leq \eta_t c_{u} $ and $ \eta_t L_F\leq {c_l}/{2 c_u^2 }$, we have following guarantee:
	\begin{align*}
		F(\w_{t+1}) \leq F(\w_t) + \frac{\eta_t c_{u}}{2}\Norm{\nabla F(\w_t) - \z_t}^2 - \frac{\eta_t c_{l}}{2}\Norm{\nabla F(\w_t)}^2 - \frac{\eta_t c_{l}}{4}\Norm{\z_t}^2.
	\end{align*}
\end{lemma}	
Then very similar to the proof to Theorem~\ref{thm:2}.
Denote  $\Gamma_{t}=F(\w_t) + \frac{B_1}{c_0 \eta_{t-1} m}\Delta_t$, where $\Delta_t = +\frac{1}{m}\Norm{\u_{t} - g(\w_{t})}^2+ \left\|\z_{t} - \frac{1}{m} \sum_{i=1}^m \nabla g_{i}(\w_t)\nabla f_i(\u^i_{t-1})\right\|^2$. We have:
\begin{equation*}
    \begin{split}
        &\Gamma_{t+1}-\Gamma_{t} \\
	    = &F(\w_{t+1}) - F(\w_t)+\frac{B_1}{c_0 \eta_{t} m}\Delta_{t+1}  - \frac{B_1}{c_0 \eta_{t-1} m }\Delta_t \\
	    \leq & - \frac{\eta_t c_l}{2} \|\nabla F(\w_t)\|^2 + \frac{\eta_t c_u}{2} \|\z_{t} - \nabla F(\w_t)\|^2 - \frac{\eta_t c_l}{4} \Norm{\z_t}^2 \\
	    &  +  \left(\frac{B_1}{c_0 \eta_{t} m} - \frac{B_1^2 \beta_{t+1}}{2 m^2 c_0 \eta_{t}}-\frac{B_1}{c_0 \eta_{t-1} m }\right)\Delta_t  +\frac{48 C}{c_0 \eta_{t}}\left\|\w_{t+1}-\w_{t}\right\|^{2}+ \frac{10 B_1^2 \beta_{t+1}^{2} C}{B_2 m^2 c_0 \eta_{t}} +\frac{2\alpha_{t+1}^2 C}{m c_0 \eta_{t}}\\
	    \leq & - \frac{\eta_t c_l}{2} \|\nabla F(\w_t)\|^2  - \frac{\eta_t c_l}{4} \Norm{\z_t}^2 +\frac{64 C c_u}{c_0 \eta_{t}}\left\|\w_{t+1}-\w_{t}\right\|^{2}+ \frac{14 B_1^2 \beta_{t+1}^{2} C c_u}{B_2 m^2 c_0 \eta_{t}} +\frac{2\alpha_{t+1}^2 C c_u}{m c_0 \eta_{t}}\\
	    &  +  \left(2C c_u \eta_t + \frac{B_1}{c_0 \eta_{t} m} - \frac{B_1^2 \beta_{t+1}}{2 m^2 c_0 \eta_{t}}-\frac{B_1}{c_0 \eta_{t-1} m }\right)\Delta_t
    \end{split}
\end{equation*}
By setting $256 C c_u / c_l =  c_0$, $\eta_{t}^2 = \frac{32 B_1^2 \beta_{t+1}}{m^2 c_0^2 c_l}$, $\alpha_{t+1} = \frac{B_1 \beta_{t+1}}{m}$ and $B_2 \leq m$, we have:
\begin{equation*}
    \begin{split}
        \Gamma_{t+1}-\Gamma_{t} 
	    &\leq - \frac{\eta_t c_l}{2} \|\nabla F(\w_t)\|^2  + \frac{14 B_1^2\beta_{t+1}^{2} C c_u}{B_2 m^2 c_0 \eta_{t}} +\frac{2\alpha_{t+1}^2 C c_u}{m c_0 \eta_{t}}  \\
	    &\leq  - \frac{\eta_t c_l}{2} \|\nabla F(\w_t)\|^2  + \frac{m^2 \eta_{t}^{3} c_0^4 c_l^3}{512 B_2 B_1^2} 
    \end{split}
\end{equation*}

This means that, by setting $\eta_t = (\frac{B_1 \sqrt{B_2}}{m})^{\frac{2}{3}}(a+t)^{-\frac{1}{3}}$
\begin{equation*}
\begin{split}
    \frac{\eta_T}{2}\E\left[\sum_{t=1}^{T}\|\nabla F(\w_t) \|^2\right] \leq& \frac{\Gamma_{1}-\Gamma_{T+1}}{c_l} + \frac{m^2  c_0^4 c_l^2}{512 B_2 B_1^2 } \E\left[ \sum_{t=1}^{T} \eta_{t}^{3}\right] \\
    \leq& \frac{\Gamma_{1}}{c_l} + \frac{c_0^4 c_l^2}{16^5} \E\left[ \sum_{t=1}^{T} (a+t)^{-1}\right] \\
    \leq& \frac{\Delta_F}{c_l} + \frac{1}{8 \eta_0 c_l} + \frac{c_0^4 c_l^2}{512}\ln{(1+T)}
\end{split}
\end{equation*}
Denote $M=\frac{\Delta_F}{c_l} + \frac{1}{8 \eta_0 c_l} + \frac{c_0^4 c_l^2}{16^5}\ln{(1+T)}$. Using Cauchy-Schwarz inequality, we have:
\begin{align*}
     \mathbb{E}\left[\sqrt{\sum_{t=1}^{T}\left\|\nabla F\left(\boldsymbol{\w}_{t}\right)\right\|^{2}}\right]^{2} &\leq \mathbb{E}\left[1 / \eta_{T}\right] \mathbb{E}\left[\eta_{T} \sum_{t=1}^{T}\left\|\nabla F\left(\boldsymbol{\w}_{t}\right)\right\|^{2}\right] \\
     &\leq \mathbb{E}\left[M \left(\frac{m}{B_1 \sqrt{B_2}} \right)^{2/3}\left(a+T\right)^{1 / 3}\right],
\end{align*}
Then following the same analysis, we will finally have :
\begin{align*}
    \mathbb{E}\left[\sum_{t=1}^{T} \frac{\left\|\nabla F\left(\boldsymbol{\w}_{t}\right)\right\|}{T}\right] &\leq \frac{\sqrt{M} \left(a+T\right)^{1 / 6}}{\sqrt{T}} \left(\frac{m}{B_1 \sqrt{B_2}} \right)^{1/3}
    \leq \mathcal{O} \left(\frac{a^{1 / 6} \sqrt{ M}}{\sqrt{T}}+\left(\frac{m}{B_1 \sqrt{B_2} T} \right)^{1/3}\right) \\ &= \mathcal{O}\left(\left(\frac{m}{T B_1 \sqrt{B_2}} \right)^{1 / 3}\right),
\end{align*}
where the last inequality is due to $(a+b)^{1 / 3} \leq a^{1 / 3}+b^{1 / 3}$. So, we can achieve the stationary point with $T=\mathcal{O}\left(m /B_1 \sqrt{B_2} \epsilon^{3}\right)$.
\end{proof}
\begin{figure*}[t]
	\centering
	\subfigure{
		\includegraphics[width=0.3\textwidth]{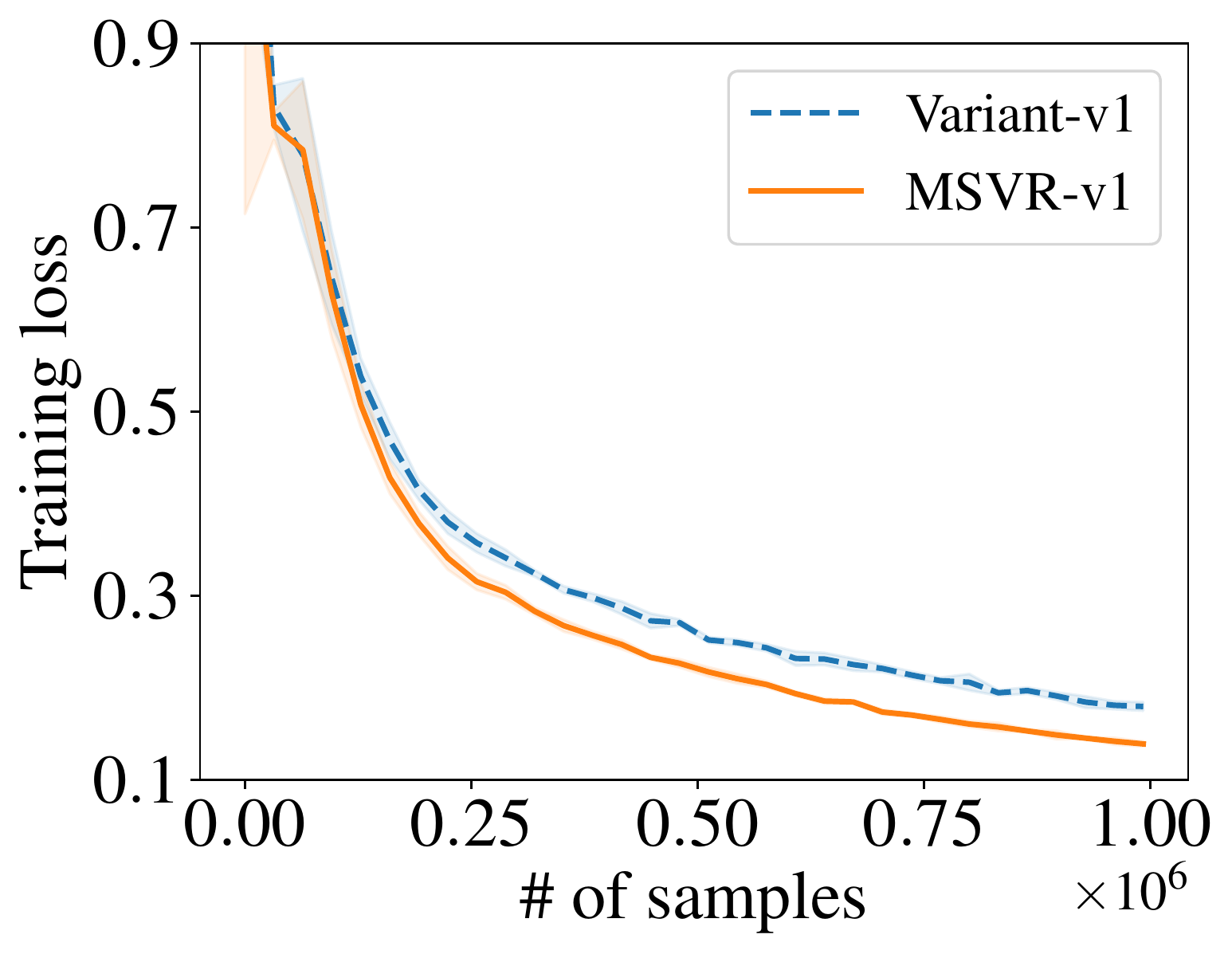}
	}
	\subfigure{
		\includegraphics[width=0.3\textwidth]{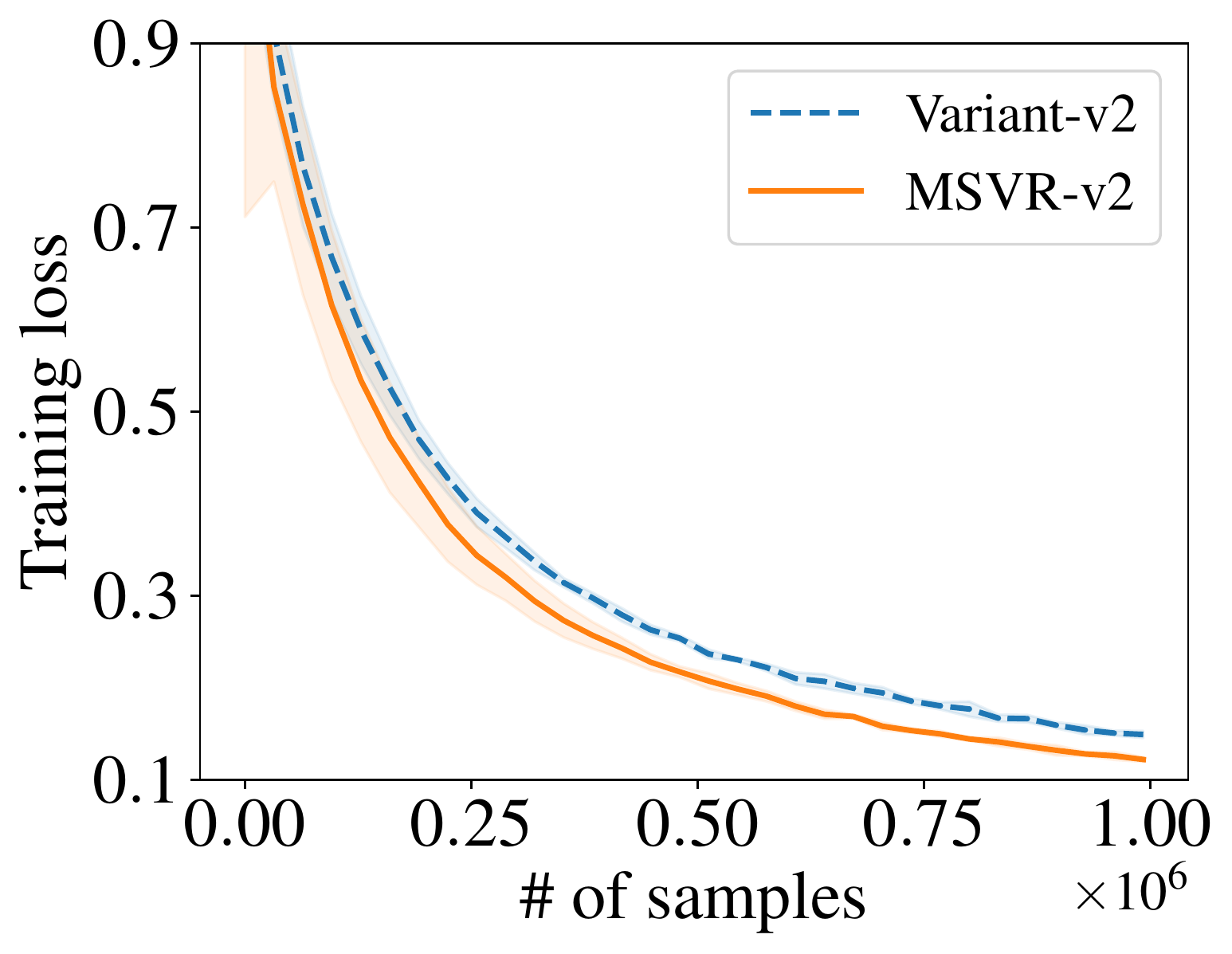}
	}
	\subfigure{
		\includegraphics[width=0.3\textwidth]{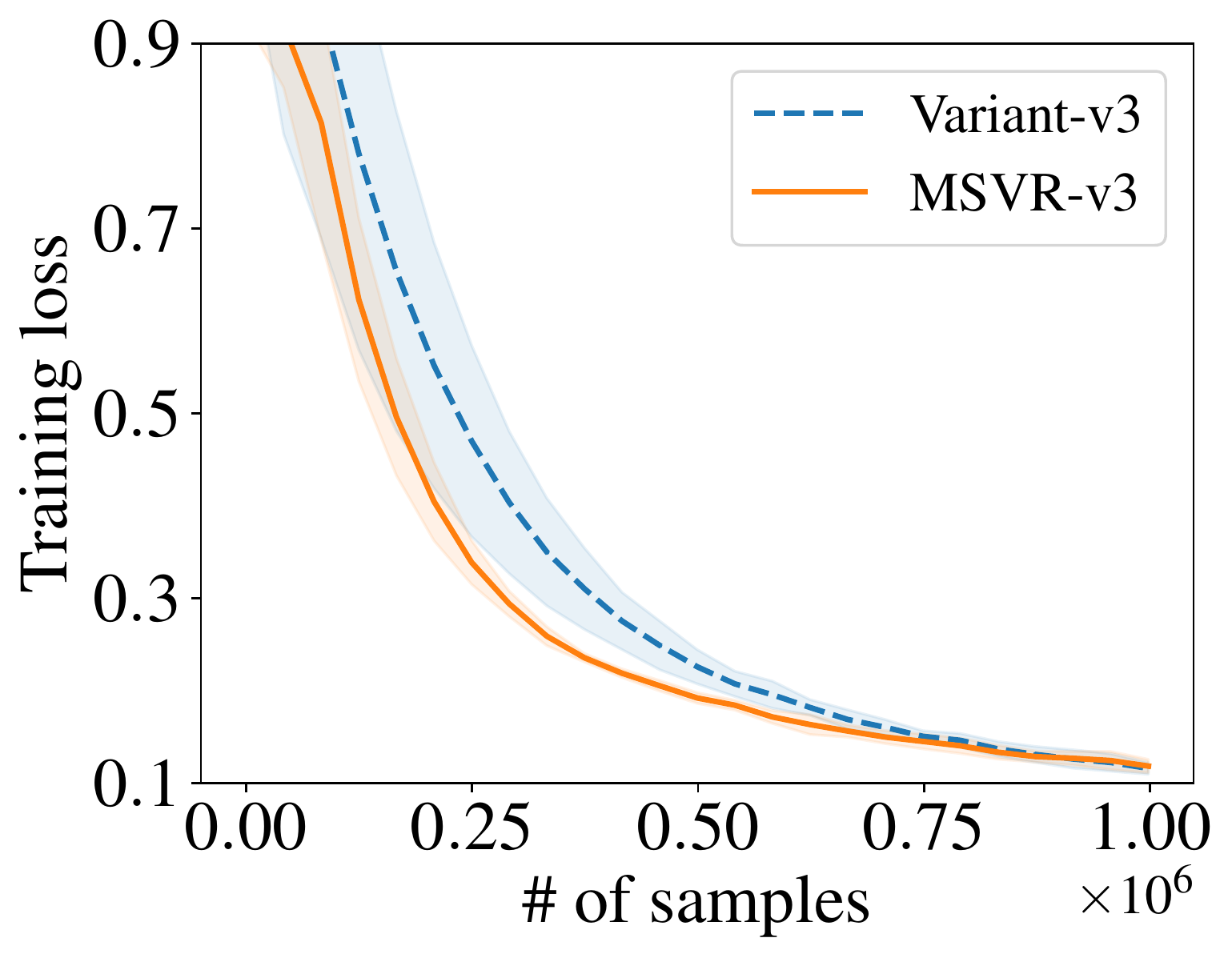}
	}
	\caption{Results for Multi-task AUC Optimization.}
	\label{fig:2}
\end{figure*}

\begin{figure*}[t]
    \centering
    \subfigure[ResNet18]{
        \includegraphics[width=0.3\textwidth]{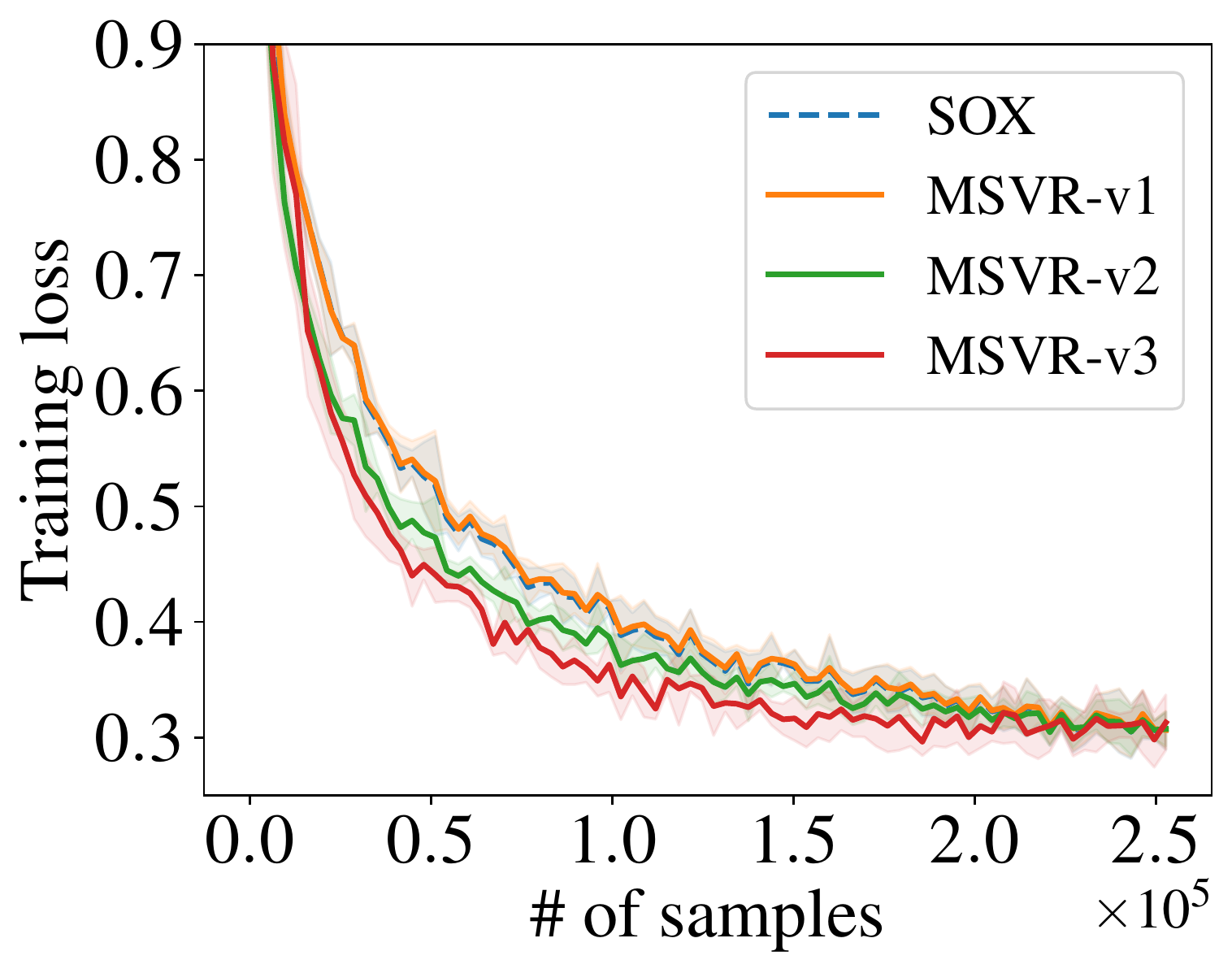}
    }
    \subfigure[ResNet34]{
        \includegraphics[width=0.3\textwidth]{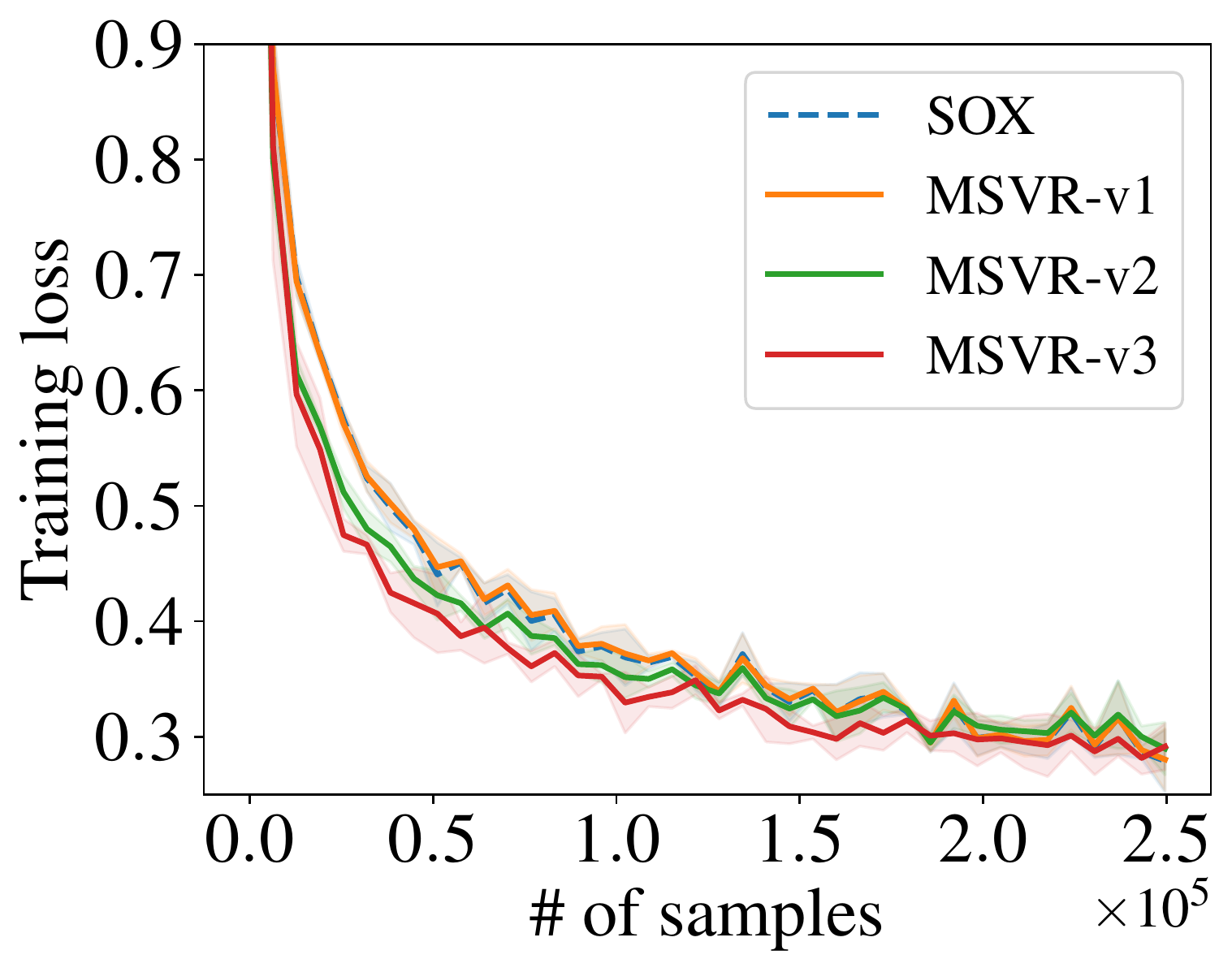}
    }
    \subfigure[DenseNet121]{
        \includegraphics[width=0.3\textwidth]{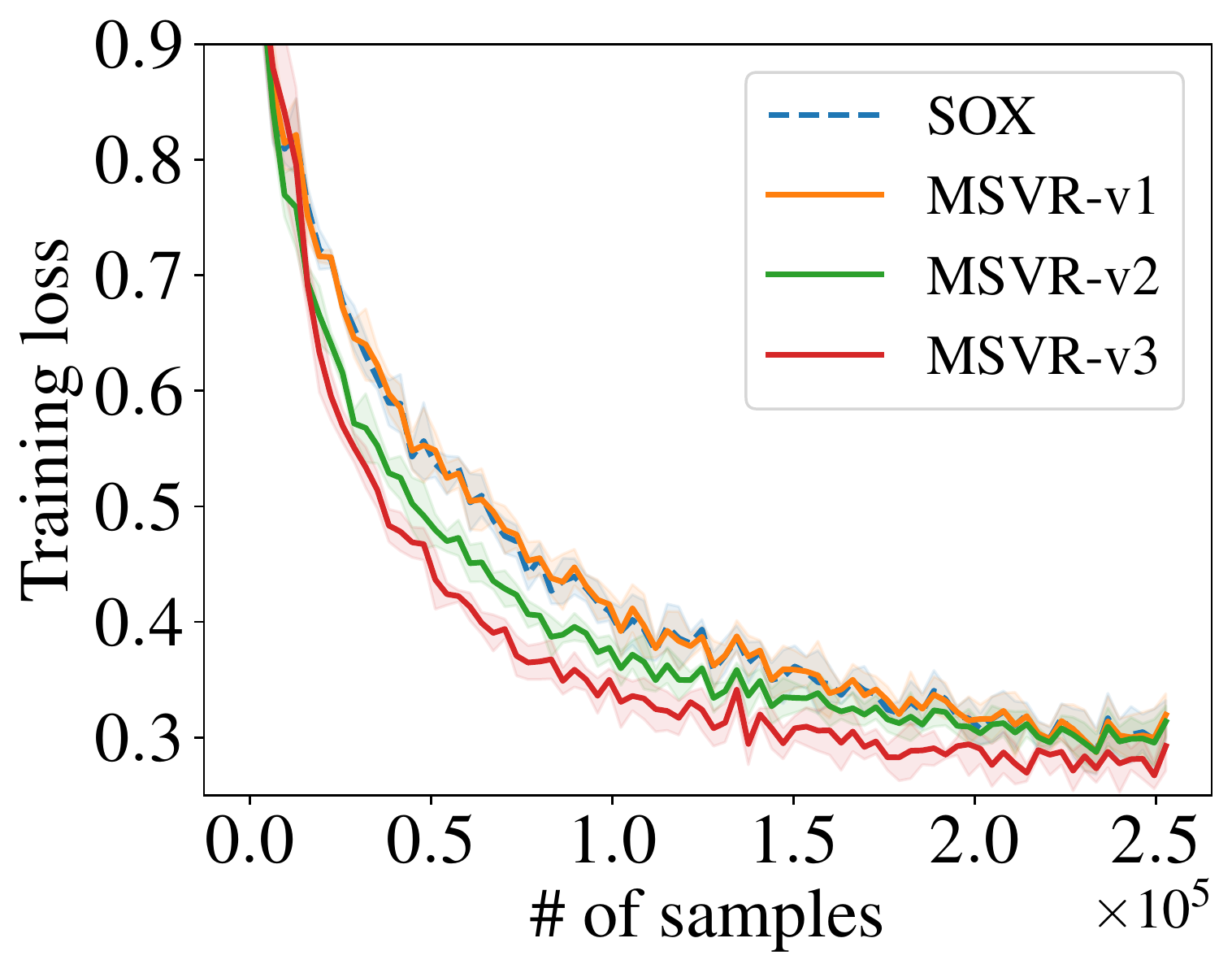}
    }
    \caption{Results with different networks.}
    \label{fig:3}
\end{figure*}
\section{More Experimental Results}
In this section, we provide more experimental results and ablation studies. We will consider more applications in the long version of the paper.
\subsection{Ablation Study on Algorithm Design}
In this subsection, we conduct the ablation study for our algorithm design. Specially, we verify the effects of our customized error correction term. To compare with traditional variance reduced estimator, we can design an estimator using STORM~\citep{cutkosky2019momentum} as follows:
\begin{equation}\label{compare}
    \begin{split}
        \mathbf{u}_{t}^{i}= \begin{cases}(1-\beta) \mathbf{u}_{t-1}^{i}+\beta \frac{m}{B_{1}} g_{i}\left(\mathbf{w}_{t} ; \xi_{t}^{i}\right)+(1-\beta) \frac{m}{B_{1}}\left(g_{i}\left(\mathbf{w}_{t} ; \xi_{t}^{i}\right)-g_{i}\left(\mathbf{w}_{t-1} ; \xi_{t}^{i}\right)\right) & i \in \mathcal{B}_{1}^{t} \\ (1-\beta) \mathbf{u}_{t-1}^{i} & i \notin \mathcal{B}_{1}^{t}\end{cases}
    \end{split}
\end{equation}
To show the effects of our customized error correction term, we replace the MSVR estimator in our MSVR-v1 and MSVR-v2 algorithm, and use equation~(\ref{compare}) instead. We name these two methods Variant-v1 and Variant-v2. For the finite-sum case, we modify the estimator similarly:
\begin{equation}\label{compare2}
    \begin{split}
        \mathbf{u}_{t}^{i}= \begin{cases}(1-\beta) \mathbf{u}_{t-1}^{i}+\beta \frac{m}{B_{1}} \hat{g}_{i}\left(\mathbf{w}_{t} ; \xi_{t}^{i}\right)+(1-\beta) \frac{m}{B_{1}}\left(g_{i}\left(\mathbf{w}_{t} ; \xi_{t}^{i}\right)-g_{i}\left(\mathbf{w}_{t-1} ; \xi_{t}^{i}\right)\right) & i \in \mathcal{B}_{1}^{t} \\ (1-\beta) \mathbf{u}_{t-1}^{i} & i \notin \mathcal{B}_{1}^{t}\end{cases}
    \end{split}
\end{equation}
where $\widehat g_{i}(\w_t; \xi_t^{i}) =   g_{i}(\w_t; \xi_t^{i}) -   g_{i}(\w_\tau; \xi_t^{i}) + g_{i}(\w_\tau)$. So, for MSVR-v3, we replace the MSVR estimator with equation~(\ref{compare2}) and keep other parts unchanged. This new method is named as Variant-v3.

\noindent{\bf Results.} We compare different methods on the CIFAR100 dataset and plot the results in Figure~\ref{fig:2}. As can be seen, all methods perform worse than the origin algorithms, indicating the effectiveness of our customized error correction term in the proposed algorithm.

\begin{figure*}[t]
	\centering
	\subfigure[MSVR-v1]{
		\includegraphics[width=0.3\textwidth]{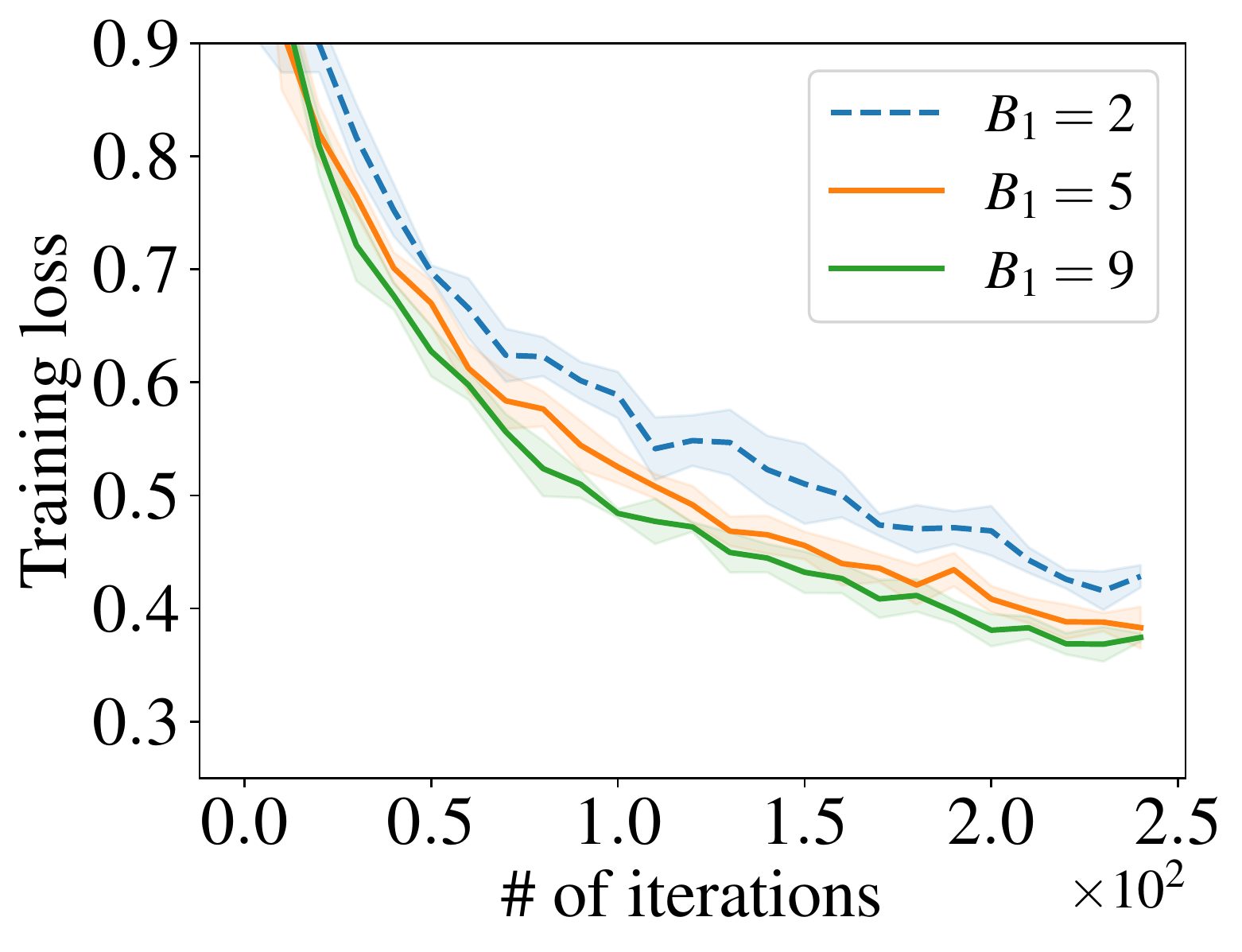}
	}
	\subfigure[MSVR-v2]{
		\includegraphics[width=0.3\textwidth]{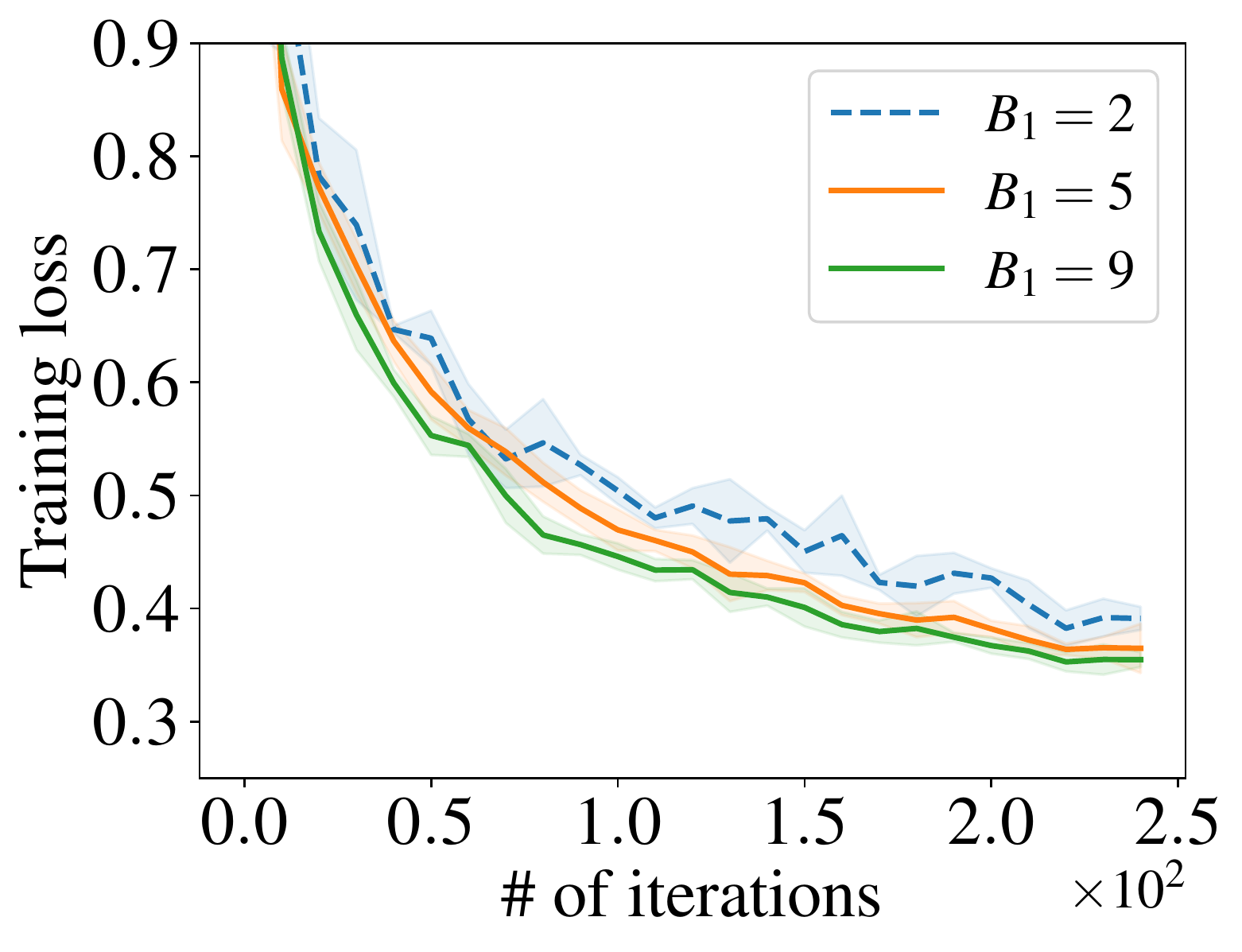}
	}
	\subfigure[MSVR-v3]{
		\includegraphics[width=0.3\textwidth]{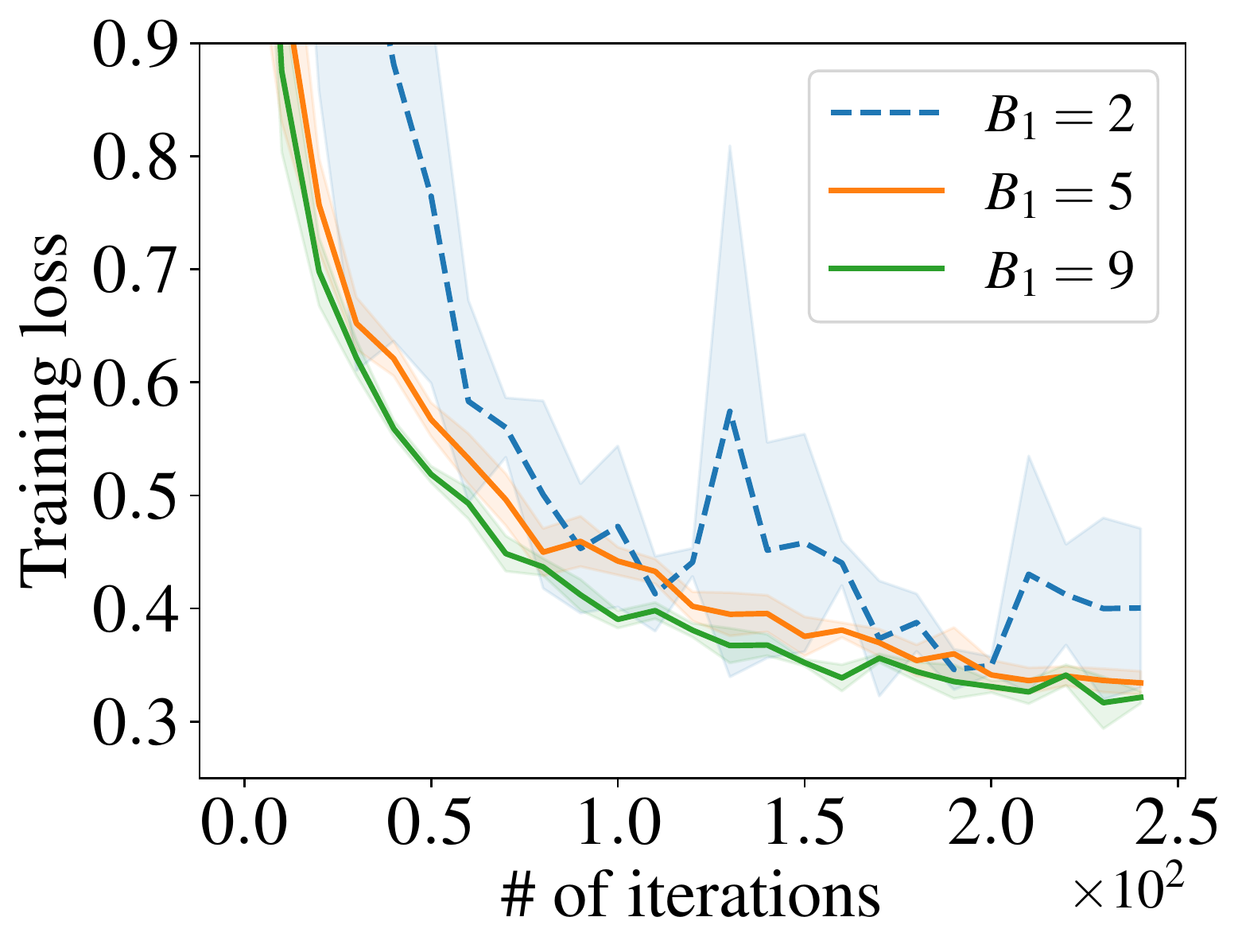}
	}
	\caption{Results with varying $B_1$.}
	\label{fig:4}
\end{figure*}

\begin{figure*}[t]
	\centering
	\subfigure[MSVR-v1]{
		\includegraphics[width=0.3\textwidth]{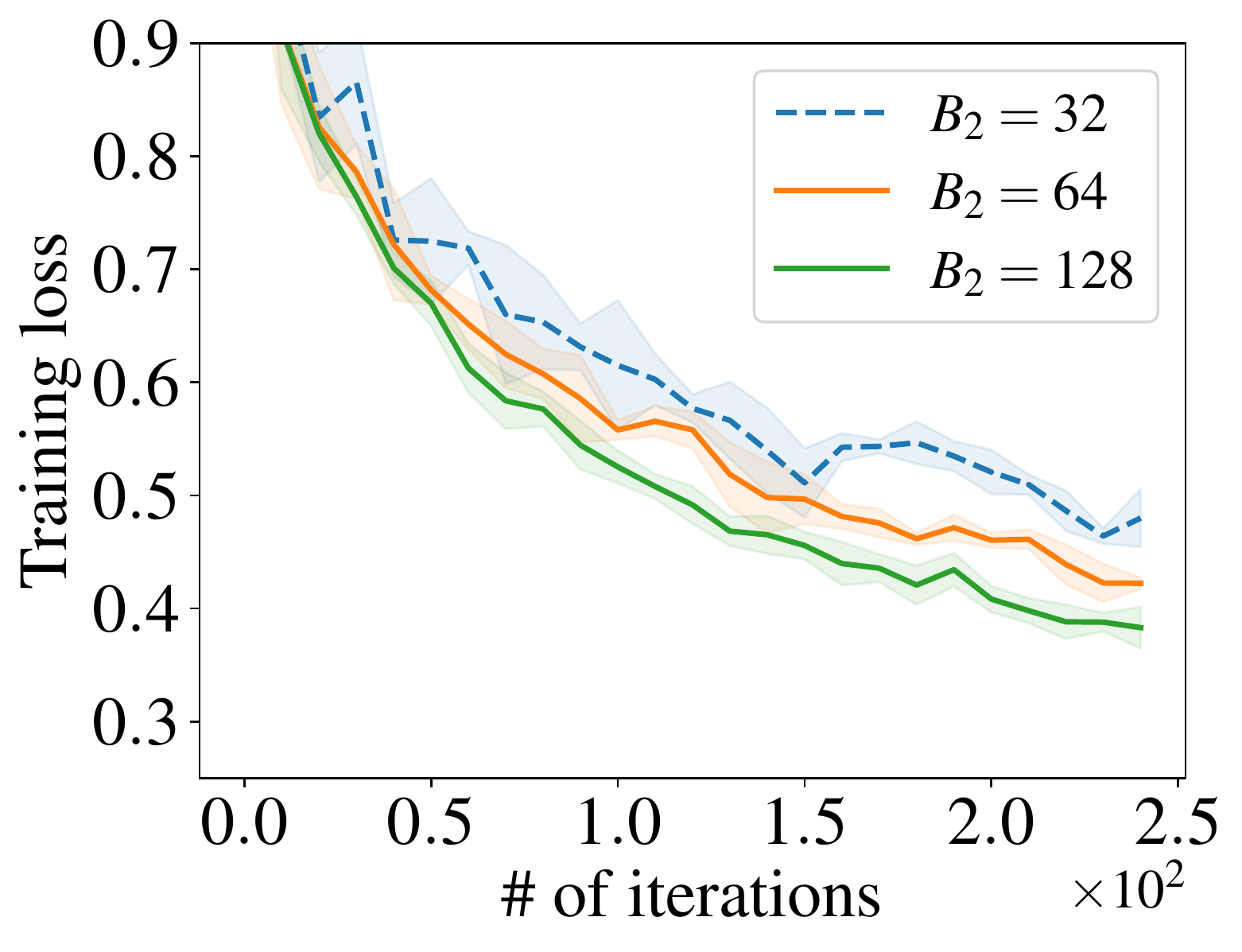}
	}
	\subfigure[MSVR-v2]{
		\includegraphics[width=0.3\textwidth]{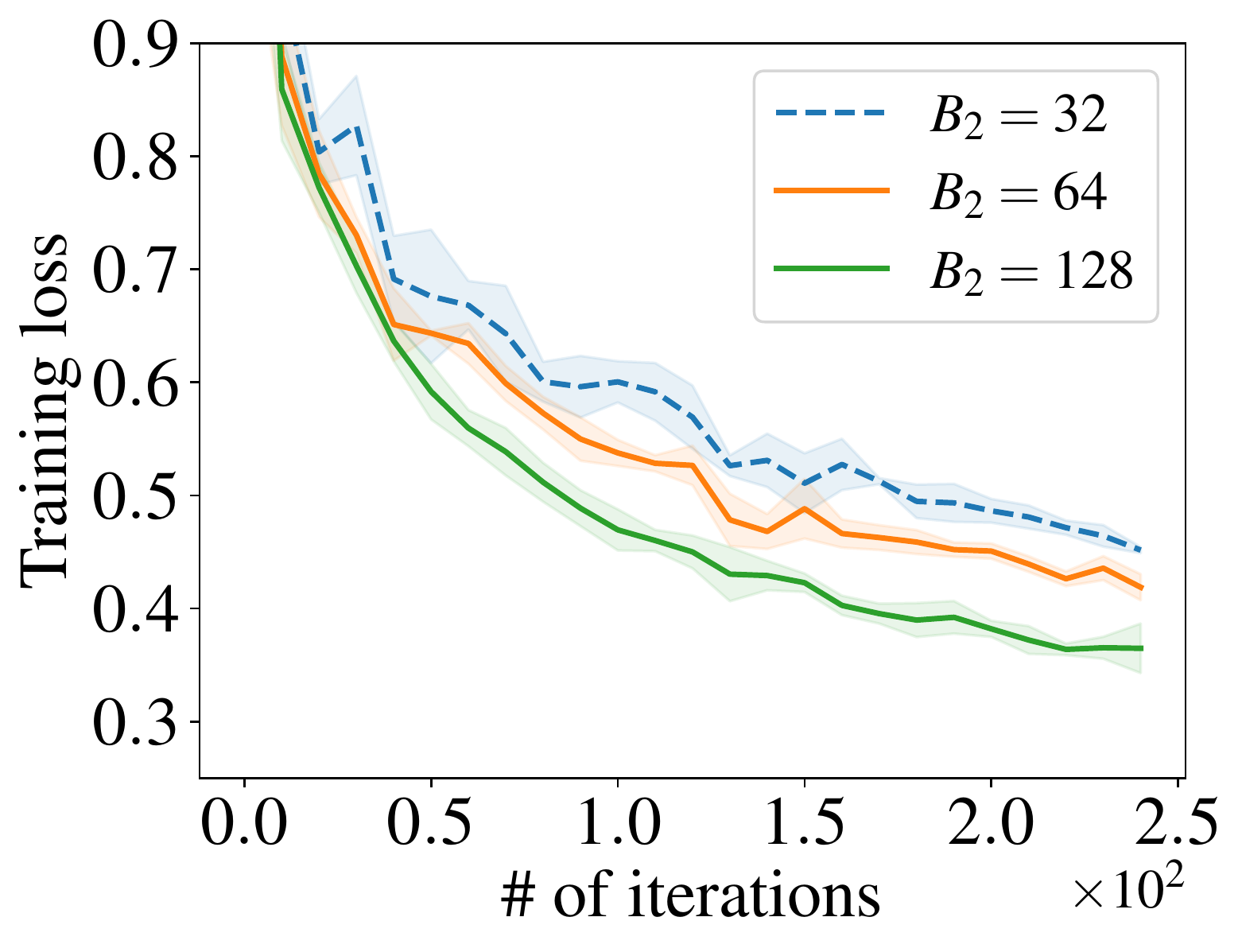}
	}
	\subfigure[MSVR-v3]{
		\includegraphics[width=0.3\textwidth]{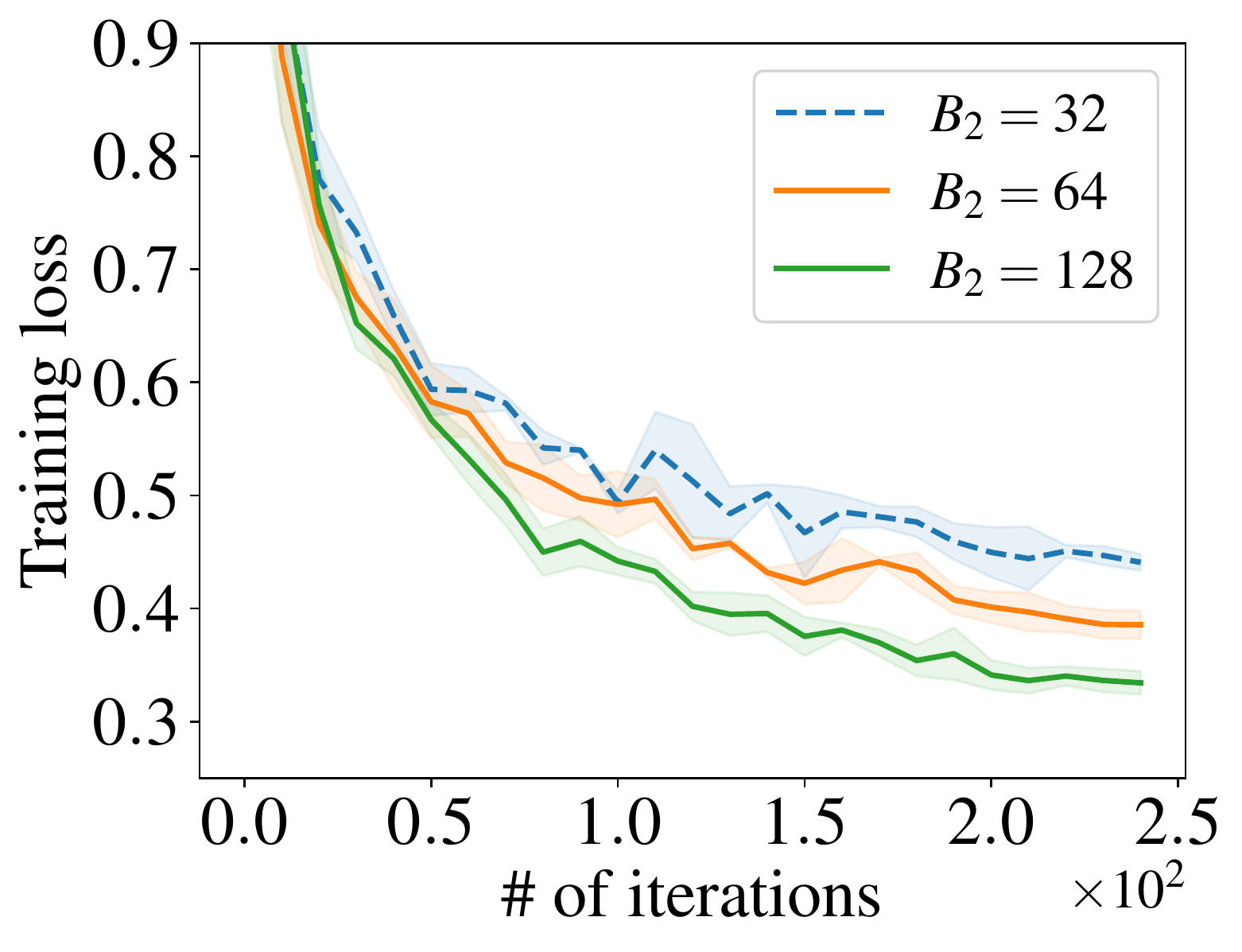}
	}
	\caption{Results with varying $B_2$.}
	\label{fig:5}
\end{figure*}
\subsection{Results with Different Networks}
In this subsection, we conduct experiments on SVHN data set with different networks,  ResNet18, ResNet34 and DenseNet121, respectively. As can be seen in Figure~\ref{fig:3}, with all three networks, MSVR-V1 performs closely to SOX, MSVR-v2 converges faster than SOX and MSVR-v1, and the loss of MSVR-v3 decreases most rapidly, indicating the effectiveness of our methods with different networks.

\subsection{Results with Different Batch size}
In this subsection, we explore the effect of different batch sizes. First, we fix the inner batch size $B_2 = 128$ and vary $B_1$ in the range $\{2,5,9\}$. Then, we fix the outer batch size $B_1=5$ and vary $B_2$ in the range $\{32,64,128\}$. We conduct the experiments on the Fashion-MNIST data set and show the results in Figure~\ref{fig:4} and ~\ref{fig:5}. As can be seen, in terms of iteration complexities, the larger batch size~($B_1$ or $B_2$), the faster the convergence, which is consistent with our theory.
\end{document}